\documentclass[11pt]{article}

\usepackage{setspace}
\usepackage{cprotect}
\usepackage{amsmath,amssymb,amsthm}
\usepackage[noend]{algorithmic}
\usepackage[ruled,vlined]{algorithm2e}
\usepackage{hyperref}
\usepackage{fullpage}

\usepackage{makeidx}
\usepackage{enumerate}
\usepackage{graphicx,float,psfrag,epsfig}
\usepackage{epstopdf}
\usepackage{color}
\usepackage{enumitem}
\usepackage{subfig}
\usepackage{caption}
\usepackage{bigints}
\usepackage{mathtools}
\usepackage[mathscr]{euscript}

\def\<{\langle}
\def\>{\rangle}
\def\veps{\varepsilon}
\def\bveps{\boldsymbol \varepsilon}
\def\bmu{{\boldsymbol \mu}}

\def\bSigma{{\boldsymbol \Sigma}}
\def\bDelta{{\boldsymbol \Delta}}
\def\bGamma{{\boldsymbol \Gamma}}

\def\bLambda{{\boldsymbol \Lambda}}
\def\bPhi{{\boldsymbol \Phi}}
\def\bPsi{{\boldsymbol \Psi}}
\def\hbPhi{\hat{\boldsymbol \Phi}}
\def\hbPsi{\hat{\boldsymbol \Psi}}
\def\blambda{{\boldsymbol \lambda}}
\def\bOmega{{\boldsymbol \Omega}}
\def\bTheta{{\boldsymbol \Theta}}
\def\bXi{{\boldsymbol \Xi}}
\def\tbPhi{\tilde{\boldsymbol \Phi}}
\def\tbPsi{\tilde{\boldsymbol \Psi}}
\def\tbM{{\tilde{\boldsymbol M}}}
\def\barbmu{\bar{\boldsymbol \mu}}
\def\barkap{{\bar{\kappa}}}
\def\barbPhi{\bar{\boldsymbol \Phi}}
\def\barbPsi{\bar{\boldsymbol \Psi}}
\def\barbDelta{{\bar{\boldsymbol \Delta}}}
\def\barbOmega{{\bar{\boldsymbol \Omega}}}
\def\barbSigma{{\bar{\boldsymbol \Sigma}}}
\def\barbGamma{{\bar{\boldsymbol \Gamma}}}
\def\barbM{{\bar{\boldsymbol M}}}

\def\bx{{\boldsymbol x}}
\def\by{{\boldsymbol y}}
\def\ba{{\boldsymbol a}}

\def\be{{\boldsymbol e}}
\def\bs{{\boldsymbol s}}
\def\br{{\boldsymbol r}}

\def\bu{{\boldsymbol u}}
\def\bv{{\boldsymbol v}}
\def\bz{{\boldsymbol z}}
\def\bg{{\boldsymbol g}}
\def\bm{{\boldsymbol m}}

\def\bO{{\boldsymbol O}}
\def\bI{{\boldsymbol I}}
\def\bA{{\boldsymbol A}}
\def\bQ{{\boldsymbol Q}}
\def\bS{{\boldsymbol S}}
\def\bM{{\boldsymbol M}}

\def\bQ{{\boldsymbol Q}}
\def\sT{{\sf T}}
\def\ssb{{\sf b}}

\def\sa{{\sf a}}

\def\hbx{\hat{\boldsymbol x}}
\def\hx{\hat{x}}
\def\hX{\hat{X}}

\def\tf{\tilde{f}}

\def\th{\tilde{h}}

\def\tbM{\widetilde{\boldsymbol M}}

\def\tbPhi{\widetilde{\boldsymbol \Phi}}
\def\tbPsi{\widetilde{\boldsymbol \Psi}}
\def\tbDelta{\widetilde{\boldsymbol \Delta}}
\def\tbOmega{\widetilde{\boldsymbol \Omega}}
\def\tbSigma{\widetilde{\boldsymbol \Sigma}}
\def\tbGamma{\widetilde{\boldsymbol \Gamma}}
\def\tbXi{\widetilde{\boldsymbol \Xi}}
\def\tbTheta{\widetilde{\boldsymbol \Theta}}
\def\barmu{\bar{\mu}}
\def\E{{\mathbb E}} 

\newcommand{\normal}{\mathcal{N}}
\newcommand{\reals}{\mathbb{R}}
\newcommand{\bzero}{\boldsymbol{0}}
\newcommand{\beq}{\begin{equation}}
\newcommand{\eeq}{\end{equation}}
\newcommand\norm[1]{\left\lVert{#1}\right\rVert}
\newcommand\abs[1]{\left\lvert{#1}\right\rvert}
\newcommand{\sign}{\textrm{\sign}}
\newcommand{\op}{\textrm{op}}

\newtheoremstyle{myexample} 
    {\topsep}                    
    {\topsep}                    
    {\rm }                   
    {}                           
    {\bf }                   
    {.}                          
    {.5em}                       
    {}  

\newtheoremstyle{myremark} 
    {\topsep}                    
    {\topsep}                    
    {\rm}                        
    {}                           
    {\bf}                        
    {.}                          
    {.5em}                       
    {}  

\newtheorem{claim}{Claim}[section]
\newtheorem{lemma}[claim]{Lemma}

\newtheorem{theorem}{Theorem}
\newtheorem{proposition}[claim]{Proposition}

\theoremstyle{myremark}

\theoremstyle{myremark}

\theoremstyle{myexample}

\title{Estimation in Rotationally Invariant Generalized Linear Models\\ via Approximate Message Passing} 

\author{Ramji Venkataramanan\thanks{Department of Engineering, University of Cambridge. Email: \texttt{ramji.v@eng.cam.ac.uk},}
\;\;\;\;
 Kevin K{\"o}gler\thanks{Institute of Science and Technology (IST) Austria. Email: \texttt{kevin.koegler@ist.ac.at},} 
\;\;\;and\;\;\;
Marco Mondelli\thanks{Institute of Science and Technology (IST) Austria. Email: \texttt{marco.mondelli@ist.ac.at}.}}

\begin{document}

\maketitle

\begin{abstract}
We consider the problem of signal estimation in generalized linear models defined via rotationally invariant design matrices. Since these matrices can have an arbitrary spectral distribution, this model is well suited for capturing complex correlation structures which often arise in applications. We propose a novel family of approximate message passing (AMP) algorithms for signal estimation, and rigorously characterize their performance in the high-dimensional limit via a state evolution recursion. Our rotationally invariant AMP has complexity of the same order as the existing AMP derived under the restrictive assumption of a Gaussian design; our algorithm also recovers this existing AMP as a special case. Numerical results showcase a performance close to Vector AMP (which is conjectured to be Bayes-optimal in some settings), but obtained with a much lower complexity, as the proposed algorithm does not require a computationally expensive singular value decomposition. 
\end{abstract}

\section{Introduction}

 We consider the problem of estimating a $d-$dimensional signal $\bx^* \in \reals^d$ from an observation  $\by \in \reals^n$ obtained via a generalized linear model (GLM) \cite{mccullagh2018generalized}. 
Specifically, given a design matrix $\bA \in \reals^{n \times d}$ with rows $\ba_1, \ldots, \ba_n \in \reals^d$, the observation $\by \equiv (y_1, \ldots, y_n)$ is generated as
\beq
y_i = q( \< \ba_i, \bx^*  \>, \, \veps_i), \quad  \text{ for } i=1, \ldots, n,
\label{eq:GLM_def}
\eeq
where $\< \ba_i, \bx^*  \> = \ba_i^{\sT} \bx^*$ denotes the Euclidean inner product, $\bveps \equiv (\veps_1, \ldots, \veps_n)$ is a noise vector and $q: \reals^2 \to \reals$ is a known function. The model \eqref{eq:GLM_def} covers many widely studied problems in statistical estimation and signal processing: examples include linear regression \cite{Donoho1,eldar2012compressed} ($y_i = \< \ba_i, \bx^*  \> + \veps_i$), phase retrieval \cite{shechtman2015phase,fannjiang2020numerics} ($y_i = \abs{\< \ba_i, \bx^*  \>}^2 + \veps_i$), and 1-bit compressed sensing \cite{boufounos20081} ($y_i= \text{sign}(\< \ba_i, \bx^*  \> + \veps_i)$).

A range of estimators based on convex relaxations, spectral methods, and non-convex methods have been proposed for specific instances of GLMs, such as sparse linear regression \cite{Tibs96,Dantzig,hastie2019statistical}, phase retrieval \cite{netrapalli2013phase,candes2013phaselift,candes2015wirt,mondelli2017fundamental,luo2019optimal,lu2020phase} and one-bit compressed sensing \cite{plan2012robust, plan2013one, jacques2013robust}. Most of these techniques are generic and can incorporate certain constraints like sparsity, but they are not well-equipped to exploit specific information about $\bx^*$, e.g.,   a known signal prior.  

Approximate message passing (AMP) is a family of iterative algorithms that can be tailored to take advantage of  structural information known about the signal.  AMP algorithms were first proposed for estimation in linear models \cite{Kabashima_2003,BayatiMontanariLASSO,BM-MPCS-2011,DMM09,  krzakala2012,maleki2013asymptotic}, but have since been applied to a range of statistical estimation problems, including generalized linear models \cite{barbier2019optimal,ma2019optimization, maillard2020phase,mondelli2021approximate,RanganGAMP,schniter2014compressive,sur2019modern} and low-rank matrix estimation \cite{deshpande2014information,fletcher2018iterative,kabashima2016phase, lesieur2017constrained,montanari2017estimation, BarbierMR20}. An attractive feature of AMP is that under suitable model assumptions, its performance in the high-dimensional limit is precisely characterized by a succinct deterministic recursion called \emph{state evolution} \cite{BM-MPCS-2011, bolthausen2014iterative, javanmard2013state}.   Using the state evolution analysis, it has been proved that AMP achieves Bayes-optimal performance for some models \cite{deshpande2014information,DonSpatialC13, montanari2017estimation,barbier2019optimal}, and a conjecture from statistical physics posits that AMP is optimal among polynomial-time algorithms.  The above works, including the original GAMP algorithm \cite{RanganGAMP} for estimation in GLMs, all assume that the matrix $\bA$ defining the model is i.i.d. Gaussian.  While some of  these results can be generalized to the broader class of i.i.d. sub-Gaussian matrices via universality arguments \cite{bayati2015universality,ChenLam21}, the i.i.d. assumption limits the applicability of AMP in practice. 
In this paper, we present an AMP algorithm for generalized linear models defined via a rotationally invariant design matrix $\bA$. The class of rotational invariant matrices includes Gaussian matrices, but is much bigger. Rotational invariance only imposes that the orthogonal matrices in the singular value decomposition of $\bA$ are uniformly random, and allows for arbitrary singular values. Hence, $\bA$ is able to capture a complex correlation structure, which is typical in applications. 

\paragraph{Main contributions.} We propose an AMP algorithm for  GLMs with rotationally invariant design matrices. The algorithm, which we call RI-GAMP, uses a pair of multivariate `denoising' functions to produce an updated signal estimate in each iteration. The iterates depend on the free cumulants of the spectral distribution of the design matrix. Assuming that these free cumulants are known (e.g., via the spectral distribution), the complexity of RI-GAMP is of the same order as that of the standard GAMP algorithm \cite{RanganGAMP}. Moreover, when the design matrix is i.i.d. Gaussian, RI-GAMP reduces to standard GAMP. Our main technical contribution is a state evolution result for RI-GAMP (Theorem \ref{thm:main}), which gives a rigorous characterization of its performance in the high-dimensional limit as $n,d \to \infty$ with a fixed ratio $\delta = n/d$, for a constant $\delta >0$. We also present numerical simulation results for linear regression and 1-bit compressed sensing, showcasing the performance of RI-GAMP on both synthetic data and images. The performance of RI-GAMP closely matches that of Vector AMP  \cite{rangan2019vector,pandit2020inference} (which is conjectured to be Bayes-optimal in some settings), but obtained with  much lower complexity, as RI-GAMP does not require computing the singular value decomposition of $\bA$. 

 RI-GAMP offers a flexible framework to analyze other estimators for GLMs, e.g., spectral methods. Standard GAMP has been used as a proof technique to study the distributional properties of linear and spectral estimators under Gaussian model assumptions \cite{mondelli2021optimal, mondelli2021approximate}. An exciting research direction is to use RI-GAMP to analyze spectral estimators for rotationally invariant GLMs in the high-dimensional limit. 


\paragraph{Proof idea.} The key idea in establishing the state evolution result is to design an auxiliary AMP algorithm whose iterates are close to those of RI-GAMP.  The auxiliary AMP is an instance of the abstract  AMP iteration for rotationally invariant matrices analyzed in \cite{fan2020approximate, zhong2021approximate}, hence a  state evolution result can be obtained for it.
We then show that each iterate of RI-GAMP is close to one of the auxiliary AMP, and use this fact to translate the state evolution result for the latter to the RI-GAMP. We emphasize that the auxiliary AMP only serves as a proof technique. Indeed, it is initialized using the unknown signal $\bx^*$, and therefore cannot be used for estimation.

\paragraph{RI-GAMP vs. Vector AMP.} Vector AMP (VAMP) is an iterative algorithm (based on Expectation Propagation) for estimation in rotationally invariant  linear  \cite{rangan2019vector, takeuchi2020rigorous, takeuchi2021convergence} and generalized linear models \cite{schniter2016vector, pandit2020inference}.  Like RI-GAMP, VAMP can be tailored to take advantage of prior information about the signal and its performance can be characterized by a state evolution recursion. It is also shown in \cite{rangan2019vector,pandit2020inference} that the asymptotic estimation error of VAMP (with optimal denoising functions) coincides with the replica prediction for the Bayes-optimal error whenever the  state evolution recursion has a unique fixed point.

The  RI-GAMP  algorithm proposed here is distinct from VAMP, and the associated state evolution recursions are also different.  Let us highlight some attractive features of RI-GAMP. First,  RI-GAMP does not require the computationally expensive ($O(d^3)$) singular value decomposition used in VAMP. Instead, it uses the free cumulants of the design matrix which can be estimated in $O(d^2)$ time (details on p.\pageref{para:freecum_est}). 
We confirm via numerical simulations  that the complexity advantage of RI-GAMP over VAMP is significant and increases with the problem dimension (see Figure \ref{fig:running_time_comp}).

The state evolution result (Theorem \ref{thm:main}) for RI-GAMP holds under mild assumptions on the denoising functions (see \textbf{(A1)} on p.\pageref{assump:A1}), while the analysis for VAMP requires slightly stronger conditions, e.g., the denoising functions and their derivatives need to be uniformly Lipschitz continuous. The numerical results in Section \ref{sec:simu} show that the performance of RI-GAMP matches that of VAMP, except near parameter values corresponding to a phase transition in the estimation error.

\paragraph{Other related work.} Orthogonal AMP \cite{ma2017orthogonal} is an algorithm equivalent to VAMP for estimation in  rotationally invariant linear models.   A variant of Expectation Propagation (an algorithm closely related to VAMP) was proposed in \cite{ccakmak2016self} for rotationally invariant GLMs. Ma et al. \cite{ma2021analysis} recently studied the performance of Expectation Propagation for rotationally invariant GLMs, and analyzed the impact of the spectrum  on the estimation performance. VAMP has also been used to obtain the asymptotic risk of convex penalized estimators for rotationally invariant GLMs \cite{gerbelot2020asymptotic1, gerbelot2020asymptotic2}.  A few lower  complexity alternatives to VAMP have been proposed recently, including convolutional AMP \cite{takeuchi2021bayes}, Memory AMP for linear models \cite{LiuMemoryAMP20}, and Generalized Memory AMP for GLMs \cite{tian2021generalized}.    The phase retrieval problem (which is a special case of a GLM) has been studied for design matrices with orthogonal columns, a model distinct from the rotationally invariant one considered here \cite{dudeja2020analysis, dudeja2020information}.  Finally, we mention that AMP has also been applied to low-rank matrix estimation with rotationally invariant noise \cite{opper2016theory, ccakmak2019memory, fan2020approximate,zhong2021approximate, mondelli21Rank1PCA}.

\section{Preliminaries}\label{sec:prel}

\paragraph{Notation and definitions.} 
 For $n\in \mathbb N$, we use the shorthand $[n]$ to denote the set $\{1, \ldots, n\}$. Given a vector $\bx$, we write $\| \bx \|$  for its $\ell_2$ norm. All vectors are treated as column vectors. Given  $\bx = (x_1, \ldots, x_d)$, we denote by $\< \bx \>$ its empirical average $\frac{1}{d} \sum_{i=1}^d x_i$.  The empirical distribution of  $\bx$ is given by $ \frac{1}{d}\sum_{i=1}^d \delta_{x_i}$, where $\delta_{x_i}$ denotes a Dirac delta mass on $x_i$. Similarly, the joint empirical  distribution of the rows of a matrix $(\bx^1, \bx^2, \ldots, \bx^t) \in \reals^{d \times t}$ is $\frac{1}{d} \sum_{i=1}^d \delta_{(x^1_i, \ldots, x^t_i)}$. Given a matrix $\bA$, we denote by $(\bA)_{i, j}$ its $(i, j)$-th element.
For a square matrix $\bM$, we follow the convention that $\bM^0$ is the identity matrix.

\paragraph{$W_2$ convergence.} We write 
$(\bx^1,  \ldots, \bx^k) \stackrel{\mathclap{W_2}}{\longrightarrow} (X_1, \ldots, X_k)$
for the  convergence in Wasserstein-$2$ distance of the joint empirical distribution of the rows of $(\bx^1, \bx^2, \ldots, \bx^k) \in \reals^{d \times k}$ to the law of the random vector $(X_1, \ldots, X_k)$.   Equivalently  \cite{feng2021unifying}[Corollary 7.4], for any $L >0$ and function $\psi: \reals^k \to \reals$ that satisfies 
\beq \abs{ \psi(\bu) - \psi(\bv)} \le L \| \bu -\bv \|\left(1  +  \| \bu \|   
+  \| \bv \| \right) 
\label{eq:PL_def}
\eeq
for all $\bu,\bv \in \reals^k$, we have
\beq
\lim_{d \to \infty} \, \frac{1}{d} \sum_{i=1}^d \psi(x^1_i,  \ldots, x^k_i)  = 
\E\{ \psi(X_1, \ldots, X_k) \}.
\label{eq:PL_conv}
\eeq
A function satisfying \eqref{eq:PL_def} for some fixed $L >0$ is called a pseudo-Lipschitz function of order $2$.

\paragraph{Rotationally invariant generalized linear models.} The $n \times d$ design matrix $\bA $ is formed by stacking the sensing vectors $\ba_1, \ldots, \ba_n$, i.e., $\bA = [\ba_1, \ldots, \ba_n]^{\sT}$. We assume that $\bA$ is bi-rotationally invariant in law, i.e., $\bA = \bO^\sT \bLambda \bQ$, where $\bLambda={\rm diag}(\blambda)$ is an $n\times d$ diagonal matrix containing the singular values of $\bA$, and $\bO$, $\bQ$ are Haar orthogonal matrices independent of one another and also of $\bLambda$. As $d\to\infty$, we assume that $n/d = \delta$, for some constant $\delta > 0$. 
The  matrix $\bA$ is independent of the signal $\bx^* \in \reals^d$, and the noise vector $\bveps \in \reals^n$. The observation $\by \in \reals^n$ is generated according to \eqref{eq:GLM_def}. 
We assume that there exist random variables $X_*, \veps$ with finite second moments such that $\bx^* \stackrel{\mathclap{W_2}}{\longrightarrow} X_*$, and $\bveps \stackrel{\mathclap{W_2}}{\longrightarrow} \veps$ as $n\to \infty$. Furthermore, we assume that the empirical distribution of $\blambda$ converges weakly almost surely to a compactly supported random variable  $\Lambda$.
We denote by $\{\kappa_{2k}\}_{k\ge 1}$ the rectangular free cumulants associated with the moments $\{m_{2k}\}_{k\ge 1}$, where $m_{2k}$ is the $k$-th moment of the empirical eigenvalue distribution of $\bA\bA^\sT$ (for details, see \eqref{eq:mcrel1rect}-\eqref{eq:mcrel2rect} in Appendix \ref{subsec:freerect}). For $\delta>1$, let $\tilde{\Lambda}$ be a mixture of $\Lambda$ (w.p. $1/\delta$) and a point mass at $0$ (w.p. $1-1/\delta$); for $\delta\le 1$, we set $\tilde{\Lambda}=\Lambda$. Then, the assumptions above imply that as $n,d\to\infty$, $m_{2k} \to \bar{m}_{2k}=\mathbb E\{\tilde{\Lambda}^{2k}\}$ and $\kappa_{2k}\to \bar{\kappa}_{2k}$ almost surely, where $\{\bar{m}_{2k}\}_{k\ge 1}$ and $\{\bar{\kappa}_{2k}\}_{k\ge 1}$ are the even moments and rectangular free cumulants of $\tilde{\Lambda}$.

\section{Rotationally Invariant Generalized AMP}\label{sec:RIGAMP}

\paragraph{Algorithm.} We propose the following rotationally invariant generalized AMP (RI-GAMP) to estimate $\bx^* \in \reals^d$ from the observation $\by \in \reals^n$ and the design matrix $\bA \in \reals^{n \times d}$.  For $t \ge 1$, compute:
\begin{align}
    & \bx^t = \bA^{\sT} \bs^t - \sum_{i=1}^{t-1} \beta_{ti} \, \hbx^i, \qquad 
    \hbx^t  = f_t(\bx^1, \ldots, \bx^t), \label{eq:AMP_xt_update} \\
    & \br^t = \bA \hbx^t  - \sum_{i=1}^t \alpha_{ti}  \, \bs^i, \qquad \bs^{t+1} = h_{t+1}(\br^1, \ldots, \br^t, \,  \by).  \label{eq:AMP_rt_update}
\end{align}
The iteration is initialized with $\bs^1=h_1(\by)$ and $\bx^{1} = \bA^{\sT} \bs^1$. 
For $t \geq 1$, the functions $f_t: \reals^{t} \to \reals$ and $h_{t+1}: \reals^{t+1} \to \reals$ are applied row-wise to vectors and matrices. The scalars $\{ \alpha_{ti} \}_{i=1}^t$ and $\{ \beta_{ti} \}_{i=1}^{t-1}$ are obtained in terms of two lower-triangular matrices $\bPsi_{t+1}, \bPhi_{t+1} \in \reals^{(t+1) \times (t+1)}$. These matrices are defined as 
\begin{align}\label{eq:Psi_Phi_def}
    \bPsi_{t+1} = \begin{pmatrix}
    0 & 0 & \ldots & 0 & 0 \\
    0 & \< \partial_1 \hbx^1 \> & 0 & \ldots & 0 \\
    0  &  \< \partial_1 \hbx^2 \> &  \< \partial_2 \hbx^2 \> & \ldots & 0 \\
    \vdots & \vdots & \vdots & \ddots & \vdots \\
    0 & \< \partial_1 \hbx^t \> & \< \partial_2 \hbx^t \> & \ldots & \< \partial_t \hbx^t \>
    \end{pmatrix}, \nonumber\\
    \bPhi_{t+1} = \begin{pmatrix}
    0 & 0 & \ldots & 0 & 0 \\
    \< \partial_g \bs^1 \> &  0 &  0 & \ldots & 0 \\
    \< \partial_g \bs^2 \>  &  \< \partial_1 \bs^2 \> &  0 & \ldots & 0 \\
    \vdots & \vdots & \ddots & \vdots & \vdots \\
    \< \partial_g \bs^t \> & \< \partial_1 \bs^t \> & \ldots & \< \partial_{t-1} \bs^t \> 
    &  0
    \end{pmatrix},
\end{align}
where for $k \in [t]$, the vector $\partial_k \hbx^t \in \reals^d$  denotes the partial derivative  $\partial_{x_k} f_t(x_1, \ldots, x_t)$ applied row-wise to $\hbx^t=f_t(\bx^1, \ldots, \bx^t)$. Similarly,  the vector $\partial_k \bs^t \in \reals^n$  denotes the partial derivative  $\partial_{r_k} h_{t}(r_1, \ldots, r_{t-1}, y)$ applied row-wise to $\bs^t = h_t(\br^1, \ldots, \br^{t-1}, \by)$. Recalling that $y=q(g, \veps)$, we can view
$h_{t}(r_1, \ldots, r_{t-1}, q(g, \veps))$ as a function of $(t+1)$  variables, with $\partial_g h_{t}(r_1, \ldots, r_{t-1}, q(g, \veps))$ being the partial derivative with respect to $g$. For $t>1$, the vector $\partial_g \bs^t$ denotes this partial derivative applied row-wise to $\bs^t = h_t(\br^1, \ldots, \br^{t-1}, q(\bg, \bveps))$, and the vector $\partial_g \bs^1$ is  defined via the partial derivative $\partial_g h_{1}(q(g, \veps))$. Next, recalling that $\{ \kappa_{2k} \}_{k \ge 1}$ denote the rectangular free cumulants, define  matrices $\bM^\alpha_{t+1}, \bM^\beta_{t+1} \in \reals^{(t+1) \times (t+1)}$ as:
\begin{align}
    & \bM^\alpha_{t+1} =     \sum_{j=0}^{t+1} \, \kappa_{2(j+1)} \, \bPsi_{t+1} \Big(\bPhi_{t+1} \bPsi_{t+1}\Big)^j, \qquad \label{eq:Malpha_def} \\
    & \bM^\beta_{t+1} =    \delta \sum_{j=0}^{t} \, \kappa_{2(j+1)} \, \bPhi_{t+1} \Big(\bPsi_{t+1} \bPhi_{t+1}\Big)^j. \label{eq:Mbeta_def}
\end{align}
Then, the coefficients $\{ \alpha_{ti} \}_{i=1}^t$ and $\{ \beta_{ti} \}_{i=1}^{t-1}$ in \eqref{eq:AMP_xt_update}-\eqref{eq:AMP_rt_update} are obtained from the last row of $\bM^\alpha_{t+1}$ and $\bM^\beta_{t+1}$ as: 
\begin{align}
    \label{eq:alpha_def}
    & (\alpha_{t1}, \ldots, \alpha_{tt})= (\,
    (\bM^\alpha_{t+1})_{t+1, 2}\,, \, \ldots, \, (\bM^\alpha_{t+1})_{t+1, t+1} \, ), \\
     \label{eq:beta_def}
     & (\beta_{t1}, \ldots, \beta_{t,t-1}) = (\, (\bM^\beta_{t+1})_{t+1, 2}, \, \ldots, \, (\bM^\beta_{t+1})_{t+1, t} \,  ).
\end{align}


\paragraph{Estimating the free cumulants.}  \label{para:freecum_est}
From the definitions of $\{ \alpha_{ti} \}_{i=1}^t$ and $\{ \beta_{ti} \}_{i=1}^{t-1}$ above, it follows that we need the free cumulants 
$\{ \kappa_{2(j+1)} \}_{j=0}^{t+1} $ 
to compute the first $t$ iterates of RI-GAMP in \eqref{eq:AMP_rt_update}. These free cumulants can be recursively computed from the moments 
$\{ m_{2(j+1)} \}_{j=0}^{t+1}$ of the spectral distribution of $\bA \bA^{\sT}$, using the formula \eqref{eq:mcrel2rect} in Appendix \ref{subsec:freerect}. The moments $\{ m_{2(j+1)} \}_{j=0}^{t+1}$  can be  estimated in $O(d^2)$ time via the following simple algorithm proposed in \cite{LiuMemoryAMP20}. Given $\bA \in \mathbb{R}^{n \times d}$, pick an independent standard Gaussian vector $\bs^0 \sim \normal(\bzero, \bI_n)$, and for $k \ge 1$, compute $\bs^k = \bA^{\sT} \bs^{k-1}$ for odd $k$  and $\bs^k = \bA \bs^{k-1}$ for even $k$.  Then, $\frac{\| \bs^k\|^2}{d}$ is a consistent estimate of the $k$-th moment of the spectral distribution of $\bA \bA^{\sT}$.  Thus,  the complexity of estimating the free cumulants is of the same order as one iteration of  RI-GAMP.

\paragraph{State evolution.}
The  coefficients   $\{ \alpha_{ti} \}_{i=1}^t$ and $\{ \beta_{ti} \}_{i=1}^{t-1}$ play a crucial role in debiasing the AMP iterates, ensuring that their limiting empirical distributions are accurately captured by state evolution. Indeed, Theorem \ref{thm:main} shows that the joint  empirical distribution of 
$(\bg, \br^1, \ldots, \br^t)$ converges to a ($t+1$)-dimensional Gaussian distribution $\normal(\bzero,\barbSigma_{t+1})$. Similarly, the joint  empirical distribution of 
$( \bx^1 - \bar{\mu}_{1} \bx, \ldots, \bx^t - \bar{\mu}_{t} \bx)$ converges to a $t$-dimensional  Gaussian  $\normal(\bzero, \barbOmega_t)$. We define the  covariance matrices $\barbOmega_{t}, \barbSigma_t \in \reals^{t \times t}$ and the vector 
$\barbmu_t  \equiv (\bar{\mu}_1, \ldots, \bar{\mu}_t)$ recursively for $t \ge 1$, starting with 
\beq
\label{eq:SE_init}
\begin{split}
&\barbSigma_1=\barkap_2 \E\{X_*^2\}, \quad \barbmu_1= \delta \barkap_2  \,  \E \{ \partial_g h_1(q(G, \veps))\},\\
& \barbOmega_1=\delta \barkap_2 \E\{h_1(q(G, \veps))^2 \}  +  \delta \barkap_4 \E\{ X_*^2\} (\E\{ \partial_g h_1(q(G, \veps))\})^2 , 
\end{split}
\eeq
where $G \sim \normal(0, \,  \barkap_2 \E\{X_*^2\})$ is independent of $\veps$. Here, $X_*$ is the law of  the limiting empirical distribution of the signal, as defined in Section \ref{sec:prel}.
For $t \ge 1$, given $\barbmu_t, \barbOmega_t, \barbSigma_t$, let 
\begin{align}
    & (G, R_1, \ldots, R_{t-1})   \sim \normal(\bzero, \barbSigma_t), \nonumber\\
    & S_t = h_t(R_1, \ldots, R_{t-1}, Y), \ \text{where } Y=q(G, \veps), \label{eq:GRi_def}\\
    & (X_1, \ldots X_t) = \barbmu_t X_*  +  (W_1, \ldots, W_t),
    \ \text{where }  (W_1, \ldots, W_t) \sim \normal(\bzero, \barbOmega_t) \text{ is  independent of } X_*, \nonumber  \\
    & \hat{X}_t = f_t(X_1, \ldots X_t).       \label{eq:Xi_def}
\end{align}
Let $\barbDelta_{t+1}, \barbGamma_{t+1} \in \reals^{(t+1) \times (t+1)}$ be symmetric matrices with entries given by 
\begin{align}
    &(\barbDelta_{t+1})_{1,i} = (\barbDelta_{t+1})_{i,1} =0, \nonumber\\ \
    & (\barbDelta_{t+1})_{i+1,j+1} = \E\{S_i S_j \}, 
    \quad  i, j \in [t], \label{eq:BarDelta_def} \\
   &  
  (\barbGamma_{t+1})_{1,1}=\E\{ X_*^2 \},\nonumber\\
  &  (\barbGamma_{t+1})_{1,i+1} = (\barbGamma_{t+1})_{i+1,1} =\E\{ X_* \hat{X}_i \}, \nonumber \\
    & (\barbGamma_{t+1})_{i+1,j+1} = \E\{ \hat{X}_i \hat{X}_j \}, \quad i, j \in [t].  \label{eq:BarGamma_def}
\end{align}
Furthermore, let $\barbPsi_{t+1}, \barbPhi_{t+1}$ denote the deterministic versions of the matrices $\bPsi_{t+1}, \bPhi_{t+1}$ in \eqref{eq:Psi_Phi_def}, obtained by replacing the empirical averages by expectations. Specifically, to obtain $\barbPsi_{t+1}, \barbPhi_{t+1}$ we replace the entries as follows:
\begin{equation}
    \begin{split}
       &  \<  \partial_k \hbx^t \> \to  \E\{ \partial_k \hat{X}_t\} = \E\{ \partial_{X_k}f_t(X_1, \ldots, X_t) \}, 
       \\
       &  \<  \partial_k \bs^t \> \to \E\{ \partial_k S_t \} = 
        \E\{ \partial_{R_k} h_t(R_1, \ldots, R_{t-1}, q(G, \veps)) \}, 
        \\
       &  \<  \partial_g \bs^t \> \hspace{-.1em}\to\hspace{-.1em} \E\{ \partial_g S_t \}\hspace{-.1em} = \hspace{-.1em}
        \E\{ \partial_{g} h_t(R_1, \ldots, R_{t-1}, q(g, \veps))  \vert_{g=G} \}.
    \end{split}
    \label{eq:emp_avg_shorthand}
\end{equation}

We now describe how $\barbSigma_{t+1}, \barbOmega_{t+1}, \barbmu_{t+1}$ are computed from $\barbSigma_t, \barbOmega_t, \barbmu_t$. Given $\barbSigma_t, \barbOmega_t, \barbmu_t$, we can evaluate the matrices $\barbDelta_{t+1}, \barbGamma_{t+1}, \barbPsi_{t+1}, \barbPhi_{t+1}$. 
From these, we compute 
$\barbSigma_{t+1} \in \reals^{(t+1) \times (t+1)}$ as
\beq
\barbSigma_{t+1} =  \sum_{j=0}^{2t+1} \barkap_{2(j+1)} \, \bXi_{t+1}^{(j)},
\label{eq:barSigma_def}
\eeq
where $\bXi_{t+1}^{(0)} = \barbGamma_{t+1}$, and for $j \ge 1$:
\begin{align}
\bXi_{t+1}^{(j)} &  = \sum_{i=0}^{j} (\barbPsi_{t+1} \barbPhi_{t+1})^i \, \barbGamma_{t+1} 
\Big((\barbPsi_{t+1} \barbPhi_{t+1})^{\sT}\Big)^{j-i}
\nonumber \\
&  +  \sum_{i=0}^{j-1} (\barbPsi_{t+1} \barbPhi_{t+1})^i  \barbPsi_{t+1}\barbDelta_{t+1} \barbPsi_{t+1}^{\sT}
\Big(\hspace{-.2em}(\barbPsi_{t+1} \barbPhi_{t+1})^{\sT}\hspace{-.1em}\Big)^{j-i-1}. \label{eq:Xi_tj_def} 
\end{align}
Recalling that $\barbSigma_{t+1}$ is the covariance of $(G, R_1, \ldots, R_t)$, we can now compute $\barbDelta_{t+2}, \barbPhi_{t+2} \in \reals^{(t+2) \times (t+2)}$.
Using these, we define a symmetric
$(t+2) \times (t+2)$ matrix $\bOmega'_{t+2}$, whose first row and column equal zero and whose lower right $(t+1) \times (t+1)$ submatrix equals $\barbOmega_{t+1}$. Specifically, 
\beq
\bOmega'_{t+2} =  \delta \sum_{j=0}^{2(t+1)} \barkap_{2(j+1)} \bTheta_{t+2}^{(j)},
\label{eq:Omega_pr_def}
\eeq
where $\bTheta_{t+2}^{(0)} = \barbDelta_{t+2}$, and for $j \ge 1$:
\beq
\begin{split}
\bTheta_{t+2}^{(j)} &  = \sum_{i=0}^j (\barbPhi_{t+2} \barbPsi_{t+2} )^i \, \barbDelta_{t+2}  \Big((\barbPhi_{t+2}\barbPsi_{t+2})^{\sT}\Big)^{j-i}  \\
&  +  \sum_{i=0}^{j-1} (\barbPhi_{t+2} \barbPsi_{t+2} )^i \, \barbPhi_{t+2} \barbGamma_{t+2}  \barbPhi_{t+2}^{\sT}
\Big(\hspace{-.1em}(\barbPhi_{t+2}\barbPsi_{t+2})^{\sT}\hspace{-.1em}\Big)^{j-i-1} \hspace{-0.2em}.
\end{split}
\label{eq:Theta_t1j_def}
\eeq
Then, the entries of the covariance matrix $\barbOmega_{t+1} \in \reals^{(t+1) \times (t+1)}$ are given by:
\beq
(\barbOmega_{t+1})_{ij}= (\bOmega'_{t+2})_{i+1, j+1}, \quad  i,j \in [t+1].
\eeq
Finally, we compute the mean parameter \\ $\bar{\mu}_{t+1} = \Big(\barbM^\beta_{t+2}\Big)_{t+2, 1}$, where
\beq
\label{eq:mut1_def}
\barbM^\beta_{t+2} = \delta \sum_{j=0}^{t+1} \, \bar{\kappa}_{2(j+1)} \, \barbPhi_{t+2} \Big(\barbPsi_{t+2} \barbPhi_{t+2}\Big)^j.
\eeq

Though the formulas for $\bTheta_{t+2}^{(j)}$ in \eqref{eq:Theta_t1j_def} and $\barbM^\beta_{t+2}$ in \eqref{eq:mut1_def} contain $\barbGamma_{t+2}$ and $\barbPsi_{t+2}$, the last rows and columns of these matrices are zeroed out in the computation (due to the form of $\barbDelta_{t+2}$ and  $\barbPhi_{t+2}$). Therefore the formulas depend only on the top left submatrices of $\barbGamma_{t+2}$ and $\barbPsi_{t+2}$, namely,  $\barbGamma_{t+1}$ and $\barbPsi_{t+1}$.  We also note that the matrices $ \barbOmega_t$ and $\barbSigma_t$ are the top left submatrices of $\barbOmega_{t+1}$ and $\barbSigma_{t+1}$, respectively. Similarly, the mean vector $\barbmu_{t+1}$ is obtained by appending $\bar{\mu}_{t+1}$ to $\barbmu_{t}$.

\paragraph{Main result.}
Having defined the state evolution recursion to  compute  $\barbmu_t, \barbOmega_t, \barbSigma_t$ (which specify the joint distributions in \eqref{eq:GRi_def}-\eqref{eq:Xi_def}), we are ready to state our main result.    We make the following assumption on the functions $f_t, h_t$ used in the AMP \eqref{eq:AMP_xt_update}-\eqref{eq:AMP_rt_update}, for $t \ge 1$:
\begin{enumerate}
    \item[\textbf{(A1)}]  The functions $f_t(X_1, \ldots, X_t)$ and $h_t(R_1, \ldots, R_{t-1}, \, q(G, \veps))$ are  Lipschitz in each of their arguments. The partial derivatives $\partial_{X_k} f_t((X_1, \ldots, X_t))$, 
    $\partial_G h_t(R_1, \ldots, R_{t-1}, \, q(G, \veps))$, and  $\partial_{R_\ell} h_t(R_1, \ldots, R_{t-1}, \, q(G, \veps))$ are all continuous on sets of probability $1$, under the laws of $(X_1 \ldots, X_t)$ and $(G, R_1,\ldots, R_{t-1})$ given in  \eqref{eq:GRi_def}-\eqref{eq:Xi_def}.
    \label{assump:A1}
\end{enumerate}

\begin{theorem}
Consider a rotationally invariant generalized linear model with the assumptions in Section \ref{sec:prel} and the AMP \eqref{eq:AMP_xt_update}-\eqref{eq:AMP_rt_update} with the assumption \textbf{(A1)} above.
Let  $\psi:\reals^{2t+1} \to \reals$ and $\phi: \reals^{2t+2} \to \reals$ be any pseudo-Lipschitz functions of order $2$. Then for each $t \ge 1$,  we almost surely have
\begin{align}
& \lim_{n \to \infty} \frac{1}{d} \sum_{i=1}^{d} \psi(x^1_i, \ldots, x^t_i, \,  \hat{x}^1_i, \ldots, \hat{x}^t_i, \, x^*_i) = \E\{ \psi(X_1, \ldots, X_t, \hat{X}_1, \ldots, \hat{X}_t, \, X_*) \}, \label{eq:PL2_main_resultx} \\
& \lim_{n \to \infty} \frac{1}{n} \sum_{i=1}^{n} \phi(r^1_i, \ldots, r^t_i, \,  s^1_i, \ldots, s^{t+1}_i, \, y_i) = \E\{ \phi(R_1, \ldots, R_t, \, S_1, \ldots, S_{t+1}, \, Y) \},
\label{eq:PL2_main_resultr}
\end{align}
where the random variables on the right are defined in \eqref{eq:GRi_def}-\eqref{eq:Xi_def}.
 Equivalently, as $n \to \infty$, the joint empirical distributions of $(\bx^1, \ldots, \bx^t, \hbx^1, \ldots, \hbx^t, \bx^*)$ and $(\br^1, \ldots, \br^t, \bs^1, \ldots, \bs^{t+1}, \by)$ converge almost surely in Wasserstein-2 distance to $(X_1, \ldots, X_t, \hat{X}_1, \ldots, \hat{X}_t, X_*)$ and $(R_1, \ldots, R_t, S_1, \ldots, S_{t+1}, Y)$, respectively.
\label{thm:main}
\end{theorem}

The proof of the theorem is given in Appendix \ref{app:full_proof}, and we provide a proof sketch in Section \ref{sec:proof_sketch}.
When the design matrix $\bA$ has i.i.d. $\normal(0, 1/n)$ entries, we have $\barkap_2=\delta$ and $\barkap_{2k}=0$ for $k  \ge 2$. In this case, the RI-GAMP \eqref{eq:AMP_xt_update}-\eqref{eq:AMP_rt_update}, with denoising functions of the form $f_t(\bx^t)$ and $h_{t+1}(\br^t,  \by)$, reduces to the existing GAMP algorithm \cite{RanganGAMP}. The  state evolution recursion also reduces to that of GAMP (see, e.g., Section 4 of \cite{feng2021unifying}). 
This opens up an exciting research direction on using RI-GAMP to generalize results where GAMP has been used as a proof technique under Gaussian model assumptions. One example is to determine the distributional properties of spectral estimators for rotationally invariant GLMs. 

\paragraph{MSE and correlation.} The result \eqref{eq:PL2_main_resultx} readily leads to the evaluation of the usual quantities of interest, such as the mean squared error (MSE) and the normalized squared correlation. Indeed, by taking $\psi(\hat{x}_i^t, x_i^*)=(\hat{x}_i^t- x_i^*)^2$, we have that $\frac{1}{d}\|\hbx^t-\bx^*\|^2_2$  $ \to \mathbb E\{(\hat{X}_t-X_*)^2\}$ for each $t \ge 1$. Furthermore, by taking $\psi(\hat{x}_i^t, x_i^*)=\hat{x}_i^t\cdot x_i^*$, $\psi(\hat{x}_i^t)=(\hat{x}_i^t)^2$ and $\psi(x_i^*)=(x_i^*)^2$, we have that $|\langle\hbx^t,\bx^*\rangle|^2/(\|\hbx^t\|^2\|\bx^*\|^2)$ tends to $(\mathbb E\{X_t \, X_*\})^2/(\mathbb E\{X_t^2\}\mathbb E\{X_*^2\})$. 

\paragraph{Empirical state evolution parameters.}
We can define empirical versions of the state evolution parameters, denoted by $\bSigma_{t+1}, \bOmega_{t+1}, \bmu_{t+1}$, by replacing $\barbPsi_{k}, \barbPhi_{k}, \barbDelta_{k}, \barbGamma_{k}$ ($k\in \{t+1, t+2\}$) with  $\bPsi_{k}, \bPhi_{k}, \bDelta_{k}, \bGamma_{k}$ in \eqref{eq:Xi_tj_def}, \eqref{eq:Theta_t1j_def}, and \eqref{eq:mut1_def}. The latter matrices are computed using empirical averages instead of expectations: $\bPhi_k, \bPsi_k$ are defined in \eqref{eq:Psi_Phi_def} and for $\bDelta_k, \bGamma_k$, we replace the expectations $\E\{ S_i S_j \}$ and $\E\{ \hat{X}_i \hat{X}_j \}$ in \eqref{eq:BarDelta_def}-\eqref{eq:BarGamma_def} by $\< \bs^i, \bs^j \>/n$ and 
$\< \hbx^i, \hbx^j \>/d$. The expectations $\E\{ X_* \hat{X}_i \}$ in $\barbGamma_k$ can be estimated for the case of posterior mean denoisers using the identity in \eqref{eq:spest}. Theorem \ref{thm:main} gives that the empirical versions of these matrices converge to the deterministic ones almost surely, and therefore, $\bSigma_{t+1} \to \barbSigma_{t+1}$,  $\bOmega_{t+1} \to \barbOmega_{t+1}$, and  $\bmu_{t+1} \to \barbmu_{t+1}$. For the simulations in Section \ref{sec:simu}, RI-GAMP is implemented with state evolution parameters empirically estimated, as this choice leads to more stable numerical results. 
\label{para:emp_SE}

\paragraph{Initialization.} Note that the algorithm is initialized with $\bx^1=\bA^\sT h_1(\by)$. If this initialization is not effective, in the sense that $\frac{1}{d}\langle \bx^1, \bx^*\rangle\to 0$, then state evolution remains stuck at a trivial fixed point, i.e., $\bar{\mu}_t=0$ for all $t$, and all the iterates produced by RI-GAMP are not correlated with the signal.\footnote{For both  linear regression and 1-bit compressed sensing, the standard initialization $\bx^1=\bA^\sT h_1(\by)$ is effective. For a characterization of the GLMs for which this initialization  is not effective for a Gaussian design matrix $\bA$, see (3.13) in \cite{mondelli2021approximate} and the discussion therein. One important example is phase retrieval ($y_i=|\langle\ba_i, \bx^*\rangle|^2+\varepsilon_i$).} To address this issue, we can assume to be given an initialization $\bx^1$, which is correlated with $\bx^*$, i.e., $\frac{1}{d}\langle \bx^1, \bx^*\rangle\to \alpha>0$, and independent of $\bA$. Theorem \ref{thm:main} still holds for such an initialization, with the only change being in the initialization of the state evolution recursion. Suppose that  $\bx^1\stackrel{\mathclap{W_2}}{\longrightarrow} X_1$ for a random variable $X_1$ satisfying $\E\{ X_1 X_* \} =\alpha$. Then, in the state evolution initialization \eqref{eq:SE_init}, we set $\barbmu_1=\mathbb E\{ X_1 X_*\}=\alpha$ and $\barbOmega_1=\mathbb E\{ (X_1- \barbmu_1 X_*)^2\}$ (the parameter $\barbSigma_1$ is unchanged). This ensures that state evolution is not stuck at a trivial fixed point. A practical alternative to assuming an informative initialization is to initialize AMP with a spectral estimator. Analyzing RI-GAMP  with  spectral initialization is an interesting direction for future research. 


\paragraph{Choice of denoisers.} The performance of RI-GAMP is determined by the functions $\{f_t, h_{t+1}\}_{t\ge 1}$. A key question is how to choose these functions to optimize the estimation performance.  Given any choice of $\{f_t, h_{t+1}\}$  satisfying Assumption \textbf{(A1)}, Theorem \ref{thm:main} implies
\begin{align}
 \label{eq:Rt1_dist}
 & \br^{t-1} \, \stackrel{\mathclap{W_2}}{\longrightarrow} \,  R_{t-1} \equiv \frac{(\bar{\bSigma}_t)_{t,1}}{(\bar{\bSigma}_t)_{1,1}} G  \,  + \,  W'_{t-1}, \nonumber\\
& W'_{t-1} \sim \normal\left( 0, \,  (\bar{\bSigma}_t)_{t,t} - \frac{((\bar{\bSigma}_t)_{t,1})^2}{(\bar{\bSigma}_t)_{1,1}}  \right), \quad t \ge 2, \\
 \label{eq:Xt_dist}
 & \bx^{t} \,   \stackrel{\mathclap{W_2}}{\longrightarrow} \,  X_t \equiv \barmu_t X_* \, + \, W_t, \nonumber \\
 & W_t \sim \normal(0, (\bar{\bOmega}_t)_{t,t}) \text{ independent of } X_*, \quad t \ge 1, 
\end{align} 
where the RHS of \eqref{eq:Rt1_dist} and \eqref{eq:Xt_dist} follow from the joint distributions specified in \eqref{eq:GRi_def}-\eqref{eq:Xi_def}. From \eqref{eq:Xt_dist}, we see that the quality of the estimate in each iteration $t$ is governed by the ratio $\barmu_t^2/(\bar{\bOmega}_t)_{t,t}$. Thus, having fixed $\{f_k, h_k\}_{k \le t-1}$, the Bayes-optimal choice for $h_t(r_1, \ldots, r_{t-1},  y)$ maximizes $\barmu_t^2/(\bar{\bOmega}_t)_{t,t}$. Similarly, given $\{f_k\}_{k \le t-1}$ and $\{h_k\}_{k \le t}$, the Bayes-optimal choice for $f_t(x_1, \ldots, x_t)$ maximizes the normalized squared correlation of $R_t$ and $G$, which is proportional to $\frac{(\bar{\bSigma}_{t+1})_{t+1,1}^2}{(\bar{\bSigma}_{t+1})_{t+1,t+1}}$. We remark that, even when the signal prior (i.e., the law of $X_*$) is known,  finding these optimal denoisers is challenging due to the complicated nature of the state evolution recursion \eqref{eq:GRi_def}-\eqref{eq:mut1_def}. However, for the special case of an i.i.d. Gaussian design, the state evolution is considerably simpler and the Bayes-optimal choices are (cf. Section 4.2 of \cite{feng2021unifying}):
\begin{align}
        & f_t(x_t) =c_1\mathbb E\{X_*\mid  X_t = x_t\}, \label{eq:BftG}\\
        & h_{t+1}(r_t, y) =\hspace{-.1em}c_2( \mathbb E\{G | R_t = r_t, Y = y\}\hspace{-.1em}-\hspace{-.1em}\mathbb E\{G | R_t = r_t\}),
        \nonumber
\end{align}
where $c_1, c_2$ are arbitrary non-zero constants. Here, for a general rotationally invariant $\bA$, we propose the following denoisers: 
\begin{align}
       &  f_t(x_1, \ldots, x_t)  =\mathbb E\{X_*\mid X_1 = x_1, \ldots, X_t = x_t\}, \label{eq:Bft}\\
       &  h_{t+1}(r_1, \ldots, r_t, y) \hspace{-.1em}=\hspace{-.1em}\mathbb E\{G | R_1 = r_1, \ldots, R_t = r_t, Y = y\} -\mathbb E\{G\mid R_1 = r_1, \ldots, R_t = r_t\}. \label{eq:Bht} 
\end{align}

For an i.i.d. Gaussian $\bA$, \eqref{eq:Bft}-\eqref{eq:Bht} reduce to \eqref{eq:BftG}, which is provably Bayes-optimal.
When $\bA$ is not Gaussian,  using denoisers that depend on all the preceding iterates (instead of only the most recent one)  can have a remarkable impact on the performance of RI-GAMP. In fact, in the setting of Section \ref{sec:simu} where the eigenvalues of $\bA$ follow a Beta distribution, taking \eqref{eq:BftG} does not improve much over the performance of the existing GAMP algorithm that assumes $\bA$ to be Gaussian (green curves in Figures \ref{fig:linfixdelta}-\ref{fig:1bitCS}). In contrast, taking \eqref{eq:Bft}-\eqref{eq:Bht} leads to a performance close to VAMP (blue curves). Though we do not expect the choices in \eqref{eq:Bft}-\eqref{eq:Bht} to be optimal iteration-by-iteration, based on the simulation results we conjecture that they achieve the same fixed point as the Bayes-optimal denoisers.

\begin{figure}[t]
    \centering
    \subfloat[Linear regression \label{fig:AMPSElin}]{\includegraphics[width=.47\columnwidth]{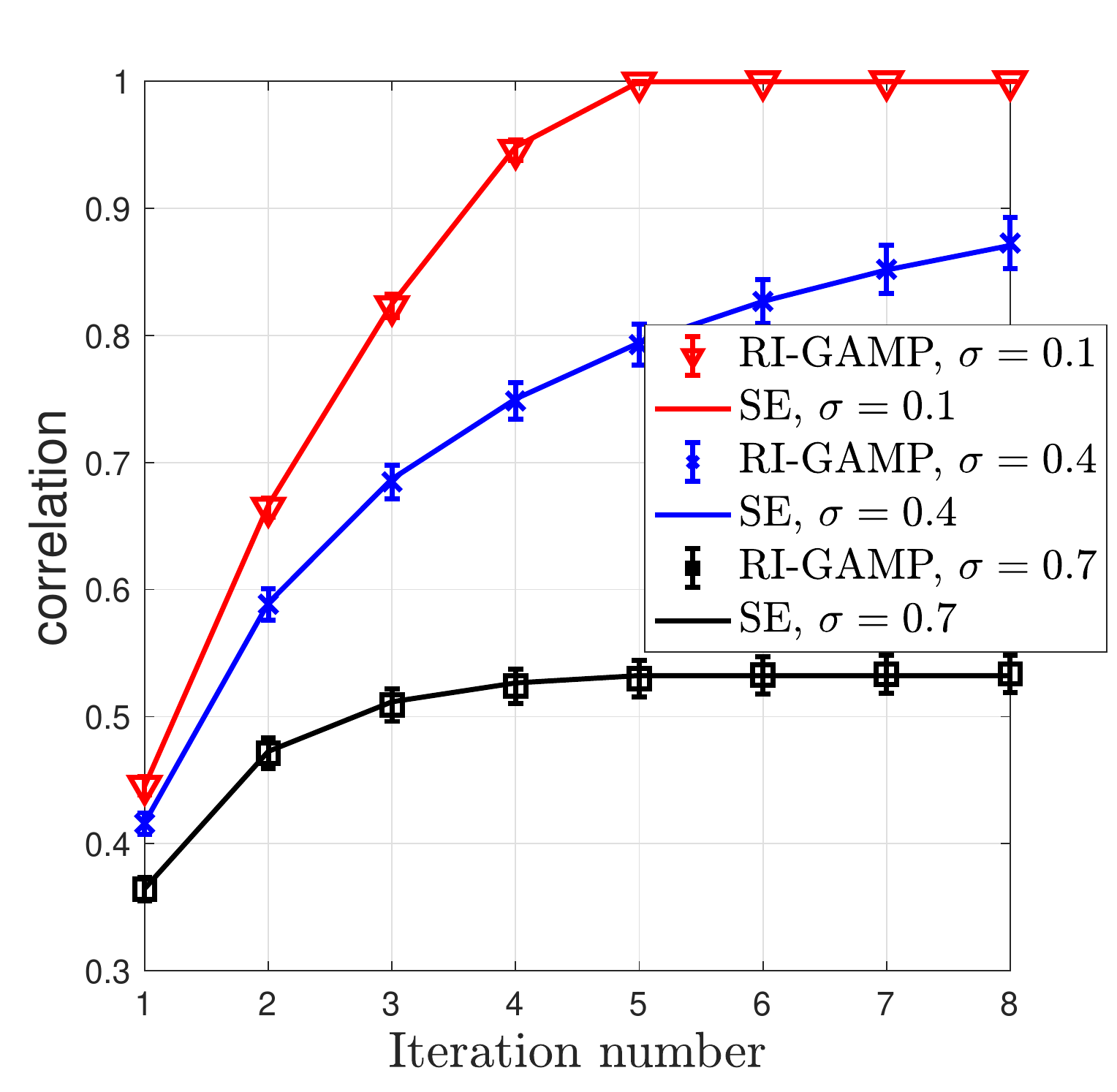}}
  \   \subfloat[1-bit    compressed sensing\label{fig:AMPSERad}]{\includegraphics[width=.47\columnwidth]{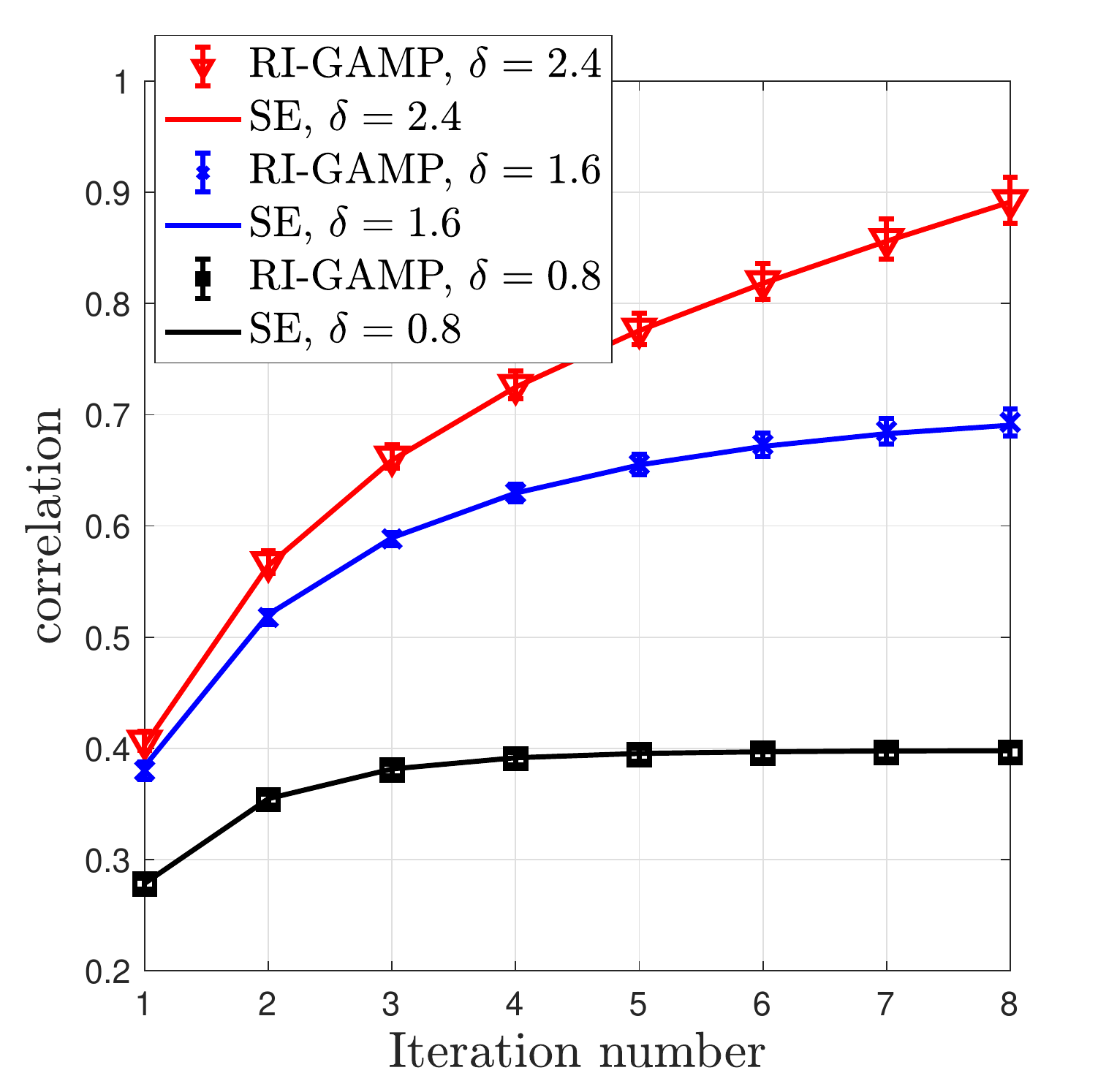}}
\caption{Normalized squared correlation between iterate $\hbx^t$ and signal $\bx^*$, as a function of the number of iterations $t$. Solid lines represent state evolution predictions, and the markers represent the empirical performance of RI-GAMP.}
\label{fig:AMPvsSE}
\end{figure}

\section{Numerical Simulations}\label{sec:simu}

For synthetic data, we consider two models: \emph{(i)} linear regression, i.e., $y_i = \langle \ba_i, \bx^*\rangle+\varepsilon_i$, with $\varepsilon_i\sim\normal(0, \sigma^2)$; \emph{(ii)} noiseless 1-bit compressed sensing, i.e., $y_i = {\rm sign}(\langle \ba_i, \bx^*\rangle)$. The design matrix $\bA$ is rotationally invariant in law, i.e., $\bA = \bO^\sT \bLambda \bQ$, where $\bO$, $\bQ$ are Haar orthogonal matrices, and $\bLambda$ has i.i.d. $\sqrt{6}\cdot {\rm Beta}(1, 2)$ diagonal entries. (The normalization of the ${\rm Beta}(1, 2)$ distribution is chosen to ensure a unit second moment.) In the simulations, the free cumulants $\kappa_{2k}$ are replaced by their limits $\bar{\kappa}_{2k}$, which can be obtained in closed-form (see Appendix \ref{subsec:freerect}). We set $d=8000$, repeat each experiment for $10$ independent runs, and  report the average and error bars at $1$ standard deviation.

We implement the RI-GAMP given in \eqref{eq:AMP_xt_update}-\eqref{eq:AMP_rt_update}, with initialization $\bs^1 = \by$ and $\bx^1 = \bA^\sT \bs^1$.
The denoisers $f_t$ and $h_{t+1}$, for $t\ge 1$, are given by \eqref{eq:Bft}-\eqref{eq:Bht}.  The expressions for these denoisers  and the associated calculations  are  given in Appendix \ref{sec:compBayes}. The denoisers $f_t, h_{t+1}$ and their derivatives depend on the state evolution parameters, which can be estimated consistently from the data. The implementation details are described at the end of Appendix \ref{sec:compBayes}.

\begin{figure}[t]
    \centering
    \subfloat[$\delta=1$\label{fig:delta1}]{\includegraphics[width=.47\columnwidth]{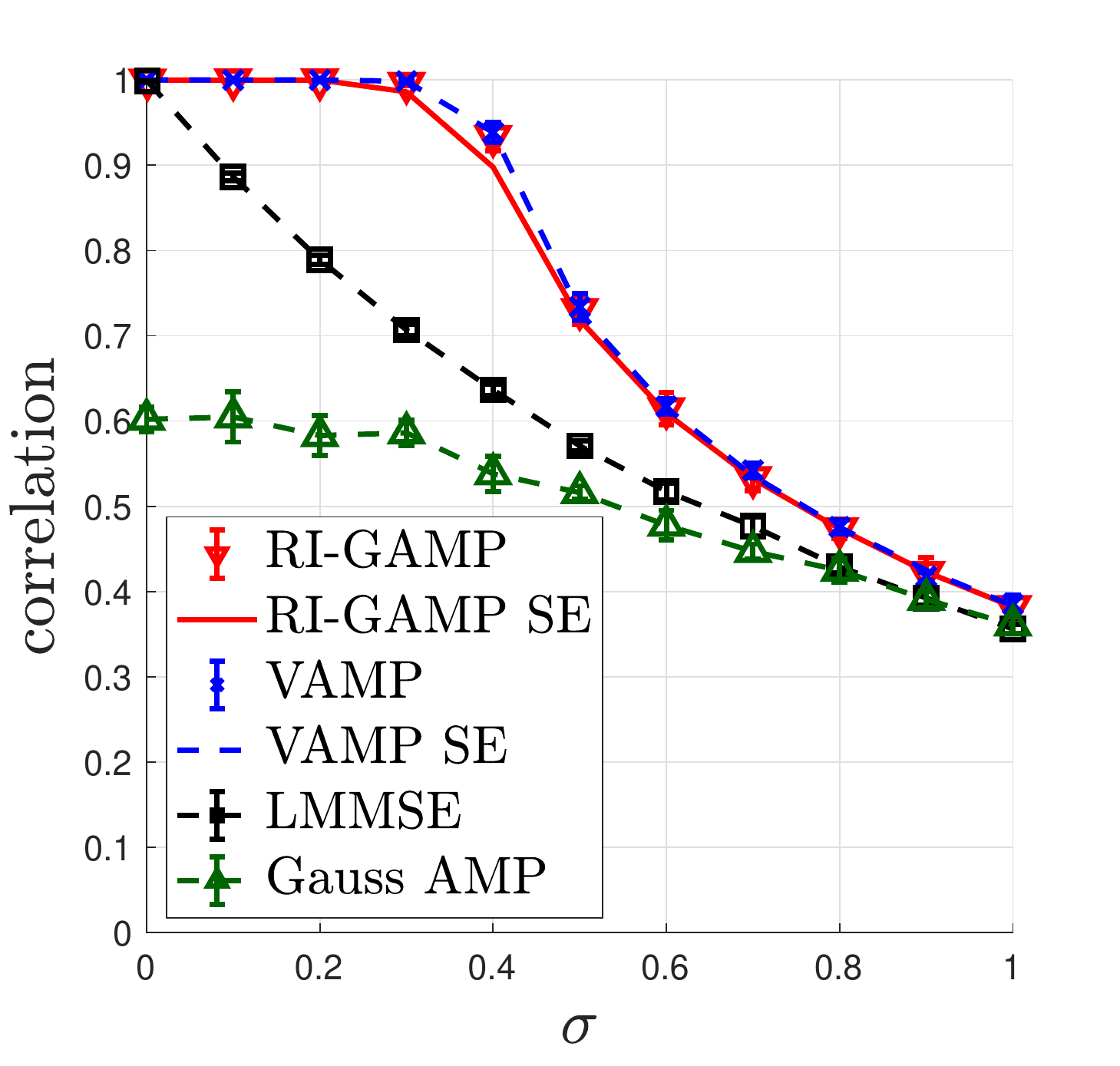}} \
            \subfloat[$\sigma=0.1$\label{fig:sigma0dot1}]{\includegraphics[width=.47\columnwidth]{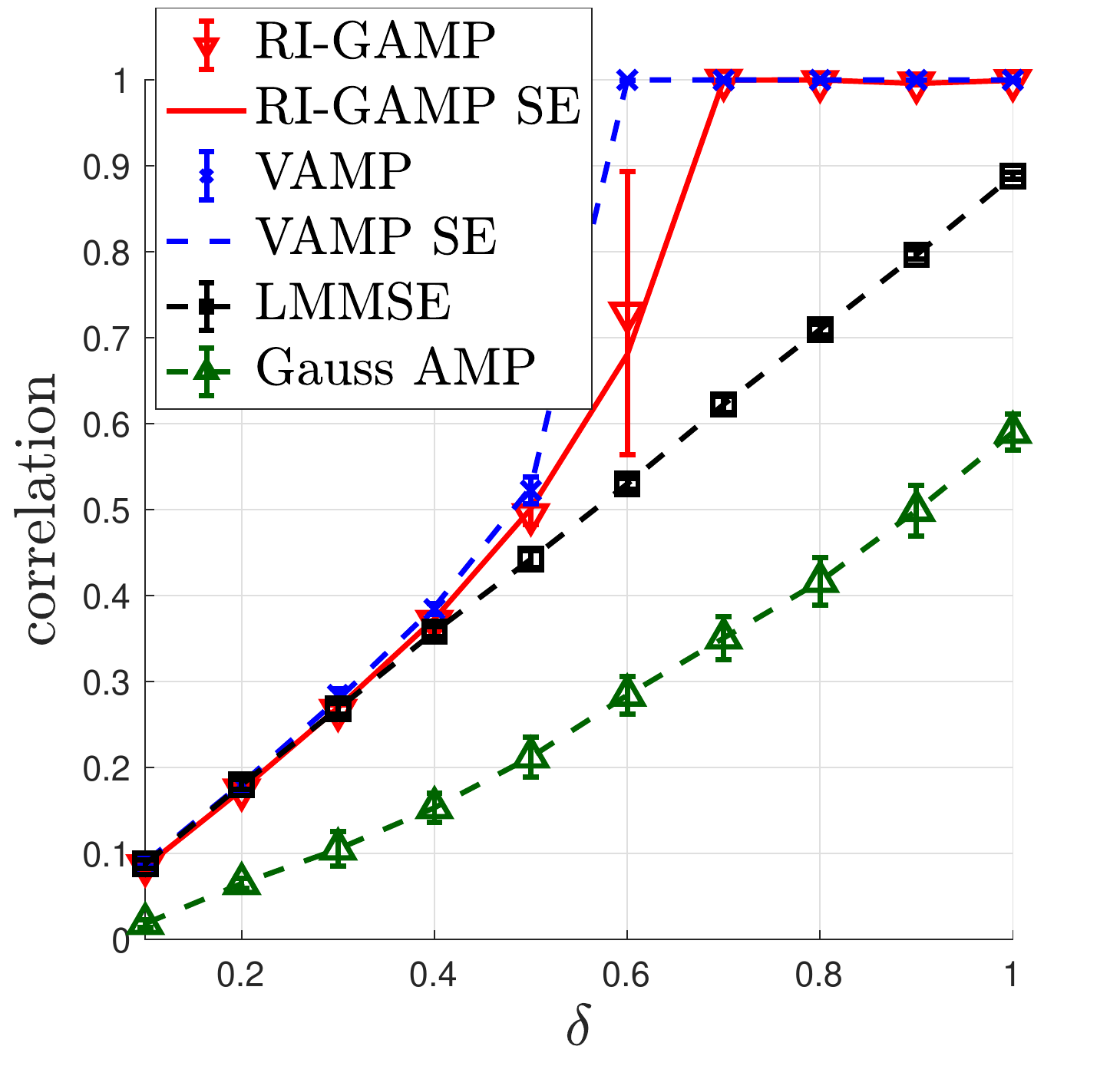}}
\caption{Linear regression with a Rademacher prior: normalized squared correlation vs. noise level $\sigma$ (on the left) and vs. aspect ratio $\delta$ (on the right).}
\label{fig:linfixdelta}
\end{figure}

Figure \ref{fig:AMPvsSE} shows that the state evolution predictions closely match the performance of RI-GAMP for practical values of $d$ and $n$, validating the result of Theorem \ref{thm:main}. We plot the normalized squared correlation $\langle \hbx^t, \bx^*\rangle^2/(\|\hbx^t\|^2 \|\bx^*\|^2)$ as a function of the iteration number. In (a), we consider linear regression with a Rademacher prior, $\delta=1$ and $\sigma\in \{0.1, 0.4, 0.7\}$; and in (b), noiseless 1-bit compressed sensing with a Rademacher prior and $\delta\in \{0.8, 1.6, 2.4\}$. 
In all cases, the agreement between RI-GAMP and its SE is excellent.  The next two figures show that the performance of RI-GAMP closely matches that of VAMP in a variety of settings. The results for linear regression with a Rademacher prior are shown in Figure \ref{fig:linfixdelta}: 
on the left, we plot the normalized squared correlation as a function of $\sigma$, for $\delta=1$; on the right, we plot the same metric as a function of $\delta$, for $\sigma=0.1$. Additional results for a different choice of $\delta$ and $\sigma$ are reported in Figure \ref{fig:linfixsigma} in Appendix \ref{app:addnum}. Figure \ref{fig:1bitCS} shows the performance for noiseless 1-bit compressed sensing: we plot  the normalized squared correlation as a function of $\delta$, for two signal priors (Rademacher in (a), and Gaussian in (b)). The \emph{red curve} in each plot corresponds to \emph{RI-GAMP}, together with the related SE.  The \emph{blue curve} corresponds to \emph{VAMP}, together with the related SE. The implementation details for VAMP are given at the end of Appendix \ref{sec:compBayes}. The \emph{green curve} corresponds to the standard \emph{GAMP} algorithm which is derived based on the (incorrect) assumption that $\bA$ is i.i.d. Gaussian. The denoisers $f_t$ and $h_{t+1}$ are given by \eqref{eq:BftG}, which would be Bayes-optimal were the design matrix $\bA$ Gaussian. 
The implementation of GAMP is a special case of our proposed RI-GAMP (obtained by setting all the rectangular free cumulants except $\bar{\kappa}_2$ to $0$).  The GAMP state evolution predictions (not shown in the plots) do not match the performance of the algorithm, since $\bA$ is not Gaussian. Finally, the \emph{black curve} corresponds to \emph{(i)} the \emph{linear minimum mean squared error} (LMMSE) estimator $\hat{\bx} = \bA^\sT (\bA \bA^\sT+\sigma \bI)^{-1} \by$ for linear regression, and \emph{(ii)} a \emph{subgradient} method for 1-bit compressed sensing. This last method minimizes $\norm{[y \odot {\rm sign}(\bA \bx)]_-}_1$ via subgradient descent (here, $\odot$ denotes the Hadamard product and $[a]_- = \max\{-a, 0\}$ is applied component-wise).  The  algorithm was proposed in  \cite{jacques2013robust}  for the recovery of sparse signals, and the original version includes a sparsity enforcing step. For our setup (with no sparsity), we run it without the sparsity enforcing step, and the method reduces to subgradient descent.

\begin{figure}[t]
    \centering
    \subfloat[Rademacher prior\label{fig:1bitCSRad}]{\includegraphics[width=.47\columnwidth]{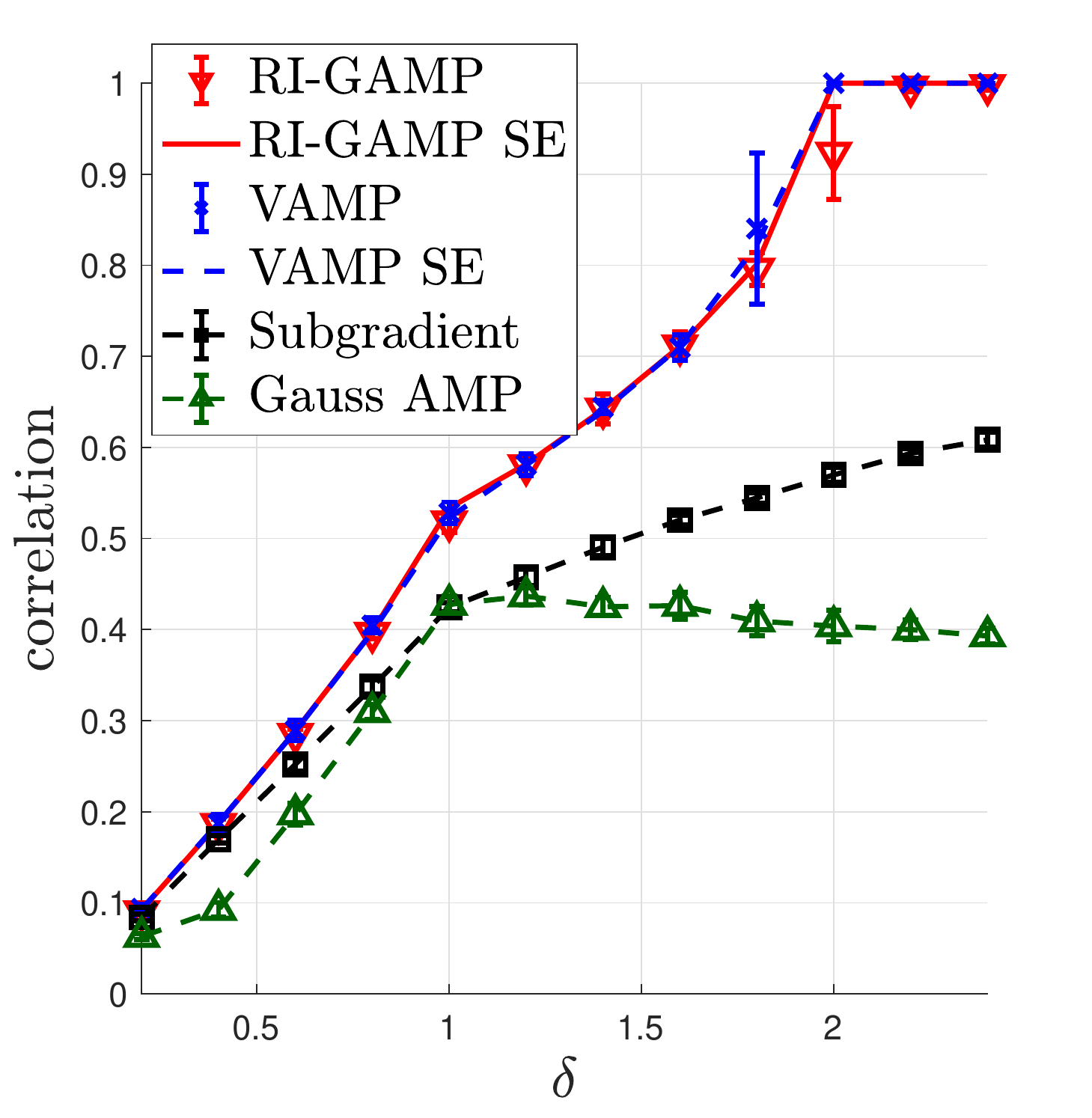}}
    \     \subfloat[Gaussian prior\label{fig:1bitCSGauss}]{\includegraphics[width=.47\columnwidth]{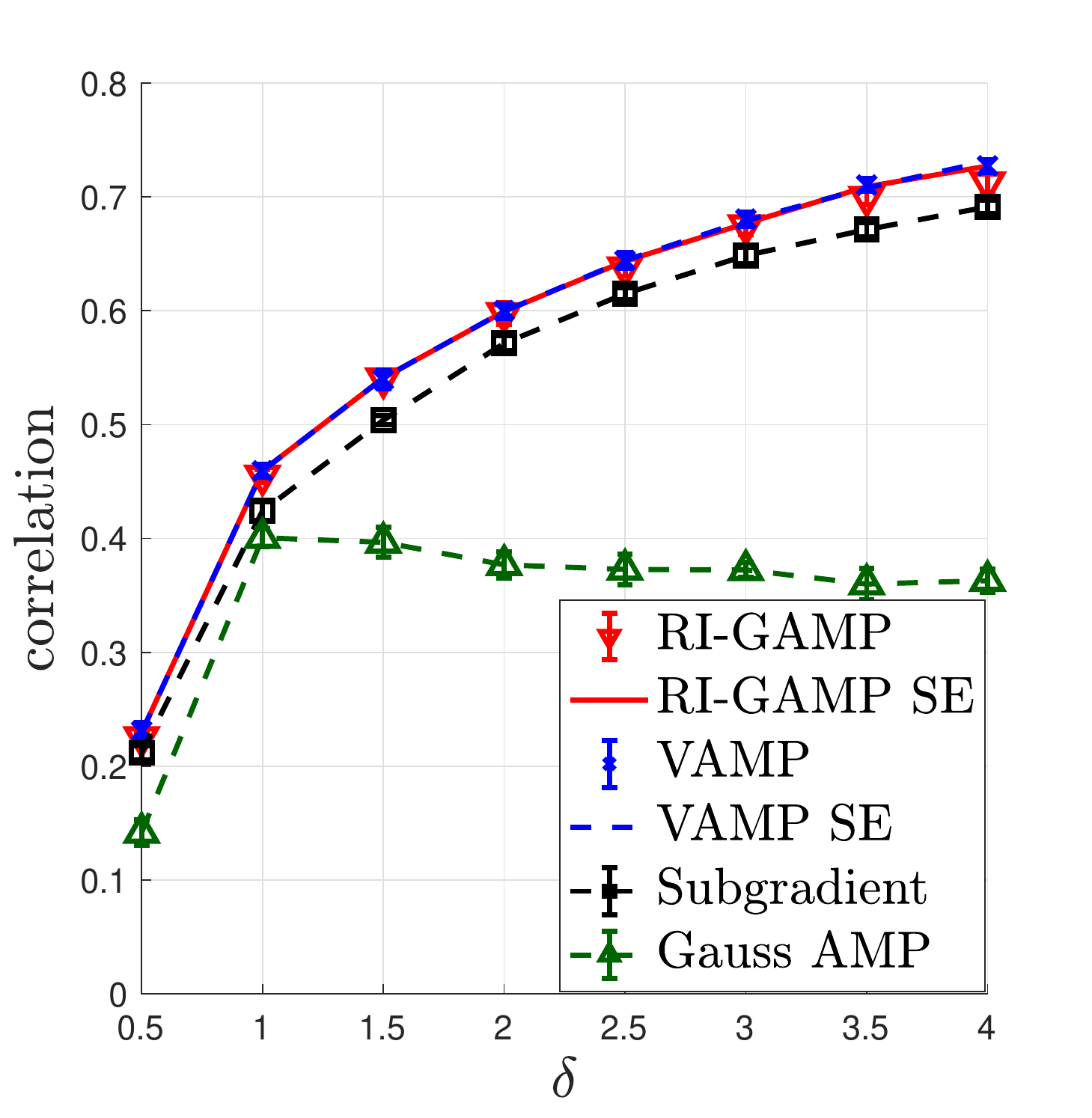}}
\caption{Noiseless 1-bit compressed sensing: normalized squared correlation vs. aspect ratio $\delta$, for two priors.}
\label{fig:1bitCS}
\end{figure}

\begin{figure}[t]
    \centering
    \subfloat[$d=2000$]{\includegraphics[width=0.33\columnwidth]{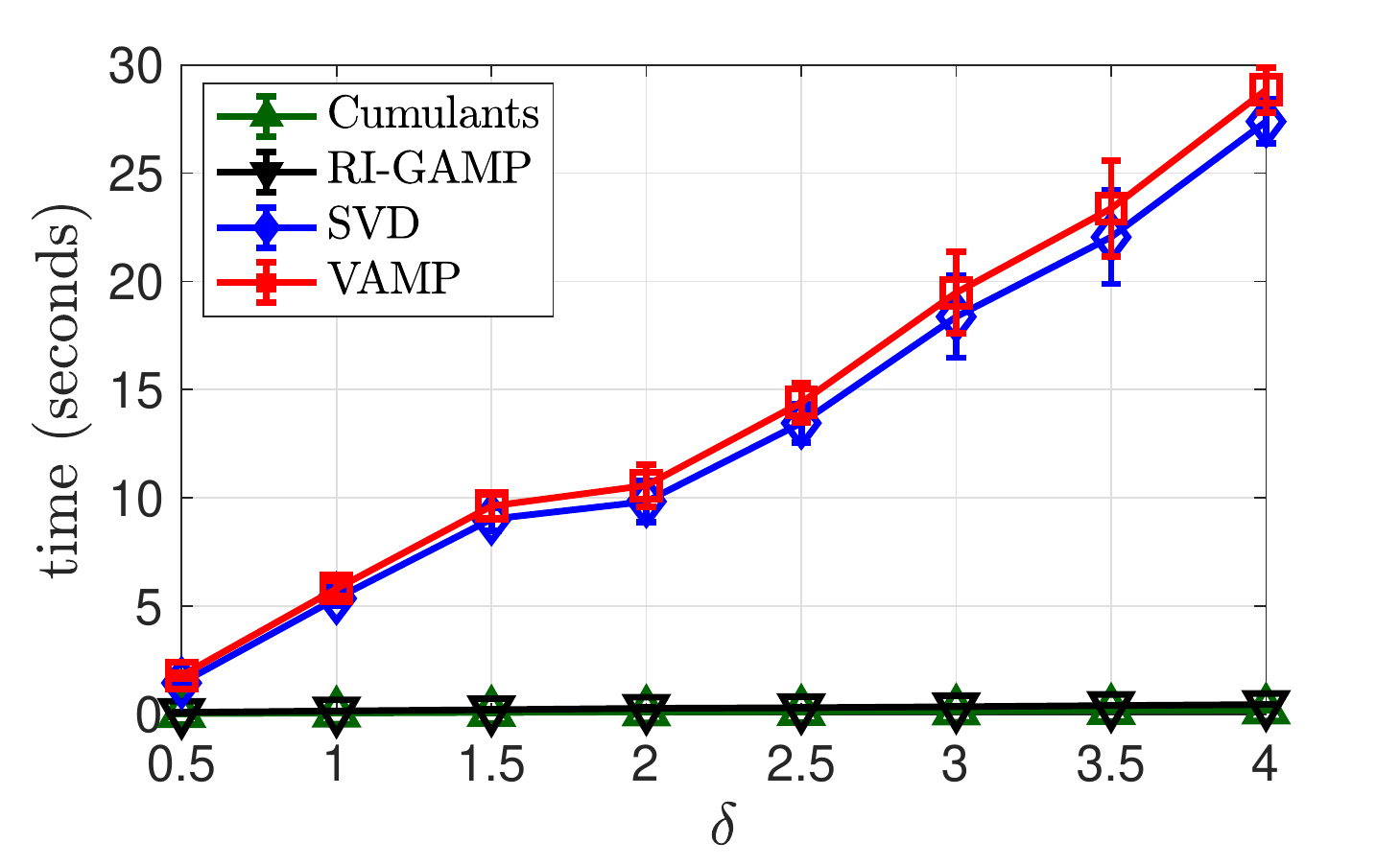}}
        \subfloat[$d=4000$]{\includegraphics[width=.33\columnwidth]{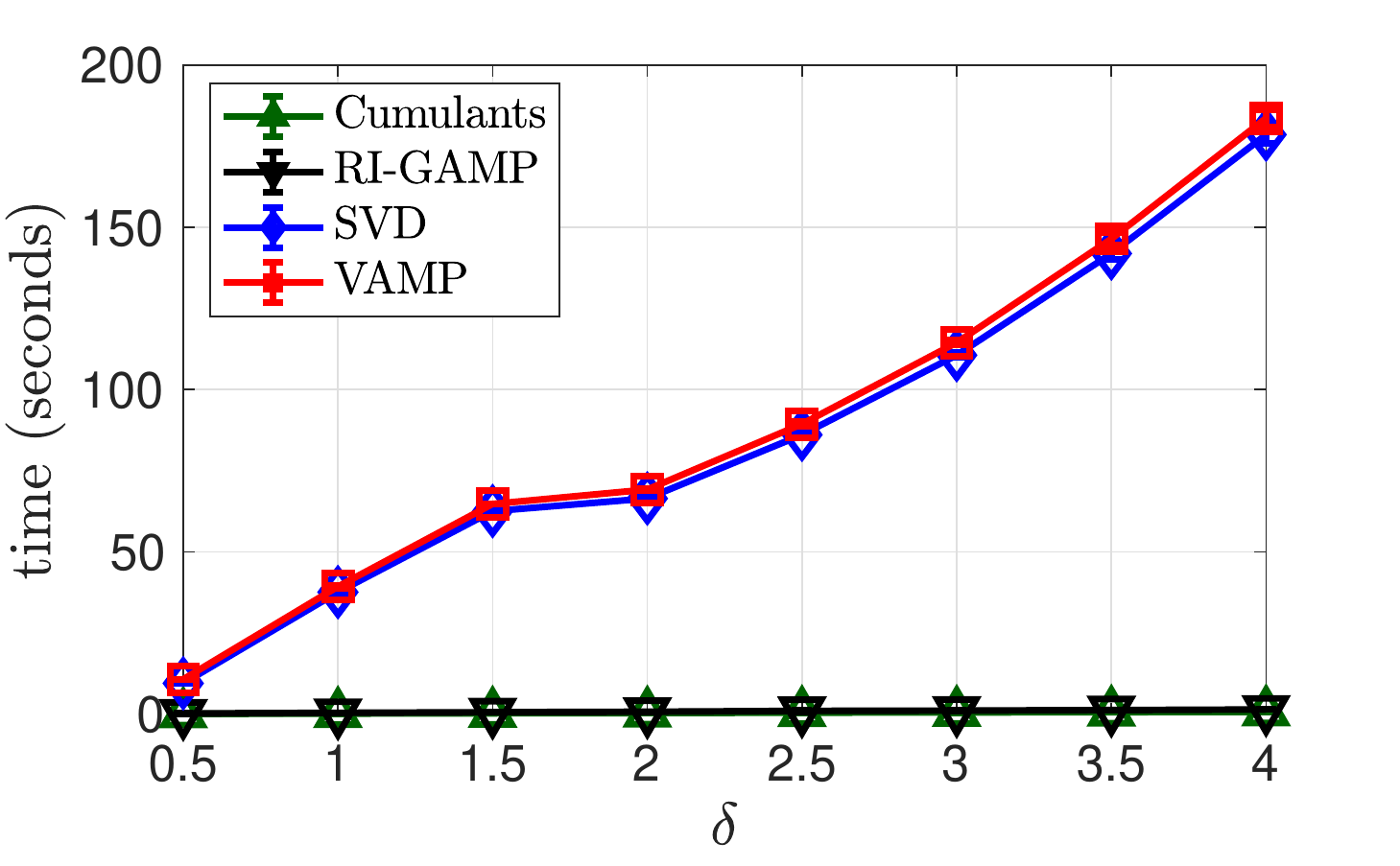}}
        \subfloat[$d=8000$]{\includegraphics[width=.33\columnwidth]{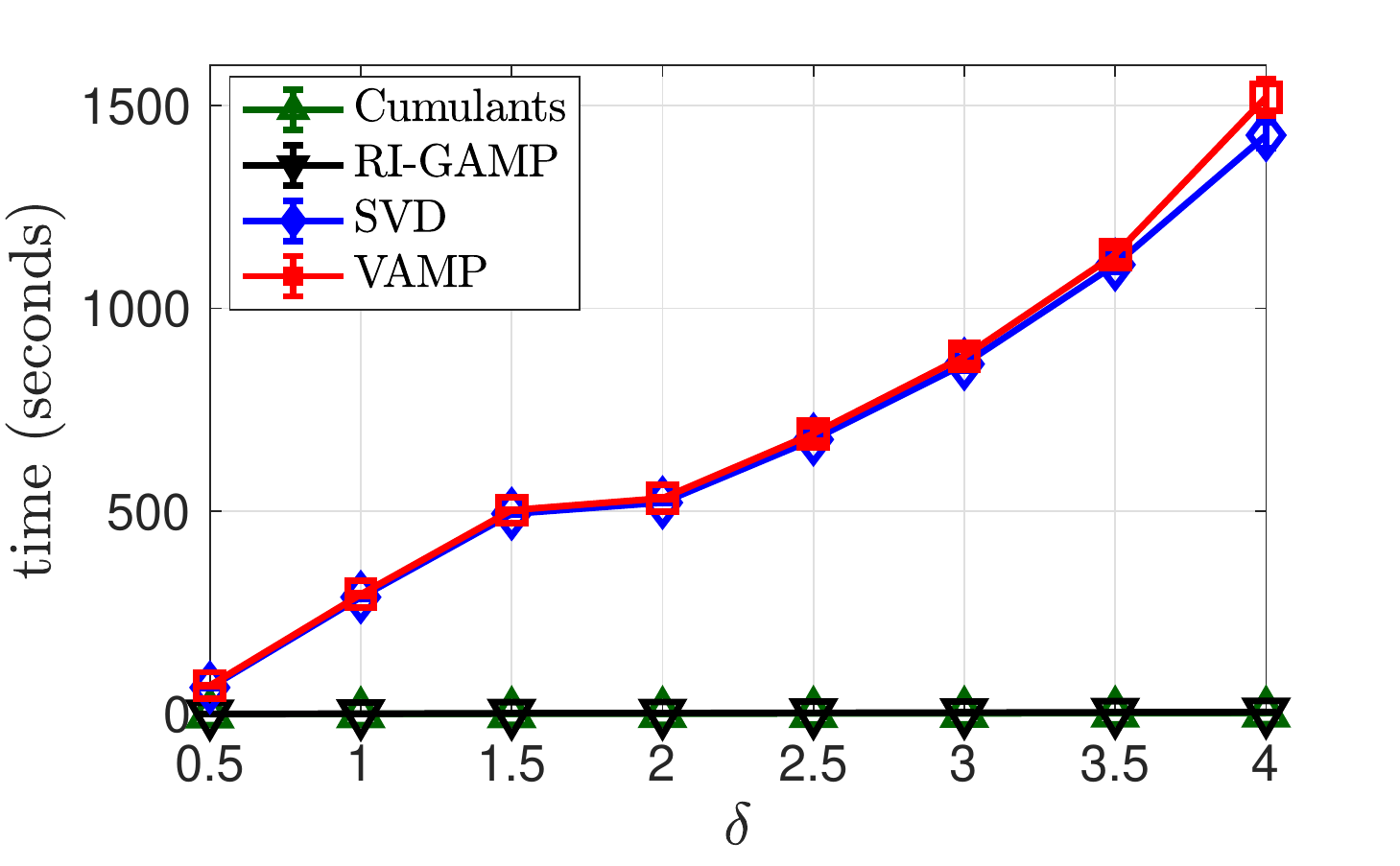}}
   \caption{Running time of RI-GAMP vs. VAMP. The green curve shows the time for computing the free cumulants in RI-GAMP, and the blue curve  the time for the SVD in the initial step of VAMP.}
\label{fig:running_time_comp}
\end{figure}

\paragraph{Performance of RI-GAMP vs. VAMP.} Taking the results of Figures \ref{fig:linfixdelta}-\ref{fig:1bitCS} together, we highlight that RI-GAMP exhibits a performance close to VAMP. Recall that the fixed points of the VAMP state evolution satisfy the replica equation, whose solution is conjectured to give the Bayes-optimal mean squared error; see Theorem 3 of \cite{rangan2019vector} for linear regression, and Theorem 4 of \cite{pandit2020inference} for the general case.  Thus, this conjecture implies that VAMP is Bayes-optimal when its state evolution has a unique fixed point, and our numerical simulations suggest that RI-GAMP is also near-optimal. We remark that, close to the phase transition for exact recovery (see 
$\delta\approx 0.6$ in Figure \ref{fig:sigma0dot1}, and $\delta\approx 1.8$ in Figure \ref{fig:1bitCSRad}), there is a (small) performance gap between RI-GAMP and VAMP. This is due to the fact that, for both RI-GAMP and VAMP, the number of iterations needed for convergence grows when the algorithm operates close to this phase transition. We note that VAMP shows a larger error bar for $\delta\approx 1.8$ in Figure \ref{fig:1bitCSRad}. For RI-GAMP, the issue is that the expressions of $f_{t}$ and $h_{t+1}$ depend on covariance matrices whose dimension grows with $t$. For large $t$, these covariance matrices become ill-conditioned and RI-GAMP is unstable. However, the stability displayed by VAMP comes at the cost of requiring the computationally expensive SVD of $\bA$. 
As shown next, RI-GAMP is significantly faster than VAMP, and is therefore an appealing alternative in many practical settings.

\paragraph{Complexity of RI-GAMP vs. VAMP.} The computational complexity of VAMP is dominated by the initial SVD which has $O(d^3)$ running time.  In contrast, the free cumulants required for RI-GAMP can be estimated in $O(d^2)$ time, as described on p.\pageref{para:freecum_est}. Each iteration of RI-GAMP also takes $O(d^2)$ time, and the algorithm typically converges in a few tens of iterations.

Figure \ref{fig:running_time_comp}
shows the running times for VAMP (including the initial SVD) and RI-GAMP (including the estimation of the free cumulants from the data), for noiseless 1-bit compressed sensing. The running time of VAMP is dominated by the SVD, and the computational advantage of RI-GAMP increases quickly with the problem dimension. For $\delta := \frac{n}{d} \in [0.5, 4]$, RI-GAMP is $20$-$60\times$ faster than VAMP at $d=2000$, $40$-$120\times$ faster at $d=4000$, and $80$-$240\times$ faster at $d=8000$.

\paragraph{Impact of eigenvalue distribution and prior.}  RI-GAMP exploits the spectral distribution of $\bA$, which gives a large performance improvement over the GAMP designed for a Gaussian $\bA$. Furthermore, RI-GAMP  also takes advantage of   the signal prior, which cannot be exploited by either the LMMSE estimator (for linear regression) or the subgradient method (for 1-bit compressed sensing). In fact, the subgradient method has roughly the same correlation for the two choices of the prior (cf. the black curves in Figure \ref{fig:1bitCSRad} and \ref{fig:1bitCSGauss}), and  is outperformed by RI-GAMP in both settings.

\paragraph{1-bit compressed sensing on a sparse image.} In Figure \ref{fig:satimgfig}, we consider noiseless 1-bit compressed sensing with the input $\bx^*$ being the sparse grayscale image considered in \cite{schniter2014compressive}, with $d=225^2=50625$ and a sparsity (fraction of non-black pixels) of $8645/50625$. The design matrix $\bA$ is  $\bA=\bQ_n\boldsymbol{\Pi}_n\bLambda \boldsymbol{\Pi}_d \bQ_d$, where $\bQ_n$, $\bQ_d$ are orthonormal Discrete Cosine Transform (DCT) matrices in $n$, $d$ dimensions, $\boldsymbol{\Pi}_n, \boldsymbol{\Pi}_d$ are random permutation matrices, and $\bLambda$ has  i.i.d. $\sqrt{6}\cdot {\rm Beta}(1, 2)$ diagonal entries. This choice of $\bA$ significantly speeds up matrix multiplications, 
as in \cite{tian2021generalized}. We report the average and error bars at 1 standard deviation for 100 independent trials. For RI-GAMP, we use a non-negative Bernoulli-Gaussian prior (cf. \cite{schniter2014compressive} and \cite{vila2014empirical}); the expression for the corresponding denoiser $f_t$ is in Appendix \ref{sec:compBayes}.  As shown in Figure \ref{fig:corrsat}, RI-GAMP improves on the subgradient method in \cite{jacques2013robust} up until $\delta = 1.5$ (this improvement is clearly visibile in the reconstructions for $\delta=0.8$, see Figures \ref{fig:amprec} and \ref{fig:gdrec}). For larger  $\delta$, the performance of RI-GAMP does not improve further, due to the aforementioned 
numerical instabilities. 
Additional experiments on RGB images when the input $\bx^*$ is obtained via a wavelet transform are reported in Appendix \ref{app:addnum}.

\begin{figure}[t]
    \centering
    \subfloat[Correlation vs. $\delta$\label{fig:corrsat}]{\includegraphics[width=.25\columnwidth]{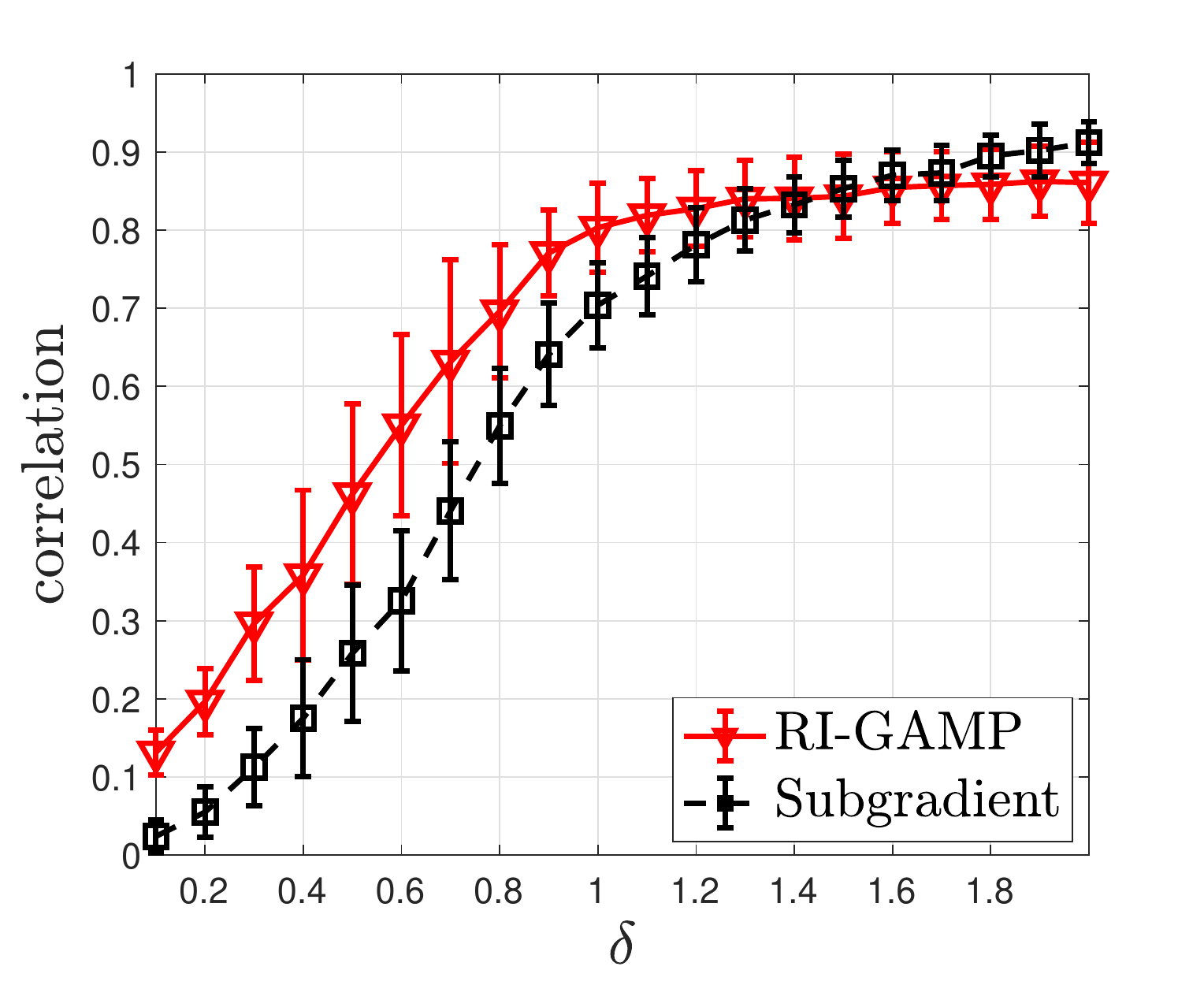}}
        \subfloat[Original image\label{fig:orim}]{\includegraphics[width=.25\columnwidth]{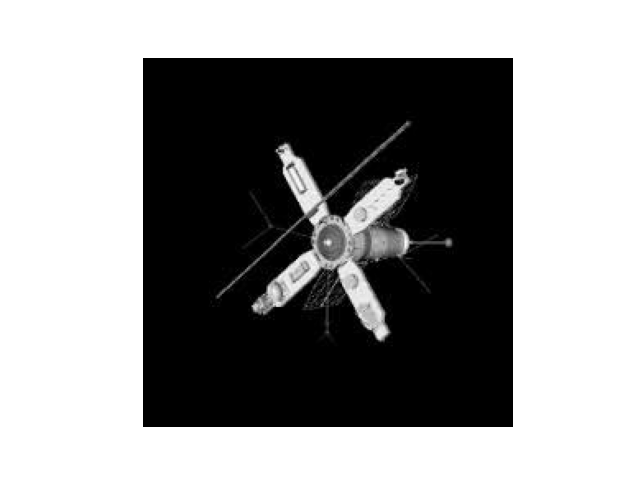}}
        \subfloat[RI-GAMP, $\delta=0.8$\label{fig:amprec}]{\includegraphics[width=.25\columnwidth]{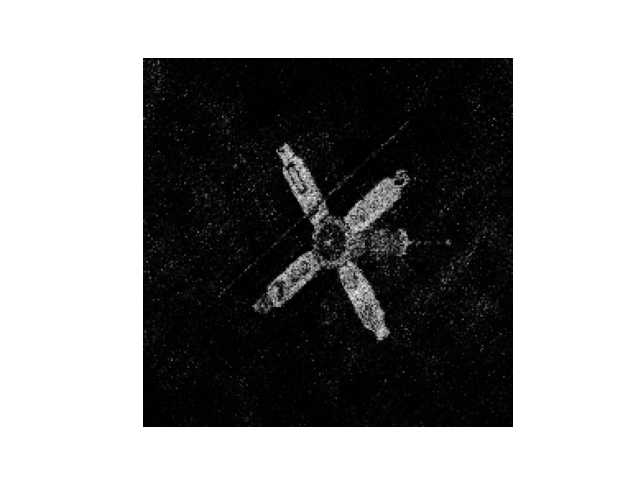}}
        \subfloat[Subgradient,  $\delta=0.8$\label{fig:gdrec}]{\includegraphics[width=.25\columnwidth]{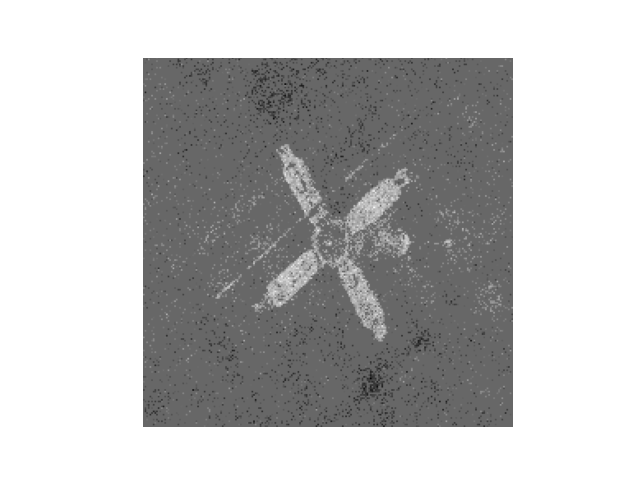}}
\caption{RI-GAMP versus subgradient method for the recovery of a sparse image from 1-bit measurements.}
\label{fig:satimgfig}
\end{figure}

\section{Proof Sketch of Theorem \ref{thm:main}} \label{sec:proof_sketch}

The proof is based on an auxiliary AMP algorithm whose iterates mimic the true AMP in \eqref{eq:AMP_xt_update}-\eqref{eq:AMP_rt_update}.  The iterates of the auxiliary AMP, denoted by $\bz^t, \bv^t \in \reals^d$ and $\bv^t, \bu^{t+1} \in \reals^n$, are computed as follows, for $t \ge 1$:
\vspace{-7pt}
\begin{align}
    \bz^t & = \bA^{\sT} \bu^t   -  \sum_{i=1}^{t-1} \ssb_{ti} \bv^i,
    \qquad  \bv^t = \tf_t(\bz^1, \ldots, \bz^t, \, \bx^*),\hspace{-1em} \label{eq:auxAMPz} \\
\hspace{-.5em}    \bm^{t} &\hspace{-.2em} = \hspace{-.2em}\bA \bv^{t} \hspace{-.2em} - \hspace{-.2em}\sum_{i=1}^{t}\hspace{-.2em} \sa_{ti} \bu^i, \qquad  \bu^{t+1} \hspace{-.2em}=\hspace{-.2em} \th_{t+1}(\bm^1, \ldots, \bm^t\hspace{-.2em}, \bveps).\hspace{-.5em}
    \label{eq:auxAMPm}
\end{align}
The iteration is initialized with $\bu^1=\bzero, \quad \bz^1 =\bzero$. The function $\tf_1: \reals^2 \to \reals$ is defined as $\tf_1(z_1, x) = x$, which yields $\bv^1= \bx^*$ and 
$\bm^1=\bA \bx^* = \bg$.  
For $t \ge 1 $, the functions $\tf_{t+1}: \reals^{t+2} \to \reals$ and $\th_{t+1}: \reals^{t+1} \to \reals$, which act row-wise on matrices, are defined as
\begin{align}
            & \tf_{t+1}(z_1, \ldots, z_{t+1}, \, x)  = f_t(z_2 + \bar{\mu}_1 x, \,  z_3 + \bar{\mu}_2 x, \,  \ldots, \, z_{t+1} + \bar{\mu}_t x),   \label{eq:tf_defs} \\
    & \th_{t+1}(m_1, \, \ldots, m_t,  \, \veps) = h_t(m_2, \ldots, m_t, \, q(m_1, \veps)). \nonumber 
\end{align}
Here, $f_t, h_t$ are the functions from the original AMP, and 
$(\barmu_1, \ldots, \barmu_t)$ are the state evolution parameters computed according to \eqref{eq:mut1_def}. The debiasing coefficients $\{\sa_{ti}\}_{i=1}^t$ and $\{\ssb_{ti}\}_{i=1}^{t-1}$
for the auxiliary AMP are given in Appendix \ref{sec:debiasing_aux}.

\paragraph{Idea of the proof.} 
The auxiliary AMP \eqref{eq:auxAMPz}-\eqref{eq:auxAMPm} is an instance of the abstract AMP recursion for non-symmetric rotationally invariant matrices, which was analyzed in \cite{zhong2021approximate}. The state evolution result in Theorem 2.6 of \cite{zhong2021approximate} implies that the joint empirical distribution of $(\bm^1, \ldots, \bm^t)$ converges to a $t$-dimensional Gaussian $\normal(\bzero, 
\tbSigma_t)$. Similarly, the joint empirical distribution of $(\bz^1, \ldots, \bz^t)$ also converges to a a $t$-dimensional Gaussian $\normal(\bzero, \tbOmega_t)$. The covariance matrices $\tbSigma_t, \tbOmega_t$ are recursively defined via the state evolution for the auxiliary AMP (for details, see Section \ref{sec:SE_aux}).  

The proof of Theorem \ref{thm:main} consists of two steps. First, we show  that the state evolution parameters of the auxiliary AMP match those of the true AMP. In particular,  Lemma \ref{lem:SE_equiv} proves that  $\tbSigma_t = \barbSigma_t$ and $\tbOmega_{t+1} = \bOmega'_{t+1}$, where the matrices on the right are defined in \eqref{eq:barSigma_def} and \eqref{eq:Omega_pr_def}. Next, we show in Lemma \ref{lem:PLdiff} that the true  AMP iterates \eqref{eq:AMP_xt_update}-\eqref{eq:AMP_rt_update}  are close to the auxiliary AMP iterates \eqref{eq:auxAMPz}-\eqref{eq:auxAMPm} in the following sense. For $t \ge 1$:
\begin{align}
&  \frac{\| \bx^t - (\bz^{t+1} + \barmu_t \bx^*) \|^2}{d} \to 0, \ \ 
\frac{\| \hbx^{t} - \bv^{t+1} \|^2}{d} \to 0, \nonumber \\
& \frac{\| \br^t - \bm^{t+1} \|^2}{n}\to 0, \ \ 
 \frac{\| \bs^t - \bu^{t+1} \|^2}{n} \to 0.
\label{eq:aux_true_conv}
\end{align}
Lemma \ref{lem:PLdiff} actually proves a more general convergence statement which implies \eqref{eq:aux_true_conv}. It shows that the empirical joint distribution of the iterates of the true AMP converges to that of the auxiliary AMP. The result of Theorem \ref{thm:main} then follows from Lemmas \ref{lem:SE_equiv} and \ref{lem:PLdiff}.

\section*{Acknowledgements}

The authors would like to thank the anonymous reviewers for their helpful comments. KK and MM were partially supported by the 2019 Lopez-Loreta Prize.

{\small{
\bibliographystyle{amsalpha}
\bibliography{paper_arxiv}
}}

\vspace{2in}

\appendix

\section{Background on Rectangular Free Cumulants}\label{subsec:freerect}

Let $X$ be a random variable of finite moments of all orders, and denote its even moments by $m_{2k} =\mathbb E\{X^{2k}\}$. In this paper, $X^2$ represents either the empirical eigenvalue distribution of $\bA\bA^\sT\in\mathbb R^{n\times n}$, or its limit law $\Lambda^2$ (in the latter case, the moments and rectangular free cumulants are denoted by $\{\bar{m}_{2k}\}_{k\ge 1}$ and $\{\bar{\kappa}_{2k}\}_{k\ge 1}$, respectively).
The rectangular free cumulants $\{\kappa_{2k}\}_{k\ge 1}$ of $X$ are defined recursively by the moment-cumulant relations (cf. Section 3 of \cite{benaych2009rectangular})
\begin{equation}\label{eq:mcrel1rect}
    m_{2k} =\delta \sum_{\pi\in {\rm NC}'(2k)}\prod_{\substack{S\in \pi\\\min S \,\,\,\mbox{\scriptsize is odd}}}\kappa_{|S|}\prod_{\substack{S\in \pi\\\min S \,\,\,\mbox{\scriptsize is even}}}\kappa_{|S|},
\end{equation}
where ${\rm NC}'(2k)$ is the set of non-crossing partitions $\pi$ of $\{1, \ldots, 2k\}$ such that each set $S\in \pi$ has even cardinality. Furthermore, by exploiting the connection between the formal power series with coefficients $\{m_{2k}\}_{k\ge 1}$ and $\{\kappa_{2k}\}_{k\ge 1}$, each rectangular free cumulant $\kappa_{2k}$ can be computed from $m_2, \ldots, m_{2k}$ and $\kappa_2, \ldots, \kappa_{2(k-1)}$ as (cf. Lemma 3.4 of \cite{benaych2009rectangular})
\begin{equation}\label{eq:mcrel2rect}
    \kappa_{2k} = m_{2k} - [z^k]\sum_{j=1}^{k-1}\kappa_{2j}\left(z(\delta M(z)+1)(M(z)+1)\right)^j,
\end{equation}
where $M(z)=\sum_{k=1}^\infty m_{2k}z^k$ and $[z^k](q(z))$ denotes the coefficient of $z^k$ in the polynomial $q(z)$.

In the numerical simulations of Section \ref{sec:simu}, the singular values of $\bA$ are i.i.d. $\sqrt{6}\cdot {\rm Beta}(1, 2)$. Hence, for $\delta\in (0, 1)$, $X$ has distribution $\sqrt{6}\cdot {\rm Beta}(1, 2)$ and consequently $\bar{m}_{2k} = \frac{6^k}{(k+1)(2k+1)}$; for $\delta\ge 1$, $X$ has distribution $\sqrt{6}\cdot {\rm Beta}(1, 2)$ w.p. $1/\delta$ and it is equal to $0$ w.p. $1-1/\delta$, and consequently $\bar{m}_{2k} = \frac{1}{\delta}\frac{6^k}{(k+1)(2k+1)}$. Then, given the moments $\{\bar{m}_{2k}\}_{k\ge 1}$, the rectangular free cumulants $\{\bar{\kappa}_{2k}\}_{k\ge 1}$ are computed recursively using \eqref{eq:mcrel2rect}.

\section{Computation of Denoisers, and Implementation Details}\label{sec:compBayes}


\paragraph{Computation of $f_t$ for Rademacher prior.} Here, $\mathbb P(X_*=1)=\mathbb P(X_*=-1)=1/2$. Hence, \eqref{eq:Bft} can be specialized as:
\begin{equation}\label{eq:ftcomp1}
    f_t(x_1, \ldots, x_t) = \E\{ X_* \mid X_1 = x_1, \ldots, X_t = x_t  \} = 2\cdot \mathbb P(X_*=1 \mid X_1 = x_1, \ldots, X_t = x_t)-1.
\end{equation}
From \eqref{eq:Xi_def}, we have that $(X_1, \ldots, X_t) = \barbmu_t X+(W_1, \ldots, W_t)$, with $(W_1, \ldots, W_t)\sim \normal(\bzero, \barbOmega_t)$. Thus,
\begin{equation}\label{eq:ftcomp2}
\begin{split}
   & P(X_*=1 \mid X_1 = x_1, \ldots, X_t = x_t) \\
    & = \frac{\exp\left(\displaystyle\frac{-(\bx-\barbmu_t)^\sT(\barbOmega_t)^{-1}(\bx-\barbmu_t)}{2}\right)}{\exp\left(\displaystyle\frac{-(\bx-\barbmu_t)^\sT(\barbOmega_t)^{-1}(\bx-\barbmu_t)}{2}\right)+\exp\left(\displaystyle\frac{-(\bx+\barbmu_t)^\sT(\barbOmega_t)^{-1}(\bx+\barbmu_t)}{2}\right)},
    \end{split}
\end{equation}
where $\bx=(x_1, \ldots, x_t)^{\sT}$. (All vectors in this section, including $\bx$ and $\barbmu_t$, are treated as column vectors, unless otherwise mentioned.)
Combining \eqref{eq:ftcomp1} and \eqref{eq:ftcomp2}, we obtain 
\begin{equation}\label{eq:ftRad}
        f_t(x_1, \ldots, x_t)  = \tanh\left(\barbmu_t^\sT(\barbOmega_t)^{-1}\bx\right).
\end{equation}
Furthermore, the partial derivatives of $f_t$ can be expressed in the following compact form:
\begin{equation}\label{eq:dftRad}
    \partial_{x_i}f_t(x_1, \ldots, x_t) = \left(1-\tanh^2\left(\barbmu_t^\sT(\barbOmega_t)^{-1}\bx\right)\right)\barbmu_t^\sT(\barbOmega_t)^{-1}\be_i, \quad \mbox{ for }i\in [t],
\end{equation}
where $\be_i$ is the vector corresponding to the $i$-th element of the canonical basis of $\mathbb R^t$.

\paragraph{Computation of $f_t$ for Gaussian prior.} Here, $X_*\sim\normal(0, 1)$. By evaluating explicitly the conditional expectation, one readily obtains that 
\begin{equation}\label{eq:ftGauss}
        f_t(x_1, \ldots, x_t) = \E\{ X_* \mid X_1 = x_1, \ldots, X_t = x_t  \} = \frac{\barbmu_t^\sT(\barbOmega_t)^{-1}\bx}{1+\barbmu_t^\sT(\barbOmega_t)^{-1}\barbmu_t},
\end{equation}
which leads to the following expressions for the partial derivatives:
\begin{equation}\label{eq:dftGauss}
    \partial_{x_i}f_t(x_1, \ldots, x_t) = \frac{\barbmu_t^\sT(\barbOmega_t)^{-1}\be_i}{1+\barbmu_t^\sT(\barbOmega_t)^{-1}\barbmu_t}, \quad \mbox{ for }i\in [t],
\end{equation}

\paragraph{Computation of $f_t$ for non-negative Bernoulli-Gaussian prior.}
Here, $X_*$ is equal to $0$ with probability $1-\lambda$ and it is distributed according to the modulus of a Gaussian with $0$ mean and variance $\sigma^2$ with probability $\lambda$, i.e., $X_*\sim (1-\lambda)\delta_0 + \lambda \,\mathcal{N}_+(0, \sigma^2)$. The parameter $\lambda$ is taken to be $1/6$, which is close to the actual sparsity of the image given by $8645/50625$; the parameter $\sigma^2$ is taken to be $1/\lambda$, which gives $\mathbb E\{X_*^2\}=1$, and the image is normalized to have unit second moment. Now, we can write 
\begin{equation}\label{eq:BernGauss}
     f_t(x_1, \ldots, x_t) = \E\{ X_* \mid X_1 = x_1, \ldots, X_t = x_t  \} \\
     =\frac{\E_{X_*} \{X_* \, \mathbb{P}(X_1 = x_1, \ldots, X_t = x_t \mid X_*) \}}
     {\E_{X_*}\, \{ \mathbb{P}(X_1 = x_1, \ldots, X_t = x_t \mid X_*) \}},
\end{equation}
where $\E_{X_*}$ denotes the expected value with respect to $X_*$.
Using that $(X_1, \ldots, X_t) = \barbmu_t X_* +(W_1, \ldots, W_t)$ with $(W_1, \ldots, W_t)\sim \normal(\bzero, \barbOmega_t)$, it is straightforward to compute the expectations on the RHS, which yields
\begin{equation}\label{eq:ftBernGauss}
          f_t(x_1, \ldots, x_t) = \frac{\frac{\lambda}{\sqrt{2 \pi \sigma^2}} \left[ \frac{\sqrt{\pi}b}{\sqrt{2 a^3}}\exp\left(\frac{b^2}{8a}\right) \left( 1 + \text{Erf}\left(\frac{b}{\sqrt{8a}} \right) \right) + \frac{2}{a} \right]}
      {1- \lambda + \frac{\lambda}{\sqrt{a \sigma^2}}\exp\left(\frac{b^2}{8a}\right) \left( 1 + \text{Erf}\left(\frac{b}{\sqrt{8a}} \right) \right) }, 
\end{equation}
where $a = 1/\sigma^2 + \barbmu_t^\sT(\barbOmega_t)^{-1}\barbmu_t$, $b=2 \barbmu_t^\sT(\barbOmega_t)^{-1}\bx$, $\bx = (x_1, \ldots, x_t)$, and $\text{Erf}$ is the error function. \\
To compute the derivative, we write
$\partial_{x_i}f_t(x_1, \ldots, x_t) = \partial_{b}f_t(x_1, \ldots, x_t) \partial_{x_i} b$. Since $\partial_{x_i} b  =  2 \barbmu_t^\sT(\barbOmega_t)^{-1}\be_i$, after some manipulations, one obtains
\begin{align}
    \partial_{x_i}f_t(x_1, \ldots, x_t) &= \frac{2 \frac{\lambda}{\sqrt{2\pi \sigma^2}}  \left[  \frac{b}{2a^2} + \frac{\sqrt{\pi}}{\sqrt{2a^3}}\left( 1 + \frac{b^2}{4a} \right)\exp\left(\frac{b^2}{8a}\right) \left( 1 + \text{Erf}\left(\frac{b}{\sqrt{8a}} \right) \right)  \right]}
    {1- \lambda + \frac{\lambda}{\sqrt{a \sigma^2}}\exp\left(\frac{b^2}{8a}\right) \left( 1 + \text{Erf}\left(\frac{b}{\sqrt{8a}} \right) \right)} \barbmu_t^\sT(\barbOmega_t)^{-1}\be_i \nonumber \\
    &-
    \frac{ \frac{\lambda}{\sqrt{2\pi \sigma^2}}\left(\frac{\sqrt{\pi}b}{\sqrt{2 a^3}} \exp\left(\frac{b^2}{8a}\right) \left( 1 + \text{Erf}\left(\frac{b}{\sqrt{8a}} \right) \right) + \frac{2}{a} \right)} {\left(1- \lambda + \frac{\lambda}{\sqrt{a \sigma^2}}\exp\left(\frac{b^2}{8a}\right) \left( 1 + \text{Erf}\left(\frac{b}{\sqrt{8a}} \right) \right) \right)^2 }   \nonumber\\
    & \ \cdot 2 \left( \frac{\lambda}{a \sqrt{2 \pi \sigma^2}} + \frac{\lambda b}{4 \sqrt{a^3 \sigma^2}} \exp\left(\frac{b^2}{8a}\right) \left( 1 + \text{Erf}\left(\frac{b}{\sqrt{8a}} \right) \right) \right)  \barbmu_t^\sT(\barbOmega_t)^{-1}\be_i.
    \label{eq:dftBernGauss}
\end{align}

\paragraph{Computation of $h_{t+1}$ for linear regression.} In this case, we have $Y = G+W$, where $W\sim\normal(0, \sigma^2)$. For $t=0$, we set $h_1(y)=y$ and consequently 
\begin{equation}\label{eq:part1lin}
    \partial_g h_1(y) = 1.
\end{equation}
For $t>0$, $h_{t+1}$ is defined as in \eqref{eq:Bht}. From \eqref{eq:GRi_def}, we have that $(G, R_1, \ldots, R_{t})   \sim \normal(\bzero, \barbSigma_{t+1})$. Thus, the second conditional expectation in \eqref{eq:Bht} can be expressed as
\begin{equation}\label{eq:htcomp1}
    \mathbb E\{G\mid R_1=r_1, \ldots, R_t = r_t\} = (\barbSigma_{t+1})_{[1, 2:t+1]}\left((\barbSigma_{t+1})_{[2:t+1, 2:t+1]}\right)^{-1}\br,
\end{equation}
where $\br = (r_1, \ldots, r_t)^\sT$.  Note that $(G, R_1, \ldots, R_{t}, Y)   \sim \normal(\bzero, \bar{\bS}_{t+2})$, where 
\begin{equation*}
    \bar{\bS}_{t+2} := \left[\begin{array}{cc}
    \barbSigma_{t+1} & (\barbSigma_{t+1})_{[1:t+1, 1]} \\
    (\barbSigma_{t+1})_{[1, 1:t+1]}     & \mathbb E\{Y^2\}
    \end{array}\right].
\end{equation*}
Here, we denote by $(\bA)_{[i_1:i_2, j_1:j_2]}$ the submatrix obtained by taking the rows of $\bA$ from $i_1$ to $i_2$ and the columns of $\bA$ from $j_1$ to $j_2$ (if $i_1=i_2$ or $j_1=j_2$, the second index is omitted). Thus, the first conditional expectation in \eqref{eq:Bht} can be expressed as
\begin{equation}\label{eq:htcomp2}
      \mathbb E\{G\mid R_1=r_1, \ldots, R_t = r_t, Y=y\} = (\bar{\bS}_{t+2})_{[1, 2:t+2]}\left((\bar{\bS}_{t+2})_{[2:t+2, 2:t+2]}\right)^{-1}\left[\begin{array}{c}
       \br      \\
            y
      \end{array}\right].
\end{equation}
 By combining \eqref{eq:htcomp1} and \eqref{eq:htcomp2}, we obtain 
\begin{equation}  \label{eq:htlin}      
\begin{split}
        h_{t+1}(r_1&, \ldots, r_t, y) =  E\{G\mid R_1=r_1, \ldots, R_t = r_t, Y=y\} - E\{G\mid R_1=r_1, \ldots, R_t = r_t\} \\
        &= (\bar{\bS}_{t+2})_{[1, 2:t+2]}\left((\bar{\bS}_{t+2})_{[2:t+2, 2:t+2]}\right)^{-1}\left[\begin{array}{c}
       \br      \\
            y
      \end{array}\right] - (\barbSigma_{t+1})_{[1, 2:t+1]}\left((\barbSigma_{t+1})_{[2:t+1, \,  2:t+1]}\right)^{-1}\br.
\end{split}
\end{equation}
Furthermore, the partial derivatives of $h_{t+1}$ can be expressed in the following compact form:
\begin{equation}\label{eq:dhtlin}
\begin{split}
    \partial_{r_i}&h_{t+1}(r_1, \ldots, r_t, y) 
    = (\bar{\bS}_{t+2})_{[1, 2:t+2]}\left((\bar{\bS}_{t+2})_{[2:t+2, 2:t+2]}\right)^{-1} \begin{bmatrix} \be_i \\ 0 \end{bmatrix}
    \\
    &\hspace{13em}- (\barbSigma_{t+1})_{[1, 2:t+1]}\left((\barbSigma_{t+1})_{[2:t+1, 2:t+1]}\right)^{-1}\be_i, \, \mbox{ for } i\in [t],\\
        \partial_{g}&h_{t+1}(r_1, \ldots, r_t, y)   
    = (\bar{\bS}_{t+2})_{[1, 2:t+2]}\left((\bar{\bS}_{t+2})_{[2:t+2, 2:t+2]}\right)^{-1}\be_{t+1}. 
\end{split}
\end{equation}

\paragraph{Computation of $h_{t+1}$ for noiseless 1-bit compressed sensing.}
In this case, we have that $Y={\rm sign}(G)$, which implies that 
\begin{equation}\label{eq:1bitY}
    \begin{split}
 \mathbb P(Y=1\mid G=g) &=\frac{1+{\rm sign}(g)}{2},       \qquad \qquad
 \mathbb P(Y=-1\mid G=g) =\frac{1-{\rm sign}(g)}{2} .
    \end{split}
\end{equation}
For $t=0$, we set $h_1(y)=y$ and consequently
\begin{equation}\label{eq:part1CS}
\mathbb E\{ \partial_g h_1(Y)\} =\mathbb E\{ \partial_g {\rm sign}(G)\} = \frac{\mathbb E\{G\,{\rm sign}(G)\}}{\mathbb E\{G^2\}} = \sqrt{\frac{2}{\pi\mathbb E\{G^2\}}},
\end{equation}
where the first equality follows from the definition of $h_1$ and $y$, and the second equality is obtained by recalling that $G$ is Gaussian with zero mean and by applying Stein's lemma.
For $t>0$, $h_{t+1}$ is defined as in \eqref{eq:Bht}. Since  $(G,R_1, \ldots, R_t) \sim \normal(\bzero, \barbSigma_{t+1})$, the conditional distribution of $G$ given $(R_1=r_1, \ldots, R_t=r_t)$ is $\normal(\hat{r}_t, \hat{\sigma}^2_t)$ where
\begin{equation}\label{eq:defsigmaht}
\begin{split}
 \hat{r}_t & =\mathbb E\{G\mid R_1=r_1, \ldots, R_t=r_t\}
 = (\barbSigma_{t+1})_{[1, 2:t+1]}\left((\barbSigma_{t+1})_{[2:t+1, 2:t+1]}\right)^{-1}\br,\\
 \hat{\sigma}^2_t&=\mathbb E\{G^2\}-(\barbSigma_{t+1})_{[1, 2:t+1]}\left((\barbSigma_{t+1})_{[2:t+1, \,  2:t+1]}\right)^{-1}(\barbSigma_{t+1})_{[2:t+1, 1]}.
\end{split}    
\end{equation}
We therefore have
\begin{equation}\label{eq:1bitht1}
\begin{split}
       &   \mathbb E\{G\mid R_1=r_1, \ldots, R_t = r_t, Y=y\} =      \mathbb E\{G\mid \hat{R}_t=\hat{r}_t, Y=y\} \\
         &  = \frac{\mathbb E_Z\left\{(\hat{r}_t+\hat{\sigma}_t Z)\displaystyle\frac{1+{\rm sign}(y(\hat{r}_t+\hat{\sigma}_t Z))}{2}\right\}}{\mathbb E_Z\left\{\displaystyle\frac{1+{\rm sign}(y(\hat{r}_t+\hat{\sigma}_t Z))}{2}\right\}},
\end{split}
\end{equation}
where $Z\sim\normal(0, 1)$, $\mathbb E_Z$ indicates that the expectation is taken over $Z$, and in the second line we use \eqref{eq:1bitY}. As $\mathbb E\{G\mid R_1=r_1, \ldots, R_t = r_t\}=\hat{r}_t$, \eqref{eq:1bitht1} readily implies that
\begin{equation}
    h_{t+1}(r_1, \ldots, r_t, y) = \frac{\hat{\sigma}_t\mathbb E_Z\left\{ Z\displaystyle\frac{1+{\rm sign}(y(\hat{r}_t+\hat{\sigma}_t Z))}{2}\right\}}{\mathbb E_Z\left\{\displaystyle\frac{1+{\rm sign}(y(\hat{r}_t+\hat{\sigma}_t Z))}{2}\right\}} = \frac{\hat{\sigma}_t\,\phi\left(\displaystyle\frac{\hat{r}_t}{\hat{\sigma}_t}\right)}{\displaystyle\frac{y+1}{2}-\Phi\left(-\displaystyle\frac{\hat{r}_t}{\hat{\sigma}_t}\right)},
    \label{eq:htCS}
\end{equation}
where $\phi(x)=\frac{1}{\sqrt{2\pi}}\exp(-x^2/2)$, and $\Phi(x) = \int_{-\infty}^x\phi(t)\,{\rm d}t$. The second equality in 
\eqref{eq:htCS}  is obtained by computing the expectations and using the fact that $y\in \{-1, 1\}$. 

For the partial derivatives of $h_{t+1}$, we note that $\partial_{r_i} h_{t+1} = \frac{\partial h_{t+1}}{\partial \hat{r}_t}\frac{\partial \hat{r}_t}{\partial r_i}$, and 
\[ \frac{\partial \hat{r}_t}{\partial r_i}=(\barbSigma_{t+1})_{[1, 2:t+1]}\left((\barbSigma_{t+1})_{[2:t+1, 2:t+1]}\right)^{-1}\be_i, \quad i \in [t]. \] Thus, by using \eqref{eq:htCS} to compute $ \frac{\partial h_{t+1}}{\partial \hat{r}_t}$, after some manipulations, we have that 
\begin{equation}\label{eq:dhtCS}
\begin{split}
    &   \partial_{r_i} h_{t+1}(r_1, \ldots, r_t, y) \\
     & = \, \frac{-\phi^2\left(\displaystyle\frac{\hat{r}_t}{\hat{\sigma}_t}\right)-\displaystyle\frac{\hat{r}_t}{\hat{\sigma}_t}\phi\left(\displaystyle\frac{\hat{r}_t}{\hat{\sigma}_t}\right)\left(\displaystyle\frac{y+1}{2}-\Phi\left(-\displaystyle\frac{\hat{r}_t}{\hat{\sigma}_t}\right)\right)}{\left(\displaystyle\frac{y+1}{2}-\Phi\left(-\displaystyle\frac{\hat{r}_t}{\hat{\sigma}_t}\right)\right)^2}\cdot(\barbSigma_{t+1})_{[1, 2:t+1]}\left((\barbSigma_{t+1})_{[2:t+1, 2:t+1]}\right)^{-1}\be_i.
\end{split}
\end{equation}
Finally, for the partial derivative with respect to $g$, we have
\begin{align}
    \mathbb E\left\{\partial_g h_{t+1}(R_1, \ldots, R_t, Y)\right\} &= \mathbb E\left\{h_{t+1}(R_1, \ldots, R_t, Y)\, \frac{\mathbb E\left\{G\mid R_1, \ldots, R_t, Y\right\}-\mathbb E\left\{G\mid R_1, \ldots, R_t\right\}}{{\rm Var}(G\mid R_1, \ldots R_t)}\right\} \nonumber \\
    &= \frac{1}{\hat{\sigma}_t^2}\mathbb E\left\{h_{t+1}(R_1, \ldots, R_t, Y)^2\right\},
    \label{eq:dhtCS2}
\end{align}
where the first equality follows from Stein's lemma (see e.g. (A.8) of \cite{mondelli2021approximate}), and in the second equality we use the definition \eqref{eq:Bht} of $h_{t+1}$ and that ${\rm Var}(G\mid R_1, \ldots R_t) = \hat{\sigma}_t^2$ (see \eqref{eq:defsigmaht}). 

\paragraph{Implementation details.} \emph{RI-GAMP}: We use consistent empirical estimates for the state evolution parameters required for the posterior mean denoisers and their partial derivatives. These estimates are computed as described on p.\pageref{para:emp_SE}. To estimate the first row  and column of $\barbGamma_{t+1}$, we use the definition \eqref{eq:Bft} and the tower property of conditional expectation: 
\begin{equation}\label{eq:spest}
    (\barbGamma_{t+1})_{1, i+1} = \mathbb E\left\{X_*f_i(X_1, \ldots, X_i)\right\} = \mathbb E\left\{ f_i(X_1, \ldots, X_i)^2\right\}, \quad i \in [t]. 
\end{equation}
Therefore, one can consistently estimate $(\barbGamma_{t+1})_{1, i+1}$ via $ \| f_i(\bx_1, \ldots, \bx_i) \|^2/d$. The partial derivatives for the matrix $\bPsi_{t+1}$ (defined in  \eqref{eq:Psi_Phi_def}) are computed using 
\eqref{eq:dftRad} (for Rademacher prior) and \eqref{eq:dftGauss} (for Gaussian prior).
The partial derivatives for the matrix $\bPhi_{t+1}$ (again, defined in  \eqref{eq:Psi_Phi_def}) are computed using 
\eqref{eq:dhtlin} (for linear regression) and \eqref{eq:dhtCS}-\eqref{eq:dhtCS2} (for 1-bit compressed sensing). For the quantity $\langle\partial_g h_1(\by)\rangle$, we use the deterministic limit $\mathbb E\{\partial_g h_1(Y)\}$ which is given in the two settings by \eqref{eq:part1lin} and \eqref{eq:part1CS}, respectively.

\emph{VAMP}: Our implementation of VAMP is based on Algorithm 2 in \cite{schniter2016vector} and the corresponding state evolution is derived from
\cite{pandit2020inference}\footnote{See also the code available at \texttt{https://sourceforge.net/projects/gampmatlab/}.}. To ensure numerical stability, we clipped the $\alpha_i$ and $\beta_i$ in Algorithm 2 to lie in $[\texttt{tol}, 1-\texttt{tol}]$, where $\texttt{tol} =10^{-11}$.


\section{Additional Numerical Results}\label{app:addnum}

\begin{figure}[t]
\centering
\begin{minipage}{.62\textwidth}
  \centering
    \subfloat[$\delta=0.7$\label{fig:delta0dot7}]{\includegraphics[width=.5\columnwidth]{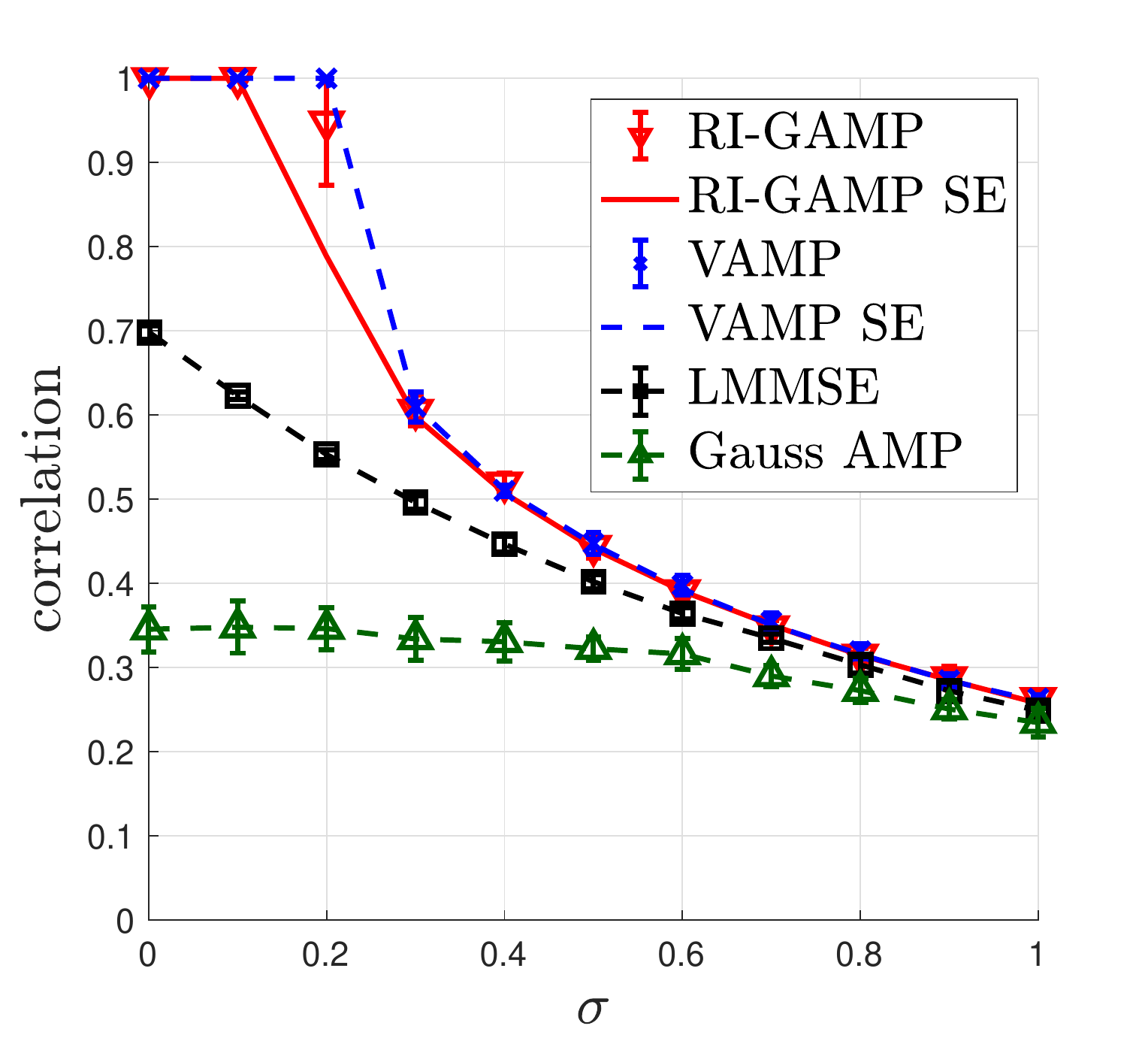}}
    \subfloat[$\sigma=0.4$\label{fig:sigma0dot4}]{\includegraphics[width=.5\columnwidth]{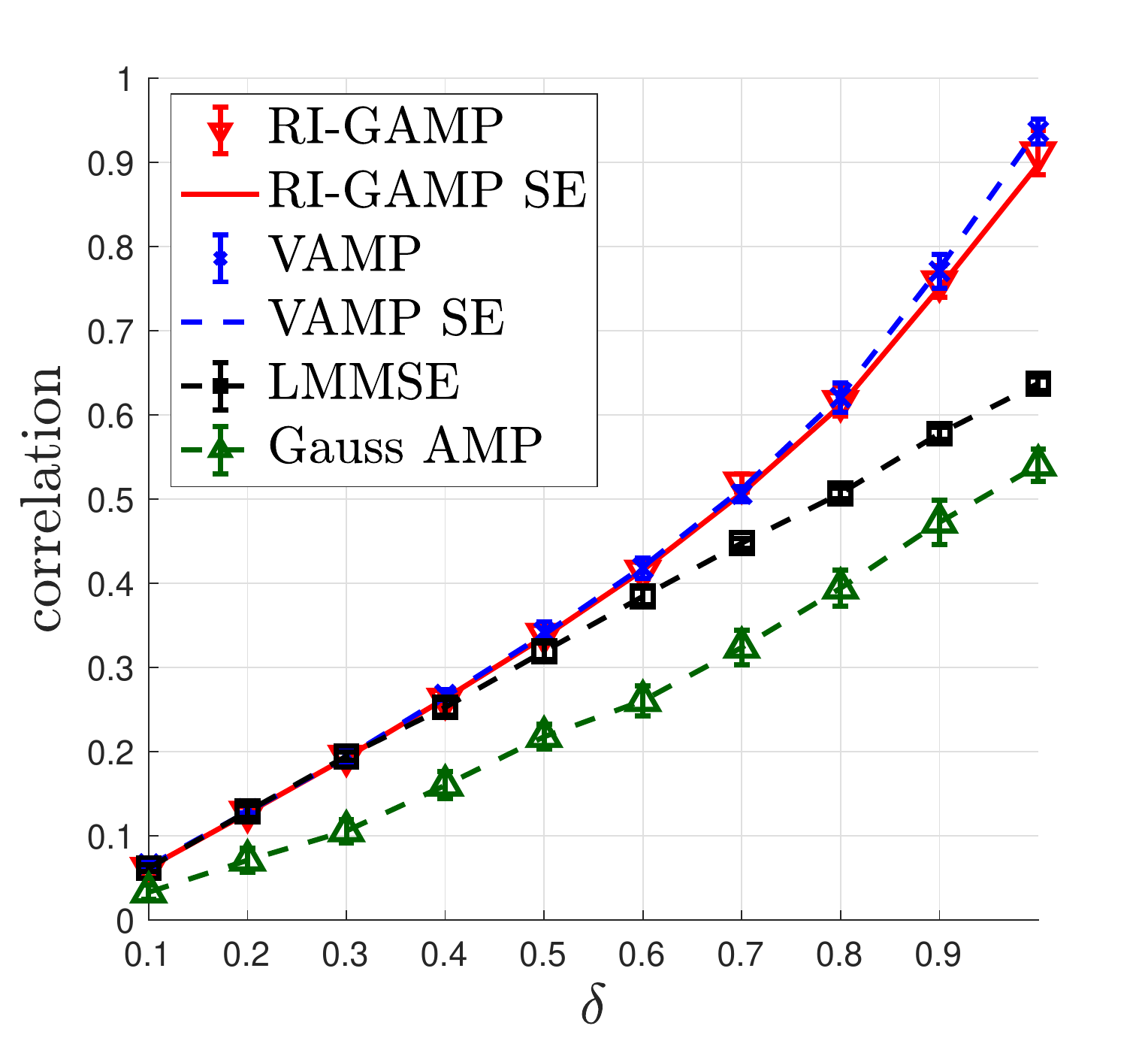}}
\captionof{figure}{\small Additional numerical results for linear regression with a Rademacher prior: normalized squared correlation vs. noise level $\sigma$ (on the left) and vs. aspect ratio $\delta$ (on the right).}
\label{fig:linfixsigma}
\end{minipage}
\hspace{2em}
\begin{minipage}{.3\textwidth}
  \centering
  \includegraphics[width=\columnwidth]{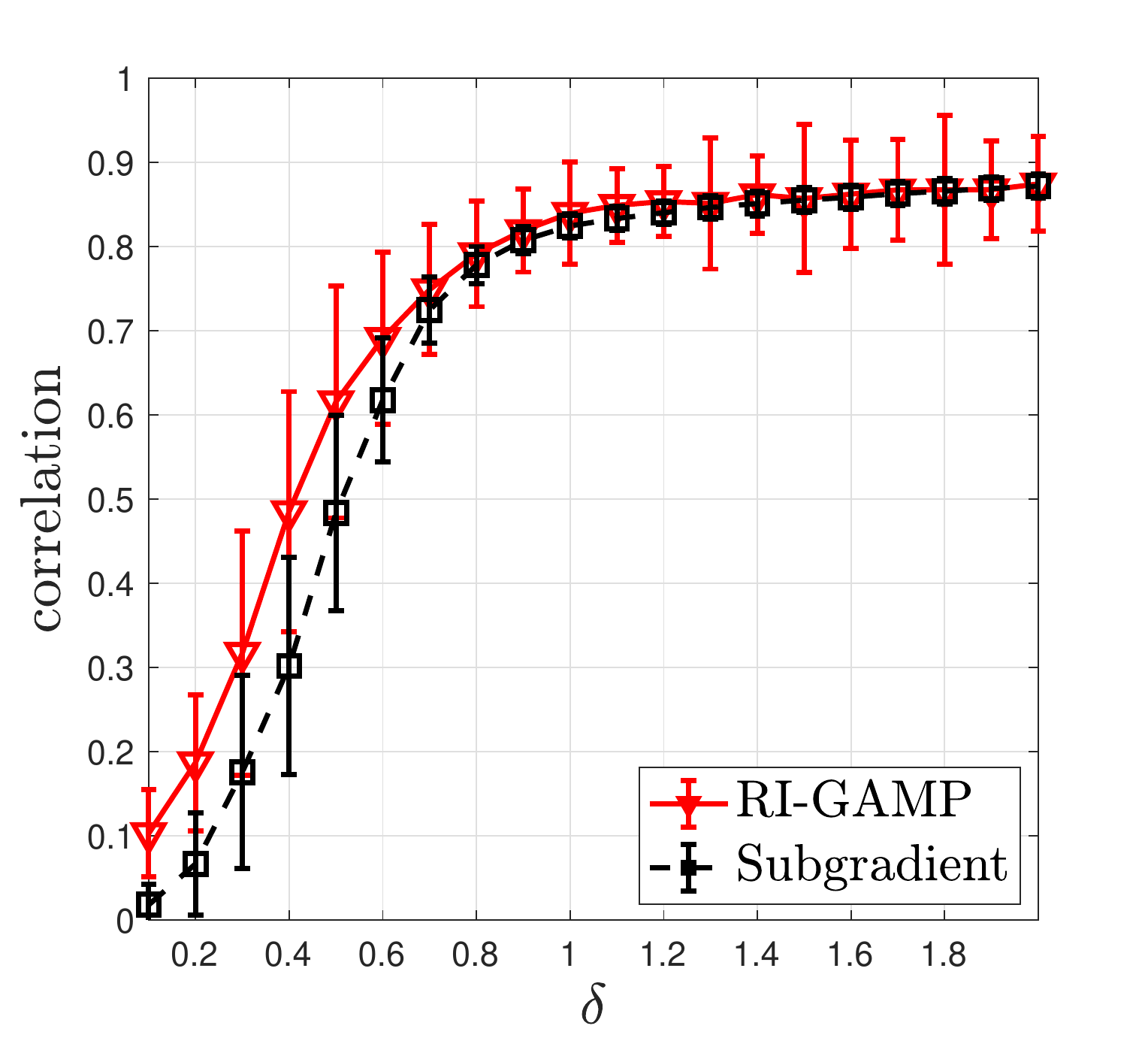}
  \captionof{figure}{\small Comparison between the normalized squared correlation of RI-GAMP and of a subgradient method for the recovery of an image from 1-bit measurements of its wavelet transform.}
  \label{fig:test2}
\end{minipage}
\end{figure}

In Figure \ref{fig:linfixsigma}, we provide additional numerical results for the model of linear regression with a Rademacher signal prior: on the left, we plot the normalized squared correlation as a function of $\sigma$, for $\delta=0.7$; and on the right, we plot the same metric as function of $\delta$, for $\sigma=0.4$. The results showcase a similar qualitative behavior as discussed in Section \ref{sec:simu}: the performance of RI-GAMP is close to that of VAMP, except when approaching the phase transition for exact recovery ($\sigma\approx 0.2$ in Figure \ref{fig:delta0dot7}); and RI-GAMP significantly improves upon algorithms that do not take into account the signal prior (LMMSE) or the spectrum of the noise (Gauss AMP). 

In Figure \ref{fig:test2}, we consider noiseless 1-bit compressed sensing.  The input $\bx^*$ is the  Haar wavelet transform of the RGB image in Figure \ref{fig:orimven}. We process the three  channels (R, G, B) separately and the input dimension is $d = 820 \times 1280 =1049600$.  The design matrix $\bA$ is given by $\bA=\bQ_n\boldsymbol{\Pi}_n\bLambda \boldsymbol{\Pi}_d \bQ_d$, where $\bQ_n$, $\bQ_d$ are orthonormal Discrete Cosine Transform (DCT) matrices  in $n$, $d$ dimensions, $\boldsymbol{\Pi}_n, \boldsymbol{\Pi}_d$ are random permutation matrices, and $\bLambda$ has i.i.d. $\sqrt{6}\cdot {\rm Beta}(1, 2)$ diagonal entries.
We compare the performance of RI-GAMP against the subgradient method from \cite{jacques2013robust}.
For the wavelet transform we use the implementation given in \cite{Lee2019}.
For RI-GAMP we use the same non-negative Bernoulli-Gaussian prior employed for the satellite image (cf. Figure \ref{fig:satimgfig} in Section \ref{sec:simu}).
Since there is no clear way to define the true sparsity of the signal in this setting, we fix $\delta=0.5$ and optimize over the sparsity rate for both algorithms, which yields a sparsity of $1/10$ (i.e., assuming that $1/10$ entries are non-zero) for RI-GAMP and of $1/20$ for the subgradient method.
We also note that the performance of the algorithms around these values is quite stable, so we don't expect the precise choice of the sparsity rate to matter much for the chosen range of $\delta$.
In Figure \ref{fig:test2}, we report the normalized squared correlation averaged over the 3 channels and error bars at 1 standard deviation  for 100 random trials. We remark that RI-GAMP improves upon the subgradient method for $\delta$ up to $1$ and, for larger $\delta$, its performance does not increase noticeably due to the already discussed numerical instabilities. The reconstructions provided by RI-GAMP and by the subgradient method for $\delta =0.5$ are also compared in Figure \ref{fig:venimgfig}. 

\begin{figure}[t]
    \centering
        \subfloat[Original image\label{fig:orimven}]{\includegraphics[width=.33\columnwidth]{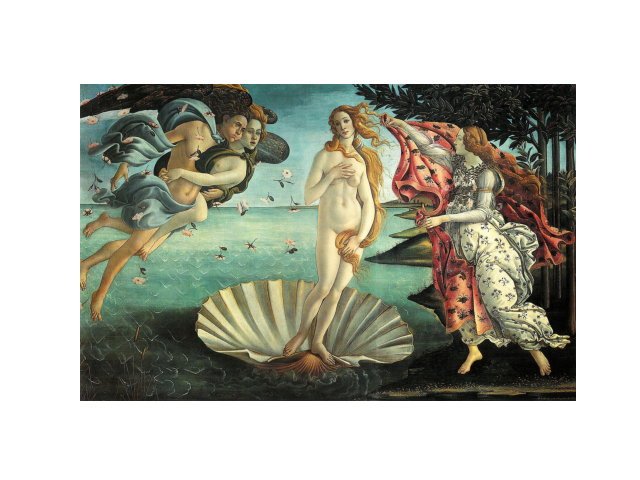}}
        \subfloat[RI-GAMP, $\delta=0.5$\label{fig:ampven}]{\includegraphics[width=.33\columnwidth]{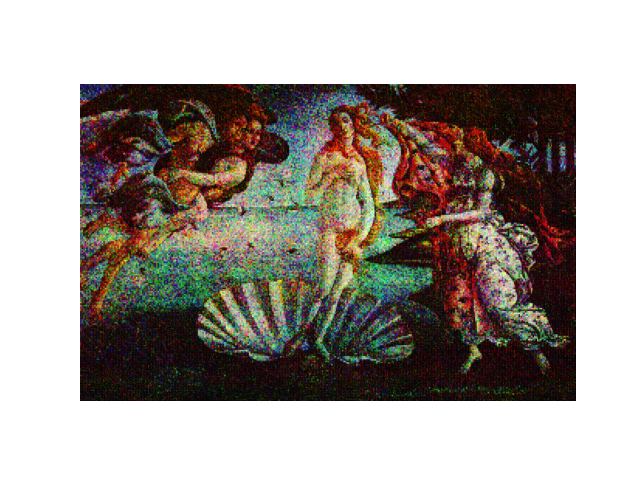}}
        \subfloat[Subgradient,  $\delta=0.5$\label{fig:gdven}]{\includegraphics[width=.33\columnwidth]{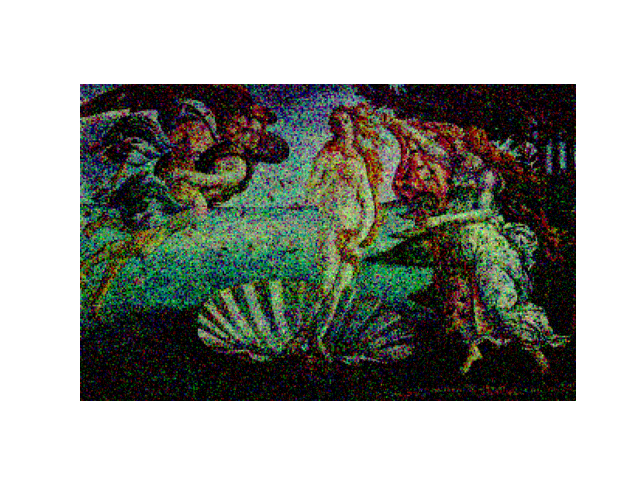}}
\caption{Reconstruction provided by RI-GAMP and by a subgradient method from 1-bit measurements of the wavelet transform of an RGB image.}
\label{fig:venimgfig}
\end{figure}

\section{Proof of Theorem \ref{thm:main}} \label{app:full_proof}

\subsection{Debiasing Coefficients for the Auxiliary AMP} \label{sec:debiasing_aux}

The debiasing coefficients $\{\sa_{ti}\}_{i=1}^t$ and $\{\ssb_{ti}\}_{i=1}^{t-1}$ for the auxiliary AMP in \eqref{eq:auxAMPz}-\eqref{eq:auxAMPm} are defined in terms of two $t \times t$ lower triangular matrices, $\hbPsi_t$ and $\hbPhi_t$, given by
\beq
\begin{split}
    \hbPsi_t & = \begin{pmatrix}
    0 & 0 & \ldots& \ldots  & 0 \\
  0  &  \< \partial_2 \bv^2 \> & 0 &   \ldots  & 0 \\
    0 & \< \partial_2 \bv^3 \> &  \< \partial_3 \bv^3 \> & \ldots & 0   \\
    \vdots & \vdots & \vdots  & \ddots & \vdots  \\
    0 & \< \partial_2 \bv^t \> & \< \partial_3 \bv^t \> & \ldots &  \< \partial_t \bv^t \>
    \end{pmatrix}, \qquad
    \hbPhi_t  = \begin{pmatrix}
    0 & 0 & \ldots & 0 & 0 \\
   \< \partial_1 \bu^2 \> &  0 & \ldots & 0  & 0 \\
    \< \partial_1 \bu^3 \> &   \< \partial_2 \bu^3 \> & \ldots & 0  & 0 \\
    \vdots & \vdots & \ddots & \vdots & \vdots  \\
     \< \partial_1 \bu^t \> & \< \partial_2 \bu^t \> & \ldots & \< \partial_{t-1} \bu^t \> & 0
    \end{pmatrix},
\label{eq:bPQ_def}
\end{split}
\eeq
where
\begin{align}
&  \partial_k \bv^t = \partial_{z_k} \tilde{f}_t(\bz^1, \ldots, \bz^t, \bx^*)  = \partial_{k-1} f_{t-1}(\bz^2 + \bar{\mu}_1 \bx, \,  \ldots, \, \bz^{t} + \bar{\mu}_{t-1} \bx^*), \qquad  k \ge 2, \nonumber \\
& \partial_1 \bu^t = \partial_g h_{t-1}(\bm^2, \ldots, \bm^{t-1}, q(\bg, \bveps) )
\Big\vert_{\bg = \bm^1}, \nonumber \\
& \partial_k \bu^t = \partial_{m_k} \tilde{h}_t(\bm^1, \ldots, \bm^{t-1}, \bveps)   = \partial_{k-1} h_{t-1}( \bm^2, \ldots, \bm^{t-1}, \, q(\bm^1, \bveps) ),  \qquad  k \ge 2.
\label{eq:part_uv}
\end{align}
Here, $\partial_{k-1}f_{t-1}$ and $\partial_{k-1}h_{t-1}$ denote the partial derivatives with respect to the $(k-1)$-th input variable. Next,  define matrices $\bM_t^{\sa}, \bM_t^{\ssb}\in \mathbb R^{t\times t}$ as:
\begin{align}
    \bM_t^{\sa} = \sum_{j=0}^t \kappa_{2(j+1)} \hbPsi_t (\hbPhi_t \hbPsi_t)^j,\nonumber  \\
    \bM_t^{\ssb}= \delta \sum_{j=0}^{t-1} \kappa_{2(j+1)} \hbPhi_t (\hbPsi_t \hbPhi_t)^j.
    \label{eq:Mta_Mtb}
\end{align}
Then, the coefficients $\{\sa_{ti}\}_{i=1}^t$ and $\{\ssb_{ti}\}_{i=1}^{t-1}$ are  obtained from the last row of $\bM_t^{\sa}$ and $\bM_t^{\ssb}$ as: 
\beq
\begin{split}
    (\sa_{t1}, \ldots, \sa_{tt}) &= (\,  (\bM_t^{\sa})_{t,1}, \ldots, (\bM_t^{\sa})_{t,t} \, ),\\
    (\ssb_{t1}, \ldots, \ssb_{t,t-1}) &= (\,  (\bM_t^{\ssb})_{t,1}, \ldots, (\bM_t^{\ssb})_{t, t-1} \, ).
\end{split}
\label{eq:atbt_coeffs}
\eeq

\subsection{State Evolution for Auxiliary AMP} \label{sec:SE_aux}

Using Theorem 2.6 of \cite{zhong2021approximate}, we first establish a state evolution result for the auxiliary AMP \eqref{eq:auxAMPz}-\eqref{eq:auxAMPm}. We will show in Proposition \ref{prop:auxSE} that the joint empirical distribution of $(\bm^1, \ldots, \bm^t)$ converges to a $t$-dimensional Gaussian $\normal(\bzero, 
\tbSigma_t)$, and the joint empirical distribution of $(\bz^1, \ldots, \bz^t)$ converges to a a $t$-dimensional Gaussian $\normal(\bzero, \tbOmega_t)$.

The covariance matrices are  defined recursively for $t \ge 1$, starting with 
$\tbOmega_1 =0$ and $\tbSigma_1 = \bar{\kappa}_2 \E\{ X_*^2\}$.  Given $(\tbOmega_t, \tbSigma_t)$,  let
\begin{align}
 & (Z_1=0,Z_2, \ldots, Z_t) \sim \normal(\bzero, \tbOmega_t),  \quad 
 V_t=f_{t-1}(Z_2 + \barmu_1 X_*, \, Z_3 + \barmu_2 X_*, \, \ldots, \, Z_{t} + \barmu_{t-1}, X_* ),  \label{eq:Zi_def}\\
 & (M_1, \ldots, M_t) \sim \normal(\bzero, \tbSigma_t), \quad U_{t+1}
=h_t(M_2, \ldots, M_t, \, q(M_1, \veps)). \label{eq:Mi_def}
\end{align}
In \eqref{eq:Zi_def},  $(Z_2, \ldots, Z_t)$ and $X_*$ are independent, and  we define $V_1=X_*$ and $U_1=0$.

Let $\tbDelta_{t+1}, \tbGamma_{t+1} \in \reals^{(t+1) \times (t+1)}$ be symmetric matrices with entries given by 
\begin{align}
\label{eq:infDeltaGamma_def}
    &(\tbDelta_{t+1})_{ij} = \E\{U_i U_j \}, 
   \quad 
  (\tbGamma_{t+1})_{ij}= \E\{ V_i V_j \} , \quad  1\le i, j \le (t+1).
\end{align}
Furthermore, we define $\tbPsi_{t+1}, \tbPhi_{t+1} \in \reals^{(t+1) \times (t+1)}$ as the deterministic versions of the matrices  $\hbPsi_{t+1}, \hbPhi_{t+1}$. Specifically,
\begin{equation}
\begin{split}
& \tbPsi_{t+1} =
\begin{pmatrix}
    0 & 0 & 0& \ldots & 0 \\
  0  &  \E\{ \partial_1 f_1 \} & 0 & \ldots  & 0 \\
      0  &  \E\{ \partial_1 f_2 \} &  \E\{ \partial_2 f_2 \}  & \ldots  & 0 \\
    \vdots & \vdots & \vdots   & \ddots & \vdots  \\
    0 & \E\{  \partial_1 f_{t} \} & \E\{ \partial_2 f_{t} \} & \ldots &  \E\{ \partial_{t} f_{t} \}
    \end{pmatrix}, \quad  \tbPhi_{t+1} = \begin{pmatrix}
    0 & 0 & \ldots & 0 & 0 \\
    \E\{ \partial_g h_1 \} &  0 & \ldots & 0  & 0 \\
     \E\{ \partial_g h_2 \} &   \E\{ \partial_1 h_2 \} & \ldots & 0  & 0 \\
    \vdots & \vdots & \ddots & \vdots & \vdots  \\
     \E\{ \partial_g h_{t} \} & \E\{ \partial_1 h_{t} \} & \ldots &  \E\{ \partial_{t-1} h_{t} \} & 0
    \end{pmatrix},
    \label{eq:tbPQ_def}
    \end{split}
\end{equation}
where we have used the shorthand
\begin{equation}
\begin{split}
        & \E\{ \partial_k f_\ell  \} \equiv \E\{  \partial_k f_{\ell}(Z_2 + \barmu_1 X_*, \, Z_3 + \barmu_2 X_*, \, \ldots, \, Z_{\ell+1} + \barmu_{\ell} X_*) \}, \\
        & \E\{ \partial_g h_\ell \} \equiv \E\{  \partial_g h_{\ell}(M_2, \ldots, M_\ell, q(g , \veps)) \vert_{g=M_1} \}, \\
        & \E\{ \partial_{k} h_\ell \} \equiv \E\{  \partial_{k} h_{\ell}(M_2, \ldots, M_\ell, q(M_1, \veps)) \}.
\end{split}
\label{eq:shorthand_def}
\end{equation}
From these matrices, we compute the covariances $\tbSigma_{t+1}, \tbOmega_{t+1} \in \reals^{(t+1) \times (t+1)}$ as:
 \beq
 \tbSigma_{t+1} =  \sum_{j=0}^{2t+1} \barkap_{2(j+1)} \, \tbXi_{t+1}^{(j)}, \qquad  \tbOmega_{t+1} = \delta \sum_{j=0}^{2t} \barkap_{2(j+1)} \, \tbTheta_{t+1}^{(j)},
 \label{eq:tSigmaOmega_def}
 \eeq
where $\tbXi_{t+1}^{(0)} = \tbGamma_{t+1}$, $\tbTheta_{t+1}^{0} = \tbDelta_{t+1}$, and for $j \ge 1$:
\beq
\begin{split}
\tbXi_{t+1}^{(j)} & = \sum_{i=0}^{j} (\tbPsi_{t+1} \tbPhi_{t+1})^i \, \tbGamma_{t+1} 
\Big((\tbPsi_{t+1} \tbPhi_{t+1})^{\sT}\Big)^{j-i}  +  \sum_{i=0}^{j-1} (\tbPsi_{t+1} \tbPhi_{t+1})^i \, \tbPsi_{t+1}\tbDelta_{t+1} \tbPsi_{t+1}^{\sT}
\Big((\tbPsi_{t+1} \tbPhi_{t+1})^{\sT}\Big)^{j-i-1}, \\
\tbTheta^{(j)}_{t+1} & = \sum_{i=0}^j (\tbPhi_{t+1} \tbPsi_{t+1} )^i \, \tbDelta_{t+1}  \Big((\tbPhi_{t+1}\tbPsi_{t+1})^{\sT}\Big)^{j-i}   +   \sum_{i=0}^{j-1} (\tbPhi_{t+1} \tbPsi_{t+1} )^i \, \tbPhi_{t+1} \tbGamma_{t+1}  \tbPhi_{t+1}^{\sT}
 \Big((\tbPhi_{t+1}\tbPsi_{t+1})^{\sT}\Big)^{j-i-1}.
\label{eq:tXi_Theta_t1j_def}
\end{split}
\eeq
We note that computing the last column and row of $\tbGamma_{t+1}, \tbPsi_{t+1}$ requires knowledge of $\tbOmega_{t+1}$. However, these entries are zeroed out in the computation of
$\tbXi_{t+1}^{(j)}, \tbTheta^{(j)}_{t+1}$ in \eqref{eq:tXi_Theta_t1j_def}, and hence only $\tbGamma_{t}, \tbPsi_{t}$ are required to compute $\tbSigma_{t+1}, \tbOmega_{t+1}$ via \eqref{eq:tSigmaOmega_def}.

\begin{proposition}[State evolution for auxiliary AMP]
Consider the auxiliary AMP in \eqref{eq:auxAMPz}-\eqref{eq:auxAMPm} and the state evolution random variables defined in \eqref{eq:Zi_def}-\eqref{eq:Mi_def}. Let $\tilde{\psi}:\reals^{2t+1} \to \reals$ and $\tilde{\phi}:\reals^{2t+2} \to \reals$ be any pseudo-Lipschitz functions of order 2. Then for each $t\ge 1$, we almost surely have
\begin{align}
& \lim_{n \to \infty}  \frac{1}{d} \sum_{i=1}^d  \tilde{\psi}(z^2_i, \ldots, z^{t+1}_i, v^2_i, \ldots, v^{t+1}_i, x^*_i) 
= \E\{ \tilde{\psi}(Z_2, \ldots Z_{t+1}, V_2, \ldots, V_{t+1}, X_*) \}, \label{eq:PL2z_conv}\\
& \lim_{n \to \infty}  \frac{1}{n} \sum_{i=1}^n \tilde{\phi}(m^1_i, \ldots, m^t_i, u^1_i, \ldots, u^{t+1}_i, \, \veps_i ) = \E\{ \tilde{\phi}(M_1, \ldots, M_t, U_1, \ldots, U_{t+1}, \, \veps) \}.
\label{eq:PL2m_conv}
\end{align}
Equivalently, as $n\to\infty$, almost surely:
\begin{align*}
    & (\bz^2, \ldots, \bz^{t+1}, \,  \bv^2, \ldots, \bv^{t+1}, \, \bx^*) \, \stackrel{\mathclap{W_2}}{\longrightarrow}  \,  (Z_2, \ldots, Z_t, \, V_2, \ldots, V_{t+1}, \, X_*),  \\
    &    (\bm^1, \ldots, \bm^t, \,  \bu^1, \ldots, \bu^{t+1}, \, \bveps) \, \stackrel{\mathclap{W_2}}{\longrightarrow}  \,  (M_1, \ldots, M_t, \, U_1, \ldots, U_{t+1}, \, \veps).
\end{align*}
\label{prop:auxSE}
\end{proposition}
The proposition follows directly from Theorem 2.6 in \cite{zhong2021approximate} as the auxiliary AMP in \eqref{eq:auxAMPz}-\eqref{eq:auxAMPm}  is of the standard form for which that state evolution result applies. That result is proved for  $\delta = \frac{n}{d} \le 1$ under two sets of assumptions (cf. Assumptions 2.4 and 2.5 in \cite{zhong2021approximate}).  The first set of assumptions concerns the design matrix, and these coincide with the ones we describe in Section \ref{sec:prel}. The second set concerns the empirical distribution of the signal and noise vectors, and the functions $\tilde{f}_t, \tilde{h}_{t}$ used in the auxiliary AMP. This set of assumptions is also satisfied since $\bx^* \in \reals^d$ and $\bveps \in \reals^n$ are both independent of the design matrix $\bA$ and satisfy $\bx^* \stackrel{W_2}{\to} X_*$ and $\bveps \stackrel{W_2}{\to} \veps$. Furthermore, our assumption \textbf{(A1)} (see p.\pageref{assump:A1}) ensures that the required Lipschitz and continuity conditions on $\tf_t, \th_t$ and their partial derivatives are satisfied.  Therefore,  for $\delta \le 1$, the iteration in \eqref{eq:auxAMPz}-\eqref{eq:auxAMPm} satisfies all the assumptions under which Theorem 2.6 in \cite{zhong2021approximate} holds.
Finally, for the case $\delta >1$, we can rewrite  the auxiliary AMP in terms of $\bA' \equiv \bA^{\sT}$ and then apply Theorem 2.6 in  \cite{zhong2021approximate}.

We conclude this section by showing that the state evolution of the auxiliary AMP described above is equivalent to the state evolution of the proposed AMP algorithm described in Section \ref{sec:RIGAMP}. 

\begin{lemma}[Equivalence of state evolution between true and  auxiliary AMP]
For $t \ge 1$, we have that
\begin{equation}
\begin{split}
& (M_1, \ldots, M_{t}) \stackrel{\rm d}{=} (G, R_1, \ldots, R_{t-1}) \ \sim \ \normal(\bzero, \barbSigma_t), \\
&      (Z_2, \ldots, Z_{t+1}) \, \stackrel{\rm d}{=} \, (X_1 - \barmu_1 X_*, \ldots, X_t -\barmu_t X_*) \ \sim \ \normal(\bzero, \barbOmega_t), 
\end{split}
\end{equation}
where the random variables on the left are defined in \eqref{eq:Zi_def}-\eqref{eq:Mi_def}, and the random variables on the right are defined in \eqref{eq:GRi_def}-\eqref{eq:Xi_def}.
\label{lem:SE_equiv}
\end{lemma}
\begin{proof}
We will prove by induction that $\tbSigma_{t} = \barbSigma_{t}$ and $\tbOmega_{t+1} = \bOmega'_{t+1}$, where the matrices on the left are defined via \eqref{eq:tSigmaOmega_def} and the matrices on the right are  defined via  \eqref{eq:barSigma_def} and \eqref{eq:Omega_pr_def}. The result of the lemma then follows since $(M_1, \ldots, M_{t}) \sim \normal(\bzero, \, \tbSigma_{t})$,    $(Z_1=0, Z_2, \ldots, Z_{t+1}) \sim \normal(\bzero, \tbOmega_{t+1})$ and  $\barbOmega_t$ is the lower right $t \times t$ submatrix of $\bOmega'_{t+1} \in \reals^{ (t+1) \times (t+1)}$.

For $t=1$, by the initialization in \eqref{eq:Mi_def}, $M_1 \sim \normal(0, \tbSigma_1)$ where $\tbSigma_1 =\kappa_2 \E\{ X_*^2 \} =\barbSigma_1$, where the last equality is from  \eqref{eq:SE_init}.  From  \eqref{eq:tSigmaOmega_def}, the matrix $\tbOmega_2$ can be computed as
\beq
\tbOmega_2 = \begin{pmatrix}
    0 & 0 \\
    0  & \delta \barkap_2 \E\{h_1(q(M_1, \veps))^2 \} \, + \, \delta \barkap_4 \E\{ V_1^2\} (\E\{ \partial_g h_1(q(M_1, \veps))\})^2
\end{pmatrix}.
\eeq
Since $V_1 \stackrel{\text{d}}{=} X_*$ and $M_1 \stackrel{\text{d}}{=} G$, the matrix above equals $\bOmega'_2$ (defined via \eqref{eq:Omega_pr_def}).  

Assume towards induction  that $\tbSigma_k = \barbSigma_k$ and $\tbOmega_{k+1} = 
\bOmega'_{k+1}$ for some $k \ge 1$.  Recalling that $(M_1, \ldots, M_k) \sim \normal(0, \tbSigma_k)$ and $(Z_1=0, \ldots, Z_{k+1}) \sim \tbOmega_{k+1}$, using the induction hypothesis in the definitions of $\tbDelta_{k+1}, \tbGamma_{k+1}, \tbPhi_{k+1}, \tbPsi_{k+1}$ in \eqref{eq:infDeltaGamma_def}-\eqref{eq:shorthand_def}, we obtain that
\beq
\tbDelta_{k+1} = \barbDelta_{k+1}, \quad \tbGamma_{k+1} = \barbGamma_{k+1}, \quad 
\tbPhi_{k+1} = \barbPhi_{k+1}, \quad \tbPsi_{k+1} = \barbPsi_{k+1},
\eeq
where $\barbDelta_{k+1}, \barbGamma_{k+1}, \barbPhi_{k+1}, \barbPsi_{k+1}$ are defined via  $\barbSigma_k$ and $\barbOmega_k$ in \eqref{eq:BarDelta_def}-\eqref{eq:emp_avg_shorthand}. It follows from the definitions in \eqref{eq:barSigma_def} and \eqref{eq:tSigmaOmega_def} that
$\tbSigma_{k+1} = \barbSigma_{k+1}$. This then implies that $ \tbPhi_{k+2} = \barbPhi_{k+2}$
and $\tbDelta_{k+2} = \barbDelta_{k+2}$. Using these in the definitions in \eqref{eq:tXi_Theta_t1j_def} and \eqref{eq:Theta_t1j_def}, we obtain that $\tbTheta^{(j)}_{k+2} = \bTheta^{(j)}_{k+2}$, and consequently, $\tbOmega_{k+2} = \bOmega'_{k+2}$. This completes the proof of the induction step, and gives the desired result. 
\end{proof}

\subsection{Proof of Theorem \ref{thm:main}}

At this point, Theorem \ref{thm:main} follows from the following intermediate result, whose proof is deferred to Section \ref{subsec:PLdiff_proof}.

\begin{lemma}
For any order 2 pseudo-Lipschitz functions  $\psi: \reals^{2t+1} \to \reals$ and $\phi: \reals^{2t+2} \to \reals$, the following limits hold almost surely for $t \ge 1$: 
\begin{align}
&  \lim_{n \to \infty}  \left\vert 
 \frac{1}{d} \sum_{i=1}^d  \psi(x^1_i, \ldots, x^{t}_i, \hx^1_i, \ldots, \hx^t_i, x^*_i)
\, - \,  \frac{1}{d} \sum_{i=1}^d  \psi(z^2_i + \bar{\mu}_1 x^*_i, \ldots, z^{t+1}_i + \bar{\mu}_t x^*_i, v^2_i, \ldots, v^{t+1}_i, x^*_i)  \right \vert  = 0,  \label{eq:PLdiff_x}\\
& \lim_{n \to \infty}  \left\vert    \frac{1}{n} \sum_{i=1}^n \phi( r^1_i, \ldots, r^t_i, s^1_i, \ldots, s^{t+1}_i, \, y_i ) \, - \, \frac{1}{n} \sum_{i=1}^n \phi(m^2_i, \ldots, m^{t+1}_i, u^2_i, \ldots, u^{t+2}_i, \, q(m^1_i, \veps_i) )  \right \vert = 0. \label{eq:PLdiff_r}
\end{align}
\label{lem:PLdiff}
\end{lemma}

\noindent \emph{Proof of Theorem \ref{thm:main}.} Applying  \eqref{eq:PL2z_conv} to the pseudo-Lipschitz function $$\tilde{\psi}(z_2, \ldots, z_{t+1}, v_2, \ldots, v_{t+1}, x_*) = \psi(z_2 + \bar{\mu}_1 x_*, \ldots, z_{t+1} + \bar{\mu}_t x_*, v_2, \ldots, v_{t+1}, x_*),$$
we obtain that almost surely
\beq
\begin{split}
& \lim_{n \to \infty}\,   \frac{1}{d} \sum_{i=1}^d  \psi(z^2_i + \bar{\mu}_1 x^*_i, \ldots, z^{t+1}_i + \bar{\mu}_t x^*_i, v^1_i, \ldots, v^t_i, x^*) \\
&  = \E\{ \psi(Z_2 + \bar{\mu}_1 X_* \, , \ldots, Z_{t+1} + \bar{\mu}_t X_*, \,  V_2, \ldots, V_{t+1}, X_*) \} \\
&  = \E\{ \psi( X_1, \ldots, X_t, \, \hat{X}_1, \ldots, \hat{X}_t, \, X_*) \},
\end{split}
\label{eq:psiz_conv}
\eeq
where the last equality follows from Lemma \ref{lem:SE_equiv}, by recalling that $V_{\ell+1} = f_\ell(Z_2 + \bar{\mu}_1 X_* \, , \ldots, Z_{\ell+1} + \bar{\mu}_\ell X_*)$ and $\hat{X}_\ell= f_\ell(X_1, \ldots, X_\ell)$, for $\ell \ge 1$.
Combining \eqref{eq:psiz_conv} with \eqref{eq:PLdiff_x} yields \eqref{eq:PL2_main_resultx} of Theorem \ref{thm:main}. The result \eqref{eq:PL2_main_resultr} is obtained similarly, using \eqref{eq:PL2m_conv}, Lemma \ref{lem:SE_equiv}, and \eqref{eq:PLdiff_r}.  
\qed

\subsection{Proof of Lemma \ref{lem:PLdiff}} \label{subsec:PLdiff_proof}

Throughout, we use $C$ to denote a generic positive constant. All the limits in the proof hold almost surely, so  we don't explicitly state 
this each time.

Since $\psi$ is pseudo-Lipschitz, we have
\begin{align}
 &    \left \vert \frac{1}{d} \sum_{i=1}^d \psi(x^1_i, \ldots, x^{t}_i, \hx^1_i, \ldots, \hx^t_i, x^*_i)
\, - \,  \frac{1}{d} \sum_{i=1}^d  \psi(z^2_i + \bar{\mu}_1 x^*_i, \ldots, z^{t+1}_i + \bar{\mu}_t x^*_i, v^2_i, \ldots, v^{t+1}_i, x^*_i) \right\vert \nonumber \\
& \le \frac{C}{d} \sum_{i=1}^{d} \left( 1+ |x^*_i| + \sum_{\ell=1}^t \Big( |x^\ell_i| + |\hx^\ell_i| + |z^{\ell+1}_i + \bar{\mu}_\ell x^*_i| + |v_i^{\ell+1}|  \Big) \right)   \nonumber  \\  
& \qquad \qquad \cdot \left( \sum_{\ell=1}^t  \left(|x^\ell_i - z^{\ell+1}_i - \bar{\mu}_\ell x^*_i |^2 \, + \, |\hx^\ell_i - v^{\ell+1}_i |^2\right) \right)^{\frac{1}{2}} \nonumber \\
& \leq C(4t+2)\left[ 1 + \frac{\| \bx^* \|^2}{d} +  
\sum_{\ell=1}^t \Big( \frac{\|\bx^\ell \|^2}{d} + \frac{\|\hbx^\ell \|^2}{d} + \frac{\| \bz^{\ell+1} + \bar{\mu}_\ell \, \bx^* \|^2}{d} + 
\frac{\| \bv^{\ell+1} \|^2}{d} \Big) \right]^{\frac{1}{2}} \nonumber \\
&  \qquad \qquad \qquad \qquad \qquad \qquad  \qquad  \qquad\cdot \left( \sum_{\ell=1}^t \left(\frac{\| \bx^\ell - \bz^{\ell+1} - \barmu_\ell \bx^*\|^2}{d} + \frac{\|\hbx^\ell - \bv^{\ell+1}\|^2}{d} \, \right)\right)^{\frac{1}{2}},  \label{eq:xz_diff}
\end{align}
where the last step uses Cauchy-Schwarz inequality (twice). Similarly, we obtain 
\begin{align}
    & \left\vert    \frac{1}{n} \sum_{i=1}^n \phi( r^1_i, \ldots, r^t_i, s^1_i, \ldots, s^{t+1}_i, \, y_i ) \, - \, \frac{1}{n} \sum_{i=1}^n \phi(m^2_i, \ldots, m^{t+1}_i, u^2_i, \ldots, u^{t+2}_i, \, q(m^1_i, \veps_i) )  \right \vert  \nonumber \\
    & \leq  C(4t+5) \left[ 1 + \frac{\| \by \|^2}{n} + 
\frac{\| q(\bm^1, \bveps) \|^2}{n} + 
\sum_{\ell=1}^t \Bigg( \frac{\|\br^\ell \|^2}{n} + \frac{\| \bm^{\ell+1} \|^2}{n} \Bigg) + 
\sum_{\ell=1}^{t+1} \Bigg( \frac{\|\bs^\ell \|^2}{n} +  \frac{\| \bu^{\ell+1} \|^2}{n} \Bigg)  \right]^{\frac{1}{2}} \nonumber \\
& \qquad  \cdot \left( \sum_{\ell=1}^t  \Big( \frac{\|\br^\ell - \bm^{\ell+1}\|^2}{n} \, + \, \frac{\| \bs^\ell - \bu^{\ell+1} \|^2}{n} \Big)  + 
\frac{\| \bs^{t+1} - \bu^{t+2} \|^2}{n} \, + \, 
\frac{\|\by - q(\bm^1, \bveps) \|^2}{n}  \right)^{\frac{1}{2}}.
\label{eq:rm_diff}
\end{align}
We will inductively show that as $n \to \infty$: \emph{(i)} each of the terms in the last line of \eqref{eq:xz_diff} and \eqref{eq:rm_diff} converges to zero, and \emph{(ii)} the terms within the square brackets in \eqref{eq:xz_diff} and \eqref{eq:rm_diff} all converge to finite, deterministic limits. 

\underline{Base case $t=1$}:  Consider \eqref{eq:xz_diff} for $t=1$. From  the AMP initialization, we have $\bx^1= \bA^{\sT} h_1(\by)$, and from \eqref{eq:auxAMPz}, we have
\beq
\bz^2 = \bA^{\sT}h_1(q(\bm^1, \bveps)) \,  -  \, \ssb_{21} \bv^1 = \bA^{\sT}h_1(\by) \,  -  \, \delta \kappa_2 \< \partial_g h_1(q(\bg, \bveps)) \> \, \bx^*,
\eeq
where the last equality is obtained by recalling that $\bm^1=\bg$, $\bv^1= \bx^*$, and computing the matrix $\bM^{\ssb}_2$ in \eqref{eq:Mta_Mtb} to verify that $\ssb_{21} = \delta \kappa_2 \< \partial_g h_1(q(\bg, \bveps)) \>$. We therefore have
\beq
\begin{split}
  &   \frac{\| \bx^1 - \bz^2 - \barmu_1 \bx^* \|^2}{d}  = \frac{\| \bx^* \|^2}{d} \left( \delta \kappa_2 \< \partial_g h_1(q(\bg, \bveps)) \> 
   - \barmu_1 \right)^2 \\
   & = \frac{\| \bx^* \|^2}{d}   \delta^2 \, \left[  \kappa_2(  \< \partial_g h_1(q(\bg, \bveps)) \> 
   -  \E\{ \partial_g h_1(q(G, \veps)) \} ) \, + \, \E\{ \partial_g h_1(q(G, \veps) \} (\kappa_2-\barkap_2) \right]^2, 
\end{split}
\eeq
where for the last equality, we use the definition of $\barmu_1$ in   \eqref{eq:SE_init}. By the assumptions of the theorem, $\| \bx^* \|^2/d  \to \E\{ X_*^2\}$. Since $\bm^1=\bg$, applying Proposition \ref{prop:auxSE} for $t=1$ gives 
\beq 
( \bg, \bveps) \stackrel{W_2}{\longrightarrow} (G, \veps).
\label{eq:bveps_conv}
\eeq
 Since $h_1(q(g, \veps))$ is Lipschitz in each argument,  \eqref{eq:bveps_conv} together with Lemma \ref{lem:lipderiv} implies that $ \< \partial_g h_1(q(\bg, \bveps)) \> \to  \E\{ \partial_g h_1(q(G, \veps)) \}$. Furthermore, by the model assumptions $\kappa_2 \to \barkap_2$. Therefore, 
\beq
\lim_{n \to \infty} \frac{1}{d}\| \bx^1 - \bz^2 - \barmu_1 \bx^* \|^2 = 0.
\label{eq:x1z2diff}
\eeq
Since $\hbx^1 = f_1(\bx^1)$ and $\bv^2=f_1(\bz^2 + \barmu_1 \bx^*)$ with $f_1$ being Lipschitz, we have
\beq
\frac{\| \hbx^1 - \bv^2 \|}{d} \le C \,  \frac{1}{d}\| \bx^1 - \bz^2 - \barmu_1 \bx^* \|^2 \to 0, \ \text{ as } n \to \infty.
\label{eq:hx_v2}
\eeq
Now consider the terms inside the square brackets in \eqref{eq:xz_diff}.
Using Proposition \ref{prop:auxSE} and Lemma \ref{lem:SE_equiv}, we have the following  limits for $t \ge 1$:
\begin{equation}
    \begin{split}
        \lim_{n \to \infty} \frac{\| \bz^{t+1} + \barmu_t \bx^* \|^2}{d} = \E\{ (Z_{t+1}+ \barmu_t X_*)^2\} =\E\{ X_t^2 \},  \qquad \lim_{n \to \infty} \frac{\| \bv^{t+1} \|^2}{d}= \E\{ V_{t+1}^2\} = \E\{\hat{X}_t^2 \}.
    \end{split}
    \label{eq:PL2lims_t1}
\end{equation} 
Using the triangle inequality, we have the following lower and upper bounds, for $t\ge 1$:
\begin{equation}
    \begin{split}
       \|  \bz^{t+1} + \barmu_t \bx^* \| \, -  \, \| \bx^t - \bz^{t+1} - \barmu_t \bx^* \| &\le    \| \bx^t \| \le  \|  \bz^{t+1}  + \barmu_t \bx^* \| \, + \, \| \bx^t - \bz^{t+1} - \barmu_t \bx^* \|,  \\
             \|  \bv^{t+1} \| \, -  \, \| \hbx^t - \bv^{t+1} \| &\le    \| \hbx^t \| \le  \|  \bv^{t+1} \| \, + \, \| \hbx^t - \bv^{t+1} \|.
    \end{split}
    \label{eq:zv_triangle_eq}
\end{equation}
Using \eqref{eq:x1z2diff}-\eqref{eq:zv_triangle_eq}, we obtain
\beq
\lim_{n \to \infty} \frac{\| \bx^1 \|^2}{d} = \E\{ X_1^2\}, \qquad  
\lim_{n \to \infty} \frac{\| \hbx^1 \|^2}{d} = \E\{ \hat{X}_1^2\}.
\label{eq:x1hx1_lim}
\eeq
Using \eqref{eq:x1z2diff}-\eqref{eq:x1hx1_lim} in \eqref{eq:xz_diff}, we obtain \eqref{eq:PLdiff_x} for $t=1$.

Next consider \eqref{eq:rm_diff} for $t=1$. From the definition of auxiliary AMP in \eqref{eq:auxAMPz}-\eqref{eq:tf_defs}, we have $\bm^1= \bg$ and 
$$\bu^2= h_1(q(\bm^1, \bveps)) = h_1(q(\bg, \bveps)) = h_1(\by) = \bs^1, $$ 
where the last equality holds due to the initialization of the true AMP (below \eqref{eq:AMP_rt_update}). We therefore have 
\beq
\frac{\| \by - q(\bm^1, \bveps) \|^2}{n} =0, \qquad \frac{\| \bs^1 - \bu^2 \|^2}{n}=0. 
\label{eq:init_eqs}
\eeq
Next, from \eqref{eq:auxAMPm} we have
\begin{equation}
    \bm^2 = \bA \bv^2 - \sa_{22} \,  \bu^2 =  
    \bA \bv^2 - \kappa_2 \< \partial_1 f_1(\bz^2 + \barmu_1 \bx^*) \> \,  h_1(\by),
    \label{eq:m2_exp}
\end{equation}
where we have used $\bu^1=\bzero$ and the value of $\sa_{22}$ obtained via $\bM^{\sa}_2$ in \eqref{eq:Mta_Mtb}. For the true AMP, from \eqref{eq:AMP_rt_update} we have
\begin{equation}
    \br^1 = \bA \hbx^1 - \alpha_{11} \bs^1 = \bA \hbx^1 - \kappa_2 \< \partial_1 f_1(\bx^1) \> \, h_1(\by).
    \label{eq:r1_exp}
\end{equation}
Combining \eqref{eq:m2_exp} and \eqref{eq:r1_exp}, we obtain
\begin{equation}
\begin{split}
    \frac{\| \br^1 - \bm^2 \|^2}{n} & \leq 2\frac{ \| \bA (\hbx^1 - \bv^2)  \|^2}{n} + 2\kappa_2^2  \frac{ \| h_1(\by) \|^2}{n}\left( \< \partial_1 f_1(\bx^1) \> - \< \partial_1 f_1(\bz^2 + \barmu_1 \bx^*) \>  \right)^2 \\
    & \le 2 \| \bA \|_{\op}^2 \frac{\| \hbx^1 - \bv^2  \|^2}{n} + 2 \kappa_2^2  \frac{ \| h_1(\by) \|^2}{n}\left( \< \partial_1 f_1(\bx^1) \> - \< \partial_1 f_1(\bz^2 + \barmu_1 \bx^*) \>  \right)^2.
    \end{split}
    \label{eq:r1m2_diff}
\end{equation}
By assumption, the empirical distribution of $\blambda$, the vector of singular values, converges to $\Lambda$ which has compact support. Therefore, $\| \bA\|_{\op} \le C$, and  by \eqref{eq:hx_v2}, the first term above tends to zero. We also have $\kappa_2 \to \bar{\kappa}_2$ and $\| h_1(\by) \|^2/n \to \E\{ h_1(Y)^2 \}$. Since $f_1$ is Lipschitz and we have shown above that $\bx^1 \stackrel{\mathclap{W_2}}{\longrightarrow} X_1$, Lemma \ref{lem:lipderiv} implies that
\beq
\lim_{n \to \infty} \, \< \partial_1 f_1(\bx^1) \> = \E\{ \partial_1 f_1(X_1) \}.
\label{eq:part_f1_x}
\eeq
Similarly, since by Proposition \ref{prop:auxSE} we have $\bz^2 + \barmu_1 \bx^* \stackrel{\mathclap{W_2}}{\longrightarrow} (Z_2 +  \barmu_1 X_*)$,  Lemma \ref{lem:lipderiv} implies
\beq
\lim_{n \to \infty} \, \< \partial_1 f_1(\bz^2 + \barmu_1 \bx^*) \> = 
E\{ \partial_1 f_1(Z_2 + \barmu_1 \bx^*) = E\{ \partial_1 f_1(X_1)\},
\label{eq:part_f1_z}
\eeq
where the last equality follows from Lemma \ref{lem:SE_equiv}. Using \eqref{eq:part_f1_x} and \eqref{eq:part_f1_z} in \eqref{eq:r1m2_diff}, we  have
\beq
\lim_{n \to \infty} \frac{\| \br^1 - \bm^2 \|^2}{n} =0.
\label{eq:r1m2_lim}
\eeq
Since $h_2$ is Lipschitz in each argument, \eqref{eq:r1m2_lim} also implies that
\beq
\lim_{n \to \infty} \frac{\| \bs^2 - \bu^3 \|^2}{n} = \lim_{n \to \infty} \frac{\| h_2(\br^1, \, q(\bg, \bveps)) - h_2(\bm^2, \, q(\bm^1, \bveps)) \|^2}{n}  =0,
\label{eq:s2u3_lim}
\eeq
where we have used $\bm^1=\bg$. Eqs.  \eqref{eq:init_eqs}, \eqref{eq:r1m2_lim} and \eqref{eq:s2u3_lim} show that for $t=1$, each term on the last line of \eqref{eq:rm_diff} tends to zero. 
Using Proposition \ref{prop:auxSE} and Lemma \ref{lem:SE_equiv}, we have for $t \ge 1$:
\begin{equation}
\begin{split}
   &  \lim_{n \to \infty} \frac{\| \bm^{t+1}\|^2}{n} = \E\{ M_{t+1}^2\} = \E\{R_t^2 \}, \qquad
    \lim_{n \to \infty} \frac{\| \bu^{t+1}\|^2}{n} = \E\{ U_{t+1}^2\} = \E\{ S_t^2\},  \\
  &  \lim_{n \to \infty} \frac{ \| q(\bm^1, \bveps) \|^2}{n} = \lim_{n \to \infty} \frac{ \| \by \|^2}{n}  =\E\{ q(M_1, \veps)^2\}  = \E\{ Y^2 \}.
     \end{split}
     \label{eq:mu_lims}
\end{equation}
Using the triangle inequality, we have the following lower and upper bounds:
\begin{equation}
    \begin{split}
      \|  \bm^{t+1} \| \, -  \, \| \br^t - \bm^{t+1} \| \le    \| \br^t \| \le  \|  \bm^{t+1} \| \, + \, \| \br^t - \bm^{t+1} \|,  \\
            \|  \bu^{t+2} \| \, -  \, \| \bs^{t+1} - \bu^{t+2} \| \le    \| \bs^{t+1} \| \le  \|  \bu^{t+2} \| \, + \, \| \bs^{t+1} - \bu^{t+2} \|.
    \end{split}
    \label{eq:mu_triangle_eq}
\end{equation}
Combining \eqref{eq:mu_triangle_eq} with \eqref{eq:r1m2_lim}-\eqref{eq:mu_lims}, we obtain the following  limits:
\beq
   \lim_{n \to \infty} \frac{\| \br^{1}\|^2}{n} = \E\{ R_{1}^2\}, \qquad 
   \lim_{n \to \infty} \frac{\| \bs^{2}\|^2}{n} = \E\{ S_2^2\}.
   \label{eq:r1s2_lims}
\eeq
Using the limits in \eqref{eq:mu_lims} and \eqref{eq:r1s2_lims} in \eqref{eq:rm_diff} yields the result \eqref{eq:PLdiff_r} for $t=1$.

\underline{Induction step}: Assume towards induction that the results \eqref{eq:PLdiff_x}-\eqref{eq:PLdiff_r} hold with $t$ replaced by $(t-1)$ for some $(t-1) \ge 1$, and that 
\beq
\begin{split}
& \lim_{n \to \infty} \, \frac{\| \bx^\ell - \bz^{\ell+1} - \barmu_\ell \bx^*\|^2}{d} = 0, \quad 
\lim_{n \to \infty} \, \frac{\|\hbx^\ell - \bv^{\ell+1}\|^2}{d} =0,  \\
& 
\lim_{n \to \infty} \, \frac{\| \br^\ell - \bm^{\ell+1} \|^2}{n} = 0, 
\quad \lim_{n \to \infty} \, \frac{\|\bs^{\ell+1} - \bu^{\ell+2}\|^2}{n} =0, \\
& \lim_{n \to \infty} \frac{\| \bx^\ell \|^2}{d} =\E\{ X_\ell^2 \}, \ \  \lim_{n \to \infty} \frac{\| \hbx^\ell \|^2}{d} =\E\{ \hX_\ell^2 \},   \\ 
&  \lim_{n \to \infty} \frac{\| \br^\ell \|^2}{n} =\E\{ R_\ell^2 \}, \lim_{n \to \infty} \frac{\| \bs^{\ell+1} \|^2}{n} =\E\{ S_{\ell+1}^2 \},\,\, \text{ for } \ell \in [t-1].
\end{split}
\label{eq:ind_hypt}
\eeq
From the definitions of $\bx^t$ and $\bz^{t+1}$ (see \eqref{eq:AMP_xt_update} and \eqref{eq:auxAMPz}), we have
\begin{equation}
    \begin{split}
        &  \bx^t - \bz^{t+1} - \barmu_t \bx^* = \bA^{\sT} (\bs^t - \bu^{t+1})
        + \sum_{i=1}^{t-1} ( \ssb_{t+1, i+1} \bv^{i+1} - \beta_{ti} \hbx^i) \, 
        + \ssb_{t+1, 1} \bv^{1}  - \barmu_t \bx^* \\
        & =  \bA^{\sT} (\bs^t - \bu^{t+1})
        + \sum_{i=1}^{t-1} \ssb_{t+1, i+1} (\bv^{i+1} - \hbx^i) 
        + \sum_{i=1}^{t-1} (\ssb_{t+1, i+1} - \beta_{ti}) \hbx^i \, 
        + (\ssb_{t+1, 1}  - \barmu_t) \bx^*,
    \end{split}
\end{equation}
where we have used the fact that $\bv^1= \bx^*$. Using Cauchy-Schwarz inequality, we then have
\begin{equation}
    \begin{split}
        \frac{\| \bx^t - \bz^{t+1} - \barmu_t \bx^* \|^2}{n} &  \le 2t \Bigg[ 
        \|\bA \|^2_{\op} \frac{\| \bs^t - \bu^{t+1} \|^2}{n} \, + \,  \sum_{i=1}^{t-1} \,  \ssb_{t+1, i+1}^2 \frac{\| \bv^{i+1} - \hbx^i \|^2}{n} \\
        & \qquad\qquad\qquad  + \, \sum_{i=1}^{t-1} \, (\ssb_{t+1, i+1} - \beta_{ti})^2 \frac{\| \hbx^i \|^2}{n} + (\ssb_{t+1, 1}  - \barmu_t)^2 \frac{\| \bx^* \|^2}{n} \Bigg] \\
        & := 2t(T_1 + T_2 + T_3 + T_4).
    \end{split}
    \label{eq:T1T2T3T4}
\end{equation}
Since $\| \bA \|_{\op} \le C$, the induction hypothesis \eqref{eq:ind_hypt} implies that $T_1 \to 0$.  By the induction hypothesis, we also have $\| \bv^{i+1} - \hbx^i \|^2/d  \to 0$ and $\| \hbx^i\|^2/d \to \E\{ \hX^2 \}$ for $i \le (t-1)$. Furthermore, $\| \bx^* \|^2/d \to \E\{ X_*^2 \}$. Hence, we can prove that $T_2, T_3, T_4$ each tend to zero by showing that:
\begin{align}
    & \lim_{n \to \infty} \ssb_{t+1,1} = \bar{\mu}_t, \label{eq:sb_mu_lim} \\
    & \lim_{n \to \infty} \ssb_{t+1, i+1} =   \lim_{n \to \infty} \beta_{t,i} = \bar{\beta}_{t,i}, \qquad  i \in [t-1], 
    \label{eq:bt1_betat_lim} 
\end{align}
 where the limiting values $(\bar{\beta}_{t,i})$ in \eqref{eq:bt1_betat_lim} will be defined below (see \eqref{eq:Mtbeta_lim}).

Recall from \eqref{eq:Mta_Mtb}-\eqref{eq:atbt_coeffs}
that the coefficients $(\ssb_{t+1,j})_{j \le t}$ are determined by the entries of the matrices $\hbPsi_{t}$ and $\hbPhi_{t+1}$,  defined in \eqref{eq:bPQ_def}.  (Though the definition of  $\bM^{\ssb}_{t+1}$ in \eqref{eq:Mta_Mtb} involves $\hbPsi_{t+1}$, it can be verified that its last row  does not affect the computation, so the formula depends only on $\hbPsi_{t}, \hbPhi_{t+1}$.) From \eqref{eq:part_uv}, the non-zero entries of these matrices are of the form
\begin{equation*}
\begin{split}
  &   \< \partial_k f_\ell( \bz^2 + \barmu_1 \bx^*, \, \ldots,  \, \bz^{\ell+1} + \barmu_\ell \bx^* ) \>, \quad 1 \le k \le \ell \le (t-1), \\
& \< \partial_g h_\ell( \bm^2, \, \ldots,  \, \bm^{\ell}, \, q(\bm^1, \bveps) ) \>, \quad \< \partial_k h_\ell( \bm^2, \, \ldots,  \, \bm^{\ell}, \, q(\bm^1, \bveps) ) \>,
\quad  1 \le k < \ell \le t,
\end{split}
\end{equation*}
where we recall that $\partial_k$ denotes the partial derivative with respect to the $k$-th argument.
By Proposition \ref{prop:auxSE} and Lemma \ref{lem:SE_equiv}, we have that for $\ell \ge 1$:
\begin{equation}
    \begin{split}
        & ( \bz^2 + \barmu_1 \bx^*, \, \ldots,  \, \bz^{\ell+1} + \barmu_\ell \bx^* ) \stackrel{W_2}{\longrightarrow} (Z_2 + \barmu_1 X_*, \ldots, Z_{\ell+1} + \barmu_\ell X_* )  \stackrel{\rm{d}}{=} (X_1, \ldots, X_\ell), \\
        & (\bm^1, \ldots, \bm^\ell, \bveps) \,  \stackrel{W_2}{\longrightarrow} \,
        (M_1, \ldots, M_\ell, \veps)  \stackrel{\rm{d}}{=} (G, R_1, \ldots, R_{\ell -1}, \veps).
    \end{split}
    \label{eq:zm_W2}
\end{equation}
Since the functions $f_\ell$ and $h_\ell$ are Lipschitz in each argument, \eqref{eq:zm_W2} together with Lemma \ref{lem:lipderiv} implies that
\begin{equation}
    \begin{split}
        & \lim_{n \to \infty} \< \partial_k f_\ell( \bz^2 + \barmu_1 \bx^*, \, \ldots,  \, \bz^{\ell+1} + \barmu_\ell \bx^* ) \> = 
        \E\{ \partial_k f_\ell(X_1, \ldots, X_\ell) \}, \quad 1 \le k \le \ell \le (t-1), \\
       &  \lim_{n \to \infty} \< \partial_g h_\ell( \bm^2, \, \ldots,  \, \bm^{\ell}, \, q(\bm^1, \bveps) ) \> = \E\{ \partial_g h_\ell(R_1, \ldots, R_{\ell-1}), q(G, \veps) \}, \\
       &  \lim_{n \to \infty} \< \partial_k h_\ell( \bm^2, \, \ldots,  \, \bm^{\ell}, \, q(\bm^1, \bveps) ) \> = \E\{ \partial_k h_\ell(R_1, \ldots, R_{\ell-1}, q(G, \veps)) \}, \quad 1 \le k < \ell \le t.
    \end{split}
    \label{eq:part_gl_hl_lim1}
\end{equation}
Therefore, $\hbPsi_{t} \to \tbPsi_{t}$ and $\hbPhi_{t+1} \to \tbPhi_{t+1}$, where $\tbPsi_t, \tbPhi_{t+1}$ are defined in \eqref{eq:tbPQ_def}. Consequently, the matrix $\bM^{\ssb}_{t+1}$ and the coefficients $(\ssb_{t+1, j})$, defined via \eqref{eq:Mta_Mtb}-\eqref{eq:atbt_coeffs}, converge to the following limits:
\beq
\lim_{n \to \infty} \bM^{\ssb}_{t+1}= \tbM^{\ssb}_{t+1} \equiv \delta \sum_{j=0}^t \barkap_{2(j+1)} \tbPhi_{t+1} (\tbPsi_{t+1} \tbPhi_{t+1})^j , \qquad \lim_{n \to \infty} \ssb_{t+1, j} =  \bar{\ssb}_{t+1, j},  \quad  j \in [t], 
\label{eq:Mtb_lim}
\eeq
where 
$(\bar{\ssb}_{t+1, j})$ are computed according to  \eqref{eq:atbt_coeffs} from the last row of $\tbM^{\ssb}_{t+1}$.
By the induction hypothesis \eqref{eq:ind_hypt}, we have $ \| \hbx^i - \bv^{i+1} \|^2/d \to  0$ for $i \le (t-1)$. Therefore the term $T_2$ in \eqref{eq:T1T2T3T4} tends to $0$ as $n \to \infty$.

Using \eqref{eq:beta_def}, the coefficients $(\beta_{ti})$ are determined by the entries of the matrices $\bPsi_{t+1}$ and $\bPhi_{t+1}$,  defined in \eqref{eq:Psi_Phi_def}.  The non-zero entries of these matrices are of the form
\begin{equation*}
\begin{split}
  &   \< \partial_k f_\ell( \bx^1, \, \ldots,  \, \bx^\ell ) \>,  \quad   1 \le k \le \ell \le t, \\
&  \< \partial_g h_\ell( \br^1, \, \ldots,  \, \br^{\ell-1}, \, q(\bg, \bveps) ) \>, 
 \quad \< \partial_k h_\ell( \br^1, \, \ldots,  \, \br^{\ell-1}, \, q(\bg, \bveps) ) \>, \quad  1 \le k < \ell \le t.
\end{split}
\end{equation*}
By the induction hypothesis \eqref{eq:PLdiff_x}-\eqref{eq:PLdiff_r} for $(t-1)$ and \eqref{eq:zm_W2}, we have 
\begin{equation}
\begin{split}
    & (\bx^1, \ldots, \bx^{t-1}) \stackrel{W_2}{\longrightarrow}  (Z_2 + \barmu_1 X_*, \ldots, Z_{t} + \barmu_{t-1} X_* )  \stackrel{\rm{d}}{=} (X_1, \ldots, X_{t-1}) \\
    & (\br^1, \ldots, \br^{t-1}, q(\bg, \bveps))
     \stackrel{W_2}{\longrightarrow} (M_2, \ldots, M_t, q(M_1, \veps)) 
     \stackrel{\rm{d}}{=} (R_1, \ldots, R_{t-1},q(G, \veps)).
    \end{split}
    \label{eq:ind_rt1_W2}
\end{equation}
Since $f_\ell$ and $h_\ell$ are Lipschitz in each argument, \eqref{eq:ind_rt1_W2} together with Lemma \ref{lem:lipderiv} implies
\begin{equation}
    \begin{split}
        & \lim_{n \to \infty} \< \partial_k f_\ell( \bx^1, \ldots, \bx^\ell ) \> = 
        \E\{ \partial_k f_\ell(X_1, \ldots, X_\ell) \}, \quad 
        1 \le k \le \ell  \le (t-1), \\
        & \lim_{n \to \infty} \< \partial_g h_\ell( \br^1, \ldots, \br^{\ell-1}, q(\bg, \bveps) ) \> = \E\{ \partial_g h_\ell(R_1, \ldots, R_{\ell-1}), q(G, \veps))\},   \\
        &  \lim_{n \to \infty} \< \partial_k h_\ell( \br^1, \ldots, \br^{\ell-1}, q(\bg, \bveps) ) \> = \E\{ \partial_k h_\ell(R_1, \ldots, R_{\ell-1}, q(G, \veps)) \}, \quad 1 \le k < \ell  \le t.
    \end{split}
    \label{eq:part_gl_hl_lim2}
\end{equation}
Therefore, $\bPsi_{t}  \to \barbPsi_{t}$ and $\bPhi_{t+1} \to \barbPhi_{t+1}$, where the entries of $\barbPsi_{t}, \barbPhi_{t+1}$ are defined as in \eqref{eq:emp_avg_shorthand}. We note that computing  $\bM_{t+1}^{\beta}$ defined in \eqref{eq:Mbeta_def} requires knowledge of only $\bPsi_{t}$ and $\bPhi_{t+1}$ since the last row of $\bPsi_{t+1}$ is zeroed out in the multiplication with $\bPhi_{t+1}$. Therefore, 
\beq 
\lim_{n \to \infty} \bM_{t+1}^{\beta} =  \barbM_{t+1}^{\beta} \equiv  \delta \sum_{j=0}^{t} \barkap_{2(j+1)} \barbPhi_{t+1} (\barbPsi_{t+1} \barbPhi_{t+1})^j, \qquad 
\lim_{n \to \infty} \beta_{t,i}= \bar{\beta}_{t,i}, \quad  i \in [t-1].
\label{eq:Mtbeta_lim}
\eeq
where $(\bar{\beta}_{t,i})$  are computed according to  \eqref{eq:beta_def} from the last row of $\barbM^{\beta}_{t+1}$. Since the limits in \eqref{eq:part_gl_hl_lim1} and \eqref{eq:part_gl_hl_lim2} are the same, using the formulas
for $\barbPsi_{t}, \barbPhi_{t+1}$ (from \eqref{eq:emp_avg_shorthand}) and for $\tbPsi_{t}, \tbPhi_{t+1}$ (from \eqref{eq:tbPQ_def}), we have 
\begin{equation}
    \barbPsi_{t} = \tbPsi_t, \qquad \barbPhi_{t+1} = \tbPhi_{t+1}, \qquad \barbM_{t+1}^{\beta} = \tbM_{t+1}^{\ssb}.
    \label{eq:bar_til_iden}
\end{equation}
Combining \eqref{eq:Mtb_lim}, \eqref{eq:Mtbeta_lim}, \eqref{eq:bar_til_iden} and recalling that $\barmu_t = (\bM^\beta_{t+1})_{t+1,1}$, we obtain the claims in \eqref{eq:sb_mu_lim}-\eqref{eq:bt1_betat_lim}. We have therefore shown that each of the four terms in \eqref{eq:T1T2T3T4} tends to zero, and hence 
\beq
\lim_{n \to \infty} \, \frac{\| \bx^t - \bz^{t+1} - \barmu_t \bx^* \|^2}{d}  = 0.
\label{eq:xtzt1_conv}
\eeq
Moreover, since $\hbx^t = f_t(\bx^1, \ldots, \bx^t)$ and $\bv^{t+1} = f_t(\bz^{2}+\barmu_1 \bx^*, \ldots, \bz^{t+1}+\barmu_t \bx^* )$ with $f_t$ Lipschitz, we also have 
\beq
\lim_{n \to \infty} \frac{\| \hbx^t - \bv^{t+1} \|^2}{d}=0.
\label{eq:hxt_vt1_conv}
\eeq
Using \eqref{eq:xtzt1_conv}-\eqref{eq:hxt_vt1_conv}  together with the bounds in \eqref{eq:zv_triangle_eq} then yields
\beq
\begin{split}
& \lim_{n \to \infty} \frac{\| \bx^t \|^2}{d} = \lim_{n \to \infty} \frac{\| \bz^{t+1} + \barmu_{t} \bx^* \|^2}{d}= \E\{ (Z_{t+1} + \barmu_t X)^2 \} = \E\{ X_t^2 \}, \\
&  \lim_{n \to \infty} \frac{\| \hbx^{t} \|^2}{d} = \lim_{n \to \infty} \frac{\| \bv^{t+1} \|^2}{d} = \E\{ V_{t+1}^2 \} = \E\{ \hX_t^2 \},
\label{eq:x2hx2_lims}
\end{split}
\eeq
where the last equality in each line above is due to Lemma \ref{lem:SE_equiv}. Using \eqref{eq:xtzt1_conv}-\eqref{eq:x2hx2_lims} and the induction hypothesis \eqref{eq:ind_hypt} in \eqref{eq:xz_diff}  yields the result \eqref{eq:PLdiff_x}.

The proof of \eqref{eq:PLdiff_r} is along similar lines. From the definitions of $\br^t$ and $\bm^{t+1}$ (see \eqref{eq:AMP_rt_update} and \eqref{eq:auxAMPm}), we have
\beq
\begin{split}
\br^t - \bm^{t+1} & = \bA(\hbx^t - \bv^{t+1}) + \sum_{i=1}^{t} (\sa_{t+1,i+1} \bu^{i+1} - \alpha_{t,i} \bs^i) \\ 
 & = \bA(\hbx^t - \bv^{t+1}) + \sum_{i=1}^{t} \sa_{t+1,i+1} (\bu^{i+1} -  \bs^i) \, +  \,
 \sum_{i=1}^{t}  (\sa_{t+1,i+1} - \alpha_{t,i}) \bs^i. 
\end{split}
\eeq
Using the Cauchy-Schwarz inequality, we obtain
\begin{equation}
    \begin{split}
        & \frac{\| \br^t - \bm^{t+1} \|^2}{n}  \\
        & \le 2(t+2)  
        \Bigg[ \| \bA\|^2_{\op} \frac{\| \hbx^t - \bv^{t+1} \|^2}{n} 
        \, + \, \sum_{i=1}^{t} \sa_{t+1, i+1}^2 
        \frac{\| \bu^{i+1} - \bs^i \|^2}{n} \, + \, 
        \sum_{i=1}^t (\sa_{t+1,i+1} - \alpha_{t,i})^2 \frac{\| \bs^i \|^2}{n} \Bigg].
        \label{eq:rt_mt1_split}
    \end{split}
\end{equation}
By assumption,  $\| \bA \|_{\op}^2 \le C$; therefore, by \eqref{eq:hxt_vt1_conv} the first term in the brackets tends to zero. By the induction hypothesis \eqref{eq:ind_hypt}, we have
\begin{equation}
    \begin{split}
       \lim_{n \to \infty} \frac{\| \bu^{i+1} - \bs^{i} \|^2}{n}=0, \quad
       \lim_{n \to \infty} \frac{\| \bs^i \|^2}{n} = \E\{ S_i^2 \},  \quad 1\le  i \le t.
    \end{split}
    \label{eq:hx_us_lims}
\end{equation}
We can also show that 
\beq
\lim_{n \to \infty} \sa_{t+1, i+1} =
\lim_{n \to \infty} \alpha_{t, i} = \bar{\alpha}_{t,i}, \qquad 1\le i \le t,
\label{eq:at_alph_lim}
\eeq
where $(\bar{\alpha}_{t,1}, \ldots, \bar{\alpha}_{t,t})$ is defined as in \eqref{eq:alpha_def} using the matrix 
$ \barbM_{t+1}^\alpha \equiv \sum_{j=0}^{t+1} \barkap_{2(j+1)} \barbPsi_{t+1} (\barbPhi_{t+1} \barbPsi_{t+1})^j$.
The proof of \eqref{eq:at_alph_lim} is omitted as it is similar to that of \eqref{eq:bt1_betat_lim}: we show that $\bM^{\alpha}_{t+1} \to \barbM^{\alpha}_{t+1}$  and $\bM^{\sa}_{t+1} \to \tbM^{\sa}_{t+1} \equiv \sum_{j=0}^{t+1} \barkap_{2(j+1)} \tbPsi_{t+1} (\tbPhi_{t+1} \tbPsi_{t+1})^j$, and then that $\tbM^{\sa}_{t+1}= \barbM^{\alpha}_{t+1}$. Using  \eqref{eq:hx_us_lims} and \eqref{eq:at_alph_lim} in \eqref{eq:rt_mt1_split}, we obtain 
\beq
\lim_{n \to \infty} \, \frac{\| \br^t - \bm^{t+1} \|^2}{n}=0, \qquad
\lim_{n \to \infty} \, \frac{\| \bs^{t+1} - \bu^{t+2} \|}{n}=0,
\label{eq:rtmt1_st1_ut2_conv}
\eeq
where the second limit holds because 
\[ 
\bs^{t+1} = h_{t+1}(\br^1, \ldots, \br^t, q(\bg, \bveps)), \qquad \bu^{t+2} = h_{t+1}(\bm^2, \ldots, \bm^{t+1}, q(\bm^1, \bveps)),
\]
with $h_{t+1}$ Lipschitz in each argument. Using \eqref{eq:rtmt1_st1_ut2_conv} together with the bounds in \eqref{eq:mu_triangle_eq} then yields
\beq
\begin{split}
& \lim_{n \to \infty} \frac{\| \br^t \|^2}{n} = \lim_{n \to \infty} \frac{\| \bm^{t+1} \|^2}{n}= \E\{ M_{t+1}^2 \} = \E\{ R_t^2 \}, \\
&  \lim_{n \to \infty} \frac{\| \bs^{t+1} \|^2}{n} = \lim_{n \to \infty} \frac{\| \bu^{t+2} \|^2}{n} = \E\{ U_{t+2}^2 \} = \E\{ S_{t+1}^2 \},
\label{eq:r2s2_lims}
\end{split}
\eeq
 where the last equality in each line above is due to Lemma \ref{lem:SE_equiv}.

Using \eqref{eq:rtmt1_st1_ut2_conv}, \eqref{eq:r2s2_lims}, and the induction hypothesis \eqref{eq:ind_hypt} in \eqref{eq:rm_diff}  yields the result \eqref{eq:PLdiff_r}, completing the proof.

\section{An Auxiliary Lemma}
\begin{lemma}
\label{lem:lipderiv}
Let $F \colon \reals^t \to\reals$ be a Lipschitz function, and let $\partial_k F$ denote its derivative with respect to the $k$-th argument, for $1 \le k \le t$. Assume that $\partial_k F$ is continuous almost everywhere in the $k$-th argument,  for each $k$. Let $(V_1^{(m)}, \ldots, V_t^{(m)})$  be a sequence of random vectors in $\reals^t$ converging in distribution to the random vector $(V_1, \ldots,V_t)$ as $m \to \infty$. Furthermore, assume that the distribution of $(V_1, \ldots, V_t)$ is absolutely continuous with respect to the Lebesgue measure. Then, 
\[ \lim_{m \to \infty}  \E\{ \partial_k F(V_1^{(m)}, \ldots, V_t^{(m)}) \} = \E\{ \partial_k F(V_1, \ldots, V_t) \}, \qquad  k \in [t]. \]
\end{lemma}
The result was proved for $t=2$ in Lemma 6 of \cite{BM-MPCS-2011}. The proof for $t > 2$ is essentially the same; see also Lemma 7.14 in \cite{feng2021unifying}.


\end{document}